\documentclass[10pt, letterpaper, twoside]{article} 

\usepackage{amsmath, amsthm, amssymb,bbold,mathrsfs} 
\usepackage{graphicx,hyperref}
\usepackage[toc,page, title]{appendix}
\usepackage{enumitem}
\setlist[enumerate]{topsep=1pt,itemsep=-0.5pt,parsep=2pt}
\usepackage[numbers]{natbib}

\newtheorem{thm}{Theorem}[section]
\newtheorem{cor}{Corollary}[section]
\newtheorem{lem}{Lemma}[section]
\newtheorem{prop}{Proposition}[section]

\newtheorem{assumption}{Assumption}[section]

\theoremstyle{remark}
\newtheorem{rem}{Remark}[section]

\theoremstyle{definition}

\numberwithin{equation}{section}

\def\re{\mathbb{R}} 
 \def\rn{\mathbb{R}^n}
\def\e{e}
\def\F{F}
\def\fe{\phi}
\def\Fe{\Phi}
\def\i{i}
\def\bi{d_{\pi^o,\i}}
\def\1{\mathbf{1}}
\def\0{\mathbf{0}}
\def\Z{Z}
\def\M{M}
\def\bM{{\bar M}}
\def\P{P_\pi}
\def\PL{P_{\pi,\gamma}^\lambda}

\def\Gm{\Gamma}
\def\Lm{\Lambda}
\def\r{r_\pi}
\def\rl{r_{\pi,\gamma}^\lambda}

\def\G{G}
\def\C{L}

\def\E{\mathbb{E}}
\def\B{\mathcal{B}}
\def\A{\mathcal{A}}
\def\S{\mathcal{S}}
\def\H{B}
\def\asto{\overset{a.s.}{\to}}

\def\J{\mathcal{J}}

\def\I{\mathbb{1}}
\def\Pr{\mathbf{P}}

\def\K{\mathcal{K}}
\def\X{\mathbf{X}}
\def\Y{\mathbf{Y}}
\def\Zs{\mathcal{Z}}

\makeatletter
\newcommand\appendix@section[1]{%
  \refstepcounter{section}%
  \orig@section*{Appendix \@Alph\c@section: #1}%
  \addcontentsline{toc}{section}{Appendix \@Alph\c@section: #1}%
}
\let\orig@section\section
\g@addto@macro\appendix{\let\section\appendix@section}
\makeatother

\begin{document} 
\markboth{Weak Convergence Properties of Constrained ETD Learning}{}

\title{Weak Convergence Properties of Constrained Emphatic Temporal-difference Learning with Constant and Slowly Diminishing Stepsize\thanks{This research was supported by a grant from Alberta Innovates---Technology Futures.}}

\author{Huizhen Yu\thanks{RLAI Lab, Department of Computing Science, University of Alberta, Canada (\texttt{janey.hzyu@gmail.com})}}
\date{}
\maketitle

\begin{abstract}
We consider the emphatic temporal-difference (TD) algorithm, ETD($\lambda$), for learning the value functions of stationary policies in a discounted, finite state and action Markov decision process. The ETD($\lambda$) algorithm was recently proposed by Sutton, Mahmood, and White \cite{SuMW14} to solve a long-standing divergence problem of the standard TD algorithm when it is applied to off-policy training, where data from an exploratory policy are used to evaluate other policies of interest. The almost sure convergence of ETD($\lambda$) has been proved in our recent work under general off-policy training conditions, but for a narrow range of diminishing stepsize. 
In this paper we present convergence results for constrained versions of ETD($\lambda$) with constant stepsize and with diminishing stepsize from a broad range. Our results characterize the asymptotic behavior of the trajectory of iterates produced by those algorithms, and are derived by combining key properties of ETD($\lambda$) with powerful convergence theorems from the weak convergence methods in stochastic approximation theory. 
For the case of constant stepsize, in addition to analyzing the behavior of the algorithms in the limit as the stepsize parameter approaches zero, we also analyze their behavior for a fixed stepsize and bound the deviations of their averaged iterates from the desired solution. These results are obtained by exploiting the weak Feller property of the Markov chains associated with the algorithms, and by using ergodic theorems for weak Feller Markov chains, in conjunction with the convergence results we get from the weak convergence methods.
Besides ETD($\lambda$), our analysis also applies to the off-policy TD($\lambda$) algorithm, when the divergence issue is avoided by setting $\lambda$ sufficiently large. It yields, for that case, new results on the asymptotic convergence properties of constrained off-policy TD($\lambda$) with constant or slowly diminishing stepsize.
\end{abstract}

temporal difference methods; importance sampling; stochastic approximation; convergence

\clearpage
\tableofcontents
\clearpage

\section{Introduction}
We consider discounted finite state and action Markov decision processes (MDPs) and the problem of learning an approximate value function for a given policy from \emph{off-policy} data, that is, from data due to a different policy. The first policy is called the \emph{target policy} and the second the \emph{behavior policy}. 
The case of \emph{on-policy} learning, where the target and behavior policies are the same, has been well-studied and widely applied (see e.g., \cite{Sut88,tr-disc} and the books \cite{BET,SuB}). Off-policy learning provides additional flexibilities and is useful in many contexts.
For example, one may want to avoid executing the target policy before estimating the potential risk for safety concerns, or one may want to learn value functions for many target policies in parallel from one exploratory behavior. These require off-policy learning. 
In addition, insofar as value functions (with respect to different reward/cost assignments) reflect statistical properties of future outcomes, off-policy learning can be used by an autonomous agent to build an experience-based internal model of the world in artificial intelligence applications \cite{Sut09}. Algorithms for off-policy learning are thus not only useful as model-free computational methods for solving MDPs, but can also potentially be a step toward the goal of making autonomous agents capable of learning over a long life-time, facing a sequence of diverse tasks.

In this paper we focus on a new off-policy learning algorithm proposed recently by Sutton, Mahmood, and White \cite{SuMW14}: the emphatic temporal-difference (TD) learning algorithm, or ETD($\lambda$). The algorithm is similar to the standard TD($\lambda$) algorithm with linear function approximation \cite{Sut88}, but uses a novel scheme to resolve a long-standing divergence problem in TD($\lambda$) when applied to off-policy data. Regarding the divergence problem, while TD($\lambda$) was proved to converge for the on-policy case~\cite{tr-disc}, it was known quite early that the algorithm can diverge in other cases \cite{Baird95,tr-disc} (for related discussions, see also the books \cite{BET,SuB} and the recent works \cite{SuMW14,MYWS15}). The difficulty is intrinsic to sampling states according to an arbitrary distribution. Since then alternative algorithms without convergence issues have been sought for off-policy learning. In particular, in the off-policy LSTD($\lambda$) algorithm \cite{by08,Yu-siam-lstd} (an extension of the on-policy least-squares version of TD($\lambda$), called LSTD($\lambda$) \cite{bt-lstd,lstd}), with higher computational complexity than TD($\lambda$), the linear equation associated with TD($\lambda$) is estimated from data and then solved.%footnote starts
\footnote{An efficient algorithm for solving the estimated equations is the one given in \cite{precond-td-yl} based on the line search method. It can also be applied to finding approximate solutions under additional penalty terms suggested by \cite{lstd-ps12}.}
%footnote ends 
In the gradient-TD algorithms \cite{gtd08,gtd09,maei11} and the proximal gradient-TD algorithms \cite{pmtd,pmtd2} (see also \cite{pmtd3,pmrl}), the difficulty in TD($\lambda$) is overcome by reformulating the approximate policy evaluation problem TD($\lambda$) attempts to solve as optimization problems and then tackle them with optimization techniques. (See the surveys \cite{bruno14,dnp14} for other algorithm examples.)

Compared to the algorithms just mentioned, ETD($\lambda$) is closer to the standard TD($\lambda$) algorithm and addresses the issue in TD($\lambda$) more directly.
It introduces a novel weighting scheme to re-weight the states when forming the eligibility traces in TD($\lambda$), so that the weights reflect the occupation frequencies of the target policy rather than the behavior policy. An important result of this weighting scheme is that under natural conditions on the function approximation architecture, the average dynamics of ETD($\lambda$) can be described by an affine function involving a negative definite matrix \cite{SuMW14,yu-etdarx},%footnote starts
\footnote{The papers \cite{SuMW14,MYWS15} work with the negation of the matrix that we associate with ETD($\lambda$) in this paper. The negative definiteness property we discuss here corresponds to the positive definiteness property discussed in \cite{SuMW14,MYWS15}.}
%footnote ends 
which provides a desired stability property, similar to the case of convergent on-policy TD algorithms. 

The almost sure convergence of ETD($\lambda$), under general off-policy training conditions, has been shown in our recent work~\cite{yu-etdarx} for diminishing stepsize. That result, however, requires the stepsize to diminish at the rate of $O(1/t)$, with $t$ being the time index of the iterate sequence. This range of stepsize is too narrow for applications. In practice, algorithms tend to make progress too slowly if the stepsize becomes too small, and the environment may be non-stationary, so it is often preferred to use a much larger stepsize or constant stepsize.

The purpose of this paper is to provide an analysis of ETD($\lambda$) for a broad range of stepsizes. Specifically, we consider constant stepsize and stepsize that can diminish at a rate much slower than $O(1/t)$. We will maintain general off-policy training conditions, without placing restrictions on the behavior policy. However, we will consider constrained versions of ETD($\lambda$), which constrain the iterates to be in a bounded set, and a mode of convergence that is weaker than almost sure convergence. Constraining the ETD($\lambda$) iterates is not only needed in analysis, but also a means to control the variances of the iterates, which is important in practice since off-policy learning algorithms generally have high variances. Almost sure convergence is no longer guaranteed for algorithms using large stepsizes; hence we analyze their behavior with respect to a weaker convergence mode.

We study a simple, basic version of constrained ETD($\lambda$) and several variations of it, some of which are biased but can mitigate the variance issue better. 
To give an overview of our results, we shall refer to the first algorithm as the unbiased algorithm, and its biased variations as the biased variants.
Two groups of results will be given to characterize the asymptotic behavior of the trajectory of iterates produced by these algorithms. The first group of results are derived by combining key properties of ETD($\lambda$) with powerful convergence theorems from the weak convergence methods in stochastic approximation theory~\cite{KuC78,KuS84a,KuY03}. The results show (roughly speaking) that:  
\begin{enumerate}
\item[(i)] In the case of diminishing stepsize, under mild conditions, the trajectory of iterates produced by the unbiased algorithm eventually spends nearly all its time in an arbitrarily small neighborhood of the desired solution, with an arbitrarily high probability (Theorem~\ref{thm-dim-stepsize}); and the trajectory produced by the biased algorithms has a similar behavior, when the algorithmic parameters are set to make the biases sufficiently small (Theorem~\ref{thm-dim-stepsize-b}). These results entail the convergence in mean to the desired solution for the unbiased algorithm (Cor.~\ref{cor-dim-stepsize}), and the convergence in probability to some vicinity of the desired solution for the biased variants. 
\item[(ii)] In the case of constant stepsize, imagine that we run the algorithms for all stepsizes; then conclusions similar to those in (i) hold in the limit as the stepsize parameter approaches zero (Theorems~\ref{thm-const-stepsize},~\ref{thm-const-stepsize-b}). In particular, a smaller stepsize parameter results in an increasingly longer segment of the trajectory to spend, with an increasing probability, nearly all its time in some neighborhood of the desired solution. The size of the neighborhood can be made arbitrarily small as the stepsize parameter approaches zero and, in the case of the biased variants, also as their biases are reduced.
\end{enumerate}
The next group of results are for the constant-stepsize case and complement the results in (ii) by focusing on the asymptotic behavior of the algorithms for a fixed stepsize. Among others, they show (roughly speaking) that:
\begin{enumerate}
\item[(iii)] For any given stepsize parameter, asymptotically, the expected maximal deviation of multiple consecutive averaged iterates from the desired solution can be bounded in terms of the masses that the invariant probability measures of certain associated Markov chains assign to a small neighborhood of the desired solution. Those probability masses approach one when the stepsize parameter approaches zero and, in the case of the biased variants, also when their biases are sufficiently small (Theorems~\ref{thm-c1}, \ref{thm-c1b}).
\item[(iv)] For a perturbed version of the unbiased algorithm and its biased variants, the maximal deviation of averaged iterates from the desired solution, under a given stepsize parameter, can be bounded almost surely in terms of those probability masses mentioned in (iii), for each initial condition (Theorems~\ref{thm-c2}, \ref{thm-c2b}).
\end{enumerate}
 
To derive the first group of results, we use powerful convergence theorems from the weak convergence methods in stochastic approximation theory \cite{KuC78,KuS84a,KuY03}. This theory builds on the ordinary differential equation (ODE) based proof method, treats the trajectory of iterates as a whole, and studies its asymptotic behavior through the continuous-time processes corresponding to left-shifted and interpolated iterates. The probability distributions of these continuous-time interpolated processes are analyzed (as probability measures on a function space) by the weak convergence methods, leading to a characterization of their limiting distributions, from which asymptotic properties of the trajectory of iterates can be obtained.

Most of our efforts in the first part of our analysis are to prove that the constrained ETD($\lambda$) algorithms satisfy the conditions required by the general convergence theorems just mentioned. We prove this by using key properties of ETD($\lambda$) iterates, most importantly, the ergodicity and uniform integrability properties of the trace iterates, and the convergence of certain averaged processes which, intuitively speaking, describe the averaged dynamics of ETD($\lambda$). Some of these properties were established earlier in our work \cite{yu-etdarx} when analyzing the almost sure convergence of ETD($\lambda$). Building upon that work, we prove the remaining properties needed in the analysis.

To derive the second group of results, we exploit the fact that in the case of constant stepsize, the iterates together with other random variables involved in the algorithms form weak Feller Markov chains, and such Markov chains have nice ergodicity properties. We use ergodic theorems for weak Feller Markov chains \cite{Mey89,MeT09}, together with the properties of ETD($\lambda$) iterates and the convergence results we get from the weak convergence methods, in this second part of our analysis.

Besides ETD($\lambda$), the analysis we give in the paper also applies to off-policy TD($\lambda$), when the divergence issue mentioned earlier is avoided by setting $\lambda$ sufficiently close to $1$. The reason is that in that case the off-policy TD($\lambda$) iterates have the same properties as the ones used in our analysis of ETD($\lambda$) and therefore, the same conclusions hold for constrained versions of off-policy TD($\lambda$), regarding their asymptotic convergence properties for constant or slowly diminishing stepsize (these results are new, to our knowledge). Similarly, our analysis also applies directly to the ETD($\lambda, \beta$) algorithm, a variation of ETD($\lambda$) recently proposed by Hallak et al.~\cite{etd-errbd2}.

Regarding practical performance of the algorithms, the biased ETD variant algorithms are much more robust than the unbiased algorithm despite the latter's superior asymptotic convergence properties. (This is not a surprise, for the biased algorithms are in fact defined by using a well-known robustifying approach from stochastic approximation theory~\cite{KuY03}.) Their behavior is demonstrated by experiments in \cite{MYWS15,etd-exp16}. In particular, \cite{etd-exp16} is our companion note for this paper and includes several simulation results to illustrate some of the theorems we give here regarding the behavior of multiple consecutive iterates of the biased algorithms.

The paper is organized as follows. In Section~\ref{sec-2} we provide the background for the ETD($\lambda$) algorithm. In Section~\ref{sec-3} we present our convergence results on constrained ETD($\lambda$) and several variants of it, and we give the proofs in Section~\ref{sec-4}. We conclude the paper in Section~\ref{sec-5} with a brief discussion on direct applications of our convergence results to the off-policy TD($\lambda$) algorithm and the ETD($\lambda, \beta$) algorithm, as well as to ETD($\lambda$) under relaxed conditions, followed by a discussion on several open issues.
In Appendix~\ref{appsec-a} we include the key properties of the ETD($\lambda$) trace iterates that are used in the analysis.

\section{Preliminaries} \label{sec-2}
In this section we describe the policy evaluation problem in the off-policy case, the ETD($\lambda$) algorithm and its constrained version, and we also review the results from our prior work \cite{yu-etdarx} that are needed in this paper. 

\subsection{Off-policy Policy Evaluation} \label{sec-2.1}

Let $\S = \{ 1, \ldots, N\}$ be a finite set of states, and let $\A$ be a finite set of actions. Without loss of generality we assume that for all states, every action in $\A$ can be applied. If $a \in \A$ is applied at state $s \in \S$, the system moves to state $s'$ with probability $p(s' \!\mid s, a)$ and yields a random reward with mean $r(s,a,s')$ and bounded variance, according to a probability distribution $q(\cdot \mid s, a, s')$. These are the parameters of the MDP model we consider; they are unknown to the learning algorithms to be introduced.

A \emph{stationary policy} is a time-invariant decision rule that specifies the probability of taking an action at each state. When actions are taken according to such a policy, the states and actions $(S_t, A_t)$ at times $t \geq 0$ form a (time-homogeneous) Markov chain on the space $\S \times \A$, with the marginal state process $\{S_t\}$ being also a Markov chain. 

Let $\pi$ and $\pi^o$ be two given stationary policies, with $\pi(a \!\mid s)$ and $\pi^o(a \!\mid s)$ denoting the probability of taking action $a$ at state $s$ under $\pi$ and $\pi^o$, respectively. While the system evolves under the policy $\pi^o$, generating a stream of state transitions and rewards, we wish to use these observations to evaluate the performance of the policy $\pi$, with respect to a discounted reward criterion, the definition of which will be given shortly. Here $\pi$ is the target policy and $\pi^o$ the behavior policy. It is allowed that $\pi^o \not= \pi$ (the off-policy case), provided that at each state, all actions taken by $\pi$ can also be taken by $\pi^o$ (cf.\ Assumption~\ref{cond-bpolicy}(ii) below).

Let $\gamma(s) \in [0,1]$, $s \in S$, be state-dependent discount factors, with $\gamma(s) < 1$ for at least one state.
We measure the performance of $\pi$ in terms of the expected discounted total rewards attained under $\pi$ as follows: for each state $s \in \S$, 
\begin{equation} \label{def-vpi}
 v_{\pi}(s) : = \E^\pi \left[ \, R_0 +  \sum_{t=1}^{\infty} \gamma(S_{1})\, \gamma(S_{2}) \, \cdots \, \gamma(S_{t})  \cdot R_t  \, \Big| \, S_0 = s \right],
\end{equation} 
where $R_t$ is the random reward received at time $t$, and $\E^\pi$ denotes expectation with respect to the probability distribution of the states, actions and rewards, $(S_t, A_t, R_t)$, $t \geq 0$, generated under the policy $\pi$.
The function $v_\pi$ on $\S$ is called the \emph{value function} of $\pi$. 
The special case of $\gamma$ being a constant less than $1$ corresponds to the $\gamma$-discounted reward criterion: $v_{\pi}(s)= \E^\pi \left[ \sum_{t=0}^{\infty} \gamma^t R_t  \mid  S_0 = s \right]$. In the general case, by letting $\gamma$ depend on the state, the formulation is able to also cover certain undiscounted total reward MDPs with termination;%footnote starts
\footnote{We may view $v_\pi(s)$ as the expected (undiscounted) total rewards attained under $\pi$ starting from the state $s$ and up to a random termination time $\tau \geq 1$ that depends on the states in a Markovian way. In particular, if at time $t \geq 1$, the state is $s$ and termination has not occurred yet, the probability of $\tau = t$ (terminating at time $t$) is $1 - \gamma(s)$. Then $v_\pi(s)$ can be equivalently written as $v_{\pi}(s) = \E^\pi \big[ \sum_{t=0}^{\tau-1} R_t  \mid S_0 = s \big]$.} 
%footnote ends
however, for $v_\pi$ to be well-defined (i.e., to have the right-hand side of (\ref{def-vpi}) well-defined for each state), a condition on the target policy is needed, which is stated below and will be assumed throughout the paper.

Let $\P$ denote the transition matrix of the Markov chain on $\S$ induced by $\pi$. Let $\Gm$ denote the $N \times N$ diagonal matrix with diagonal entries $\gamma(s), s \in \S$. 

%\smallskip
\begin{assumption}[conditions on the target and behavior policies] \label{cond-bpolicy} \hfill
\begin{enumerate}
\item[{\rm (i)}] The target policy $\pi$ is such that $(I - \P \Gm)^{-1}$ exists.
\item[{\rm (ii)}] The behavior policy $\pi^o$ induces an irreducible Markov chain on $\S$, and moreover, for all $(s,a) \in \mathcal{S} \times \mathcal{A}$, $\pi^o(a \!\mid s) > 0$ if $\pi(a \!\mid s) > 0$.
\end{enumerate}
\end{assumption}
%\smallskip

Under Assumption~\ref{cond-bpolicy}(i), the value function $v_\pi$ in (\ref{def-vpi}) is well-defined, and furthermore, $v_\pi$ satisfies uniquely the Bellman equation%footnote starts
\footnote{One can verify this Bellman equation directly. It also follows from the standard MDP theory, as by definition $v_\pi$ here can be related to a value function in a discounted MDP where the discount factors depend on state transitions, similar to discounted semi-Markov decision processes (see e.g., \cite{puterman94}).}
%footnote ends
$$v_\pi = r_{\pi} + \P \Gm \, v_\pi, \qquad \text{i.e.}, \quad v_\pi = (I - \P \Gm)^{-1} r_\pi,$$
where $r_{\pi}$ is the expected one-stage reward function under $\pi$ (i.e., $r_{\pi}(s) = \E^\pi [ R_0 \mid S_0=s ]$ for $s \in \S$).

\subsection{The ETD($\lambda$) Algorithm} \label{sec-2.2}

Like the standard TD($\lambda$) algorithm \cite{Sut88,tr-disc}, the ETD($\lambda$) algorithm \cite{SuMW14} approximates the value function $v_\pi$ by a function of the form $v(s) = \fe(s)^\top \theta$, $s \in \S$, using a parameter vector $\theta \in \rn$ and $n$-dimensional feature representations $\fe(s)$ for the states. (Here $\fe(s)$ is a column vector and $\text{}^\top$ stands for transpose.) In matrix notation, denote by $\Fe$ the $N \times n$ matrix with $\fe(s)^\top, s \in \S$, as its rows. Then the columns of $\Fe$ span the subspace of approximate value functions, and the approximation problem is to find in that subspace a function $v = \Fe \theta \approx v_{\pi}$.

We focus on a general form of the ETD($\lambda$) algorithm, which uses state-dependent $\lambda$ values specified by a function $\lambda: \S \to [0,1]$. 
Inputs to the algorithm are the states, actions and rewards, $\{(S_t, A_t, R_{t}) \}$, generated under the behavior policy $\pi^o$,
where $R_{t}$ is the random reward received upon the transition from state $S_t$ to $S_{t+1}$ with action $A_t$.
The algorithm can access the following functions, in addition to the features $\fe(s)$: 
\begin{enumerate}
\item[(i)] the state-dependent discount factor $\gamma(s)$ that defines $v_{\pi}$, as described earlier;
\item[(ii)] $\lambda: \S \to [0,1]$, which determines the single or multi-step Bellman equation for the algorithm (cf.\ the subsequent Eqs.~(\ref{eq-bellman})-(\ref{eq-bellman-def}) and Footnote~\ref{footnote-gbe}); 
\item[(iii)] $\rho: \mathcal{S} \times \mathcal{A} \to \re_+$ given by $\rho(s, a) = \pi(a \!\mid s)/ \pi^o(a \!\mid s)$ (with $0/0=0$),  which gives the likelihood ratios for action probabilities that can be used to compensate for sampling states and actions according to the behavior policy $\pi^o$ instead of the target policy $\pi$;
\item[(iv)] $\i: \mathcal{S} \to \re_+$, which gives the algorithm additional flexibility to weigh states according to the degree of ``interest'' indicated by $\i(s)$.
\end{enumerate}
The algorithm also uses a sequence $\alpha_t > 0, t \geq 0$, as stepsize parameters. We shall consider only deterministic $\{\alpha_t\}$.

To simplify notation, let
$$ \rho_t = \rho(S_t, A_t), \qquad \gamma_t = \gamma(S_t), \qquad \lambda_t = \lambda(S_t).   $$
ETD($\lambda$) calculates recursively $\theta_t \in \rn$, $t \geq 0$, according to
\begin{equation} \label{eq-emtd0}
  \theta_{t+1} = \theta_t + \alpha_t \, \e_t \cdot \rho_t \, \big( R_{t} + \gamma_{t+1} \fe(S_{t+1})^\top \theta_t - \fe(S_t)^\top \theta_t \big),
\end{equation}
where $\e_t \in \rn$, called the ``eligibility trace,'' is calculated together with two nonnegative scalar iterates $(\F_t, \M_t)$ according to%footnote starts
\footnote{The definition (\ref{eq-td1}) we use here differs slightly from the original definition of $\e_t$ in \cite{SuMW14}, but the two are equivalent and (\ref{eq-td1}) appears to be more convenient for our analysis.}
%footnote ends
\begin{align}
   \F_t & = \gamma_t \, \rho_{t-1}  \,  \F_{t-1} + \i(S_t),    \label{eq-td3} \\
   \M_t & = \lambda_t \, \i(S_t) + ( 1 - \lambda_t ) \, \F_t, \label{eq-td2} \\
   \e_t & =  \lambda_t \, \gamma_t \,  \rho_{t-1} \, \e_{t-1} + \M_t \, \fe(S_t). \label{eq-td1} 
\end{align}
For $t = 0$, $(\e_0, \F_0, \theta_0)$ are given as an initial condition of the algorithm.

We recognize that the iteration (\ref{eq-emtd0}) has the same form as TD($\lambda$), but the trace $\e_t$ is calculated differently, involving an ``emphasis'' weight $\M_t$ on the state $S_t$, which itself evolves along with the iterate $\F_t$, called the ``follow-on'' trace. If $\M_t$ is always set to $1$ regardless of $\F_t$ and $\i(\cdot)$, then the iteration (\ref{eq-emtd0}) reduces to the off-policy TD($\lambda)$ algorithm in the case where $\gamma$ and $\lambda$ are constants.

\subsection{Associated Bellman Equations and Approximation and Convergence Properties of ETD($\lambda$)} \label{sec-2.2b}

Let $\Lm$ denote the diagonal matrix with diagonal entries $\lambda(s), s \in \S$. Associated with ETD($\lambda$) is a generalized multistep Bellman equation of which $v_\pi$ is the unique solution \cite{Sut95}:%footnote starts
\footnote{\label{footnote-gbe}For the details of this Bellman equation, we refer the readers to the early work \cite{Sut95,SuB} and the recent work \cite{SuMW14}. 
We remark that similar to the standard one-step Bellman equation, which is a recursive relation that expresses $v_\pi$ in terms of the expected one-stage reward and the expected total future rewards given by $v_\pi$ itself, one can use the strong Markov property to derive other recursive relations satisfied by $v_\pi$, in which the expected one-stage reward is replaced by the expected rewards attained by $\pi$ up to some random stopping time. This gives rise to a general class of Bellman equations, of which (\ref{eq-bellman}) is one example. Earlier works on using such equations in TD learning include \cite{Sut95} and \cite[Chap.\ 5.3]{BET}. The recent work \cite{gen-td} considers an even broader class of Bellman equations using the concept of estimating equations from statistics, and the recent work \cite{yb-bellmaneq} focuses on a special class of generalized Bellman equations and discusses their potential advantages from an approximation viewpoint. But an in-depth study of the application of such equations is still lacking currently. Because generalized Bellman equations offer flexible ways to address the bias vs.\ variance problem in learning the value functions of a policy, they are especially important and deserve further study, in our opinion.} %footnote ends
\begin{equation} \label{eq-bellman}
  v = \rl + \PL \, v.
\end{equation}  
Here $\PL$ is an $N \times N$ substochastic matrix, $\rl \in \re^N $ is a vector of expected discounted total rewards attained by $\pi$ up to some random time depending on the function $\lambda$, and they can be expressed in terms of $\P$ and $\r$ as 
\begin{equation} \label{eq-bellman-def}
  \PL  = I - (I - \P \Gm \Lm)^{-1} \, (I - \P \Gm), \qquad \rl  = (I - \P \Gm \Lm)^{-1} \, \r.
\end{equation}  

ETD($\lambda$) aims to solve a projected version of the Bellman equation (\ref{eq-bellman}) \cite{SuMW14}, which takes the following forms in the space of approximate value functions and in the space of the $\theta$-parameters, respectively:
\begin{equation} \label{eq-proj-bellman}
       v = \Pi \big( \rl + \PL \, v \big),  \quad v \in \text{column-space}(\Fe),  \qquad \Longleftrightarrow \qquad C \theta + b = 0, \quad \theta \in \rn.
\end{equation}       
Here $\Pi$ is a projection onto the approximation subspace with respect to a weighted Euclidean norm or seminorm, under a condition on the approximation architecture that will be explained shortly. The weights that define this norm also define the diagonal entries $\bM_{ss}, s \in \S$, of a diagonal matrix $\bM$, which are given by
\begin{equation} \label{eq-m}
  diag(\bM)= \bi^{\top} (I - \PL)^{-1}, \qquad
\text{with} \  \ \ \bi \in \re^N, \, \bi(s) = d_{\pi^o}(s) \cdot  \i(s), \ s \in \S,
\end{equation}
where $d_{\pi^o}(s) > 0$ denotes the steady state probability of state $s$ for the behavior policy $\pi^o$, under Assumption~\ref{cond-bpolicy}(ii).
For the corresponding linear equation in the $\theta$-space in~(\ref{eq-proj-bellman}), 
\begin{align}  \label{eq-Ab}
    C    =  - \Fe^\top  \bM \, (I - \PL) \,  \Fe , \qquad \quad
    b   =  \Fe^\top  \bM \,  \rl.
\end{align}  

From the expression (\ref{eq-m}) of the diagonal matrix $\bM$, the most important difference between the earlier TD algorithms and ETD($\lambda$) can be seen. For on-policy TD($\lambda$), in stead of (\ref{eq-m}), the diagonal matrix $\bM$ is determined by the steady state probabilities of the states under the target policy $\pi$ under an ergodicity assumption \cite{tr-disc}, and for off-policy TD($\lambda$), it is determined by the steady state probabilities $d_{\pi^o}(s)$ under the behavior policy $\pi^o$. Here, due to the  emphatic weighting scheme (\ref{eq-td3})-(\ref{eq-td1}), the diagonals of $\bM$ given by (\ref{eq-m}) reflect the occupation frequencies (with respect to $\PL$) of the target policy rather than the behavior policy.

Let $| \cdot|$ denote the (unweighted) Euclidean norm. The matrix $C$ is said to be \emph{negative definite} if there exists $c > 0$ such that $\theta^\top C \theta \leq - c  | \theta |^2$ for all $\theta \in \rn$; and \emph{negative semidefinite} if in the preceding inequality $c = 0$. A salient property of ETD($\lambda$) is that the matrix $C$ is always negative semidefinite \cite{SuMW14}, and under natural and mild conditions, $C$ is negative definite. 
This is proved in \cite{yu-etdarx} and summarized below.

Call those states $s$ with $\bM_{ss} > 0$ \emph{emphasized states} (define this set of states to be empty if $\bM$ given by (\ref{eq-m}) is ill-defined, a case we will not encounter). 

\begin{assumption}[condition on the approximation architecture] \label{cond-features} 
The set of feature vectors of emphasized states, $\{\fe(s) \mid s \in \S,  \bM_{ss} > 0 \}$, contains $n$ linearly independent vectors.
\end{assumption}

\begin{thm}[{\cite[Prop.\ C.2]{yu-etdarx}}] \label{thm-matrix}
Under Assumption \ref{cond-bpolicy}, the matrix $C$ is negative definite if and only if Assumption~\ref{cond-features} holds. 
\end{thm}

Assumption~\ref{cond-features}, which implies the linear independence of the columns of $\Fe$, is satisfied in particular if the set of feature vectors, $\{\fe(s) \mid s \in \S,  \i(s) > 0 \}$, contains $n$ linearly independent vectors, since states with positive interest $\i(s)$ are among the emphasized states.%footnote starts
\footnote{This follows from the definition (\ref{eq-m}) of the diagonals $\bM_{ss}$. Since $(I - \PL)^{-1} = I + \sum_{k=1}^\infty (\PL)^k \geq I$, we have $diag(\bM)= \bi^{\top} (I - \PL)^{-1} \geq \bi^{\top}$. Hence $\i(s) > 0$ implies $\bM_{ss} \geq d_{\pi^o}(s) \cdot \i(s) > 0$.}
%footnote ends
So this assumption can be easily satisfied in reinforcement learning without model knowledge.\footnote{There is another way to verify Assumption~\ref{cond-features} without calculating $\bM$. Suppose ETD($\lambda$) starts from a state $S_0$ with $\i(S_0) > 0$. Then it can be shown that if $S_t=s$ and $\M_t > 0$, we must have $\bM_{ss} > 0$. This means that as soon as we find among states $S_t$ with emphasis weights $\M_t > 0$ $n$ states that have linearly independent feature vectors, we can be sure that Assumption~\ref{cond-features} is satisfied.}

In view of Theorem~\ref{thm-matrix}, under Assumptions~\ref{cond-bpolicy}-\ref{cond-features}, the equation $C \theta + b = 0$ has a unique solution $\theta^*$; equivalently, $\Fe \theta^*$ is the unique solution to the projected Bellman equation (\ref{eq-bellman-def}):
$$  \Fe \theta^* = \Pi \big( \rl + \PL \,  \Fe \theta^* \big), $$
where $\Pi$ is a well-defined projection operator that projects a vector in $\re^N$ onto the approximation subspace with respect to the seminorm on $\re^N$ 
given by 
$$  \sqrt{\textstyle{\sum}_{s \in \S}  \, \bM_{ss} \cdot v(s)^2}, \quad \forall \, v \in \re^N$$
(which is a norm if $\bM_{ss} > 0$ for all $s \in \S$). The relation between the approximate value function $v = \Fe \theta^*$ and the desired value function $v_\pi$, in particular, the approximation error, can be characterized by using the oblique projection viewpoint \cite{bruno-oblproj} for projected Bellman equations.%footnote starts
\footnote{Briefly speaking, \cite{bruno-oblproj} showed that the solutions of projected Bellman equations are oblique projections of $v_\pi$ on the approximation subspace. An oblique projection is defined by two nonorthogonal subspaces of equal dimensions and is the projection onto the first subspace orthogonally to the second \cite{saad}. In the special case of ETD($\lambda$), the first of these two subspaces is the approximation subspace $\{ v \in \re^N \mid v = \Phi \theta \ \text{for some} \ \theta \in \re^n\}$, and the second is the image of the approximation subspace under the linear transformation $(I - \PL)^{\top} \bM$. Essentially it is the angle between the two subspaces that determines the approximation bias $\Fe \theta^* - \Pi v_{\pi}$ in the worst case, for a worst-case choice of $\rl$. (For details, see also \cite[Sec.\ 2.2]{yb-bellmaneq}.) Recently, for the case of constant $\lambda, \i$ and $\gamma$, \cite{etd-errbd2} derived bounds on the approximation bias that are based on contraction arguments and are comparable to the bound for on-policy TD($\lambda$) \cite{tr-disc}. These bounds lie above the bounds given by the oblique projection view (cf.\ \cite{yb-errbd} and \cite[Sec.\ 2.2]{yb-bellmaneq}); however, they are expressed in terms of $\lambda$ and $\gamma$, so they give us explicit numbers instead of analytical expressions to bound the approximation bias.}
%footnote ends

The almost sure convergence of ETD($\lambda$) to $\theta^*$ is proved in \cite[Theorem 2.2]{yu-etdarx} under Assumptions~\ref{cond-bpolicy} and~\ref{cond-features}, for diminishing stepsize satisfying $\alpha_t = O(1/t)$ and $\tfrac{\alpha_t - \alpha_{t+1}}{\alpha_t} = O(1/t)$. 
Despite this convergence guarantee, the stepsize range is too narrow for applications, as we discussed in the introduction. 
In this paper we will focus on constrained ETD($\lambda$) algorithms that restrict the $\theta$-iterates in a bounded set, but can operate with much larger stepsizes and also suffer less from the issue of high variance in off-policy learning. We will analyze their behavior under Assumptions~\ref{cond-bpolicy} and~\ref{cond-features}, although our analysis extends to the case without Assumption~\ref{cond-features} (see the discussion in Section~\ref{sec-5.1}).

\subsection{Constrained ETD($\lambda$), Averaged Processes and Mean ODE} \label{sec-2.3}
We consider first a constrained version of ETD($\lambda$) that simply scales the $\theta$-iterates, if necessary, to keep them bounded:
\begin{equation} \label{eq-emtd-const0}
 \theta_{t+1} = \Pi_{\H} \Big( \theta_t + \alpha_t \, \e_t \cdot \rho_t \big(R_t + \gamma_{t+1} \fe(S_{t+1})^\top \theta_t - \fe(S_t)^\top \theta_t \big) \Big),
\end{equation}
where $\Pi_\H$ is the Euclidean projection onto a closed ball $\H \subset \rn$ at the origin with radius $r_\H$: $\H = \{ \theta \in \re^n \mid | \theta | \leq r_\H \}$. Under Assumptions~\ref{cond-bpolicy} and~\ref{cond-features}, when the radius $r_\H$ is sufficiently large (greater than the threshold given in Lemma~\ref{lma-pode} below), from any given initial $(\e_0, \F_0, \theta_0)$, the algorithm (\ref{eq-emtd-const0}) converges almost surely to $\theta^*$, for diminishing stepsize $\alpha_t = O(1/t)$ \cite[Theorem 4.1]{yu-etdarx}. 

Our interest in this paper is to apply (\ref{eq-emtd-const0}) with a much larger range of stepsize, in particular, constant stepsize or stepsize that diminishes much more slowly than $O(1/t)$. In Sections~\ref{sec-3} and \ref{sec-4}, we will analyze the algorithm (\ref{eq-emtd-const0}) and its two variants for such stepsizes. 
To prepare for the analysis, in the rest of this section, 
we review several results from \cite{yu-etdarx} that will be needed.
 
First, we discuss about the ``mean ODE'' that we wish to associate with (\ref{eq-emtd-const0}).
It is the projected ODE
\begin{equation} \label{eq-pode}
  \dot{x} = \bar h(x) + z, \qquad z \in - \mathcal{N}_\H(x),
\end{equation}  
where the function $\bar h$ is the left-hand side of the equation $Cx + b = 0$ we want to solve:
\begin{equation} \label{eq-bh}
 \bar h(x) = C x + b;
\end{equation} 
$\mathcal{N}_\H(x)$ is the normal cone of $\H$ at $x$, i.e., 
$$\mathcal{N}_\H(x) = \{0\} \ \  \text{for $x$ in the interior of $\H$}, \quad  \mathcal{N}_\H(x) = \{a x \mid a \geq 0\} \ \  \text{for $x$ on the boundary of $\H$};$$ 
and $z$ is the boundary reflection term that cancels out the component of $\bar h(x)$ in $\mathcal{N}_\H(x)$ (i.e., $z = - y$ where $y$ is the projection of $\bar h(x)$ on $\mathcal{N}_\H(x)$), and it is the ``minimal force'' needed to keep the solution $x(\cdot)$ of (\ref{eq-pode}) in $\H$ \cite[Chap.\ 4.3]{KuY03}.

The negative definiteness of the matrix $C$ ensures that when the radius of $\H$ is sufficiently large, the boundary reflection term is zero for all $x \in \H$ and the projected ODE (\ref{eq-pode}) has no stationary points other than $\theta^*$ (see \cite[Sec.\ 4.1]{yu-etdarx} for a simple proof):

\begin{lem}\label{lma-pode}
Let $c > 0$ be such that $x^\top C x \leq - c | x |^2$ for all $x \in \rn$. Suppose $\H$ has a radius $r_B >  | b |/c$. Then $\theta^*$ lies in the interior of $\H$; a solution $x(\tau), \tau \in [0, \infty)$, to the projected ODE (\ref{eq-pode}) for an initial condition $x(0) \in \H$ coincides with the unique solution to $\dot{x} = \bar h(x)$, with the boundary reflection term being $z(\cdot) \equiv 0$;
and the only solution $x(\tau), \tau \in (-\infty, +\infty)$, of (\ref{eq-pode}) in $\H$ is $x(\cdot) \equiv \theta^*$.
\end{lem}

Informally speaking, suppose we have proved that (\ref{eq-pode}) is the mean ODE for the algorithm (\ref{eq-emtd-const0}) under stepsizes of our interest.
Then applying powerful convergence theorems from the stochastic approximation theory \cite{KuY03}, we can assert that the iterates $\theta_t$ will eventually ``follow closely'' a solution of the mean ODE. This together with the solution property of the mean ODE given in Lemma~\ref{lma-pode} will then give us a characterization of the asymptotic behavior of the algorithm (\ref{eq-emtd-const0}) for a constraint set $\H$ with sufficiently large radius.

Several properties of the ETD($\lambda$) iterates will be important in proving that (\ref{eq-pode}) is indeed the mean ODE for (\ref{eq-emtd-const0}) and reflects its average dynamics. We now discuss two such properties (other key properties will be given in Appendix~\ref{appsec-a}). They concern the ergodicity of the Markov chain $\{(S_t, A_t, \e_t, \F_t)\}$ on the joint space of states, actions and traces, and the convergence of certain averaged sequences associated with the algorithm (\ref{eq-emtd-const0}). They will also be useful in analyzing variants of (\ref{eq-emtd-const0}).

Let $\Z_t = (S_t, A_t, \e_t, \F_t)$, $t \geq 0$.
It is shown in~\cite{yu-etdarx} that under Assumption~\ref{cond-bpolicy}, $\{\Z_t\}$ is a weak Feller Markov chain\footnote{See Section~\ref{sec-4.3.1} or the book \cite[Chap.\ 6]{MeT09} for the definition and properties of weak Feller Markov chains.} 
on the infinite state space $\S \times \A \times \re^{n+1}$ and is ergodic. 
Specifically, on a metric space, a sequence of probability measures $\{\mu_t\}$ is said to \emph{converge weakly} to a probability measure $\mu$ if for any bounded continuous function $f$, $\int f d\mu_t \to \int f d\mu$ as $t \to \infty$ \cite[Chap.\ 9.3]{Dud02}.
We are interested in the weak convergence of the occupation probability measures of the process $\{\Z_t\}$,
where for each initial condition $\Z_0=z$, the \emph{occupation probability measures} $\mu_{z, t}$, $t \geq 0$, are defined by
$\mu_{z,t}(D) =  \frac{1}{t+1} \sum_{k=0}^{t} \mathbb{1}(\Z_k \in D)$ for any Borel subset $D$ of $\mathcal{S} \times \mathcal{A} \times \re^{n+1}$, 
with $\mathbb{1}(\cdot)$ denoting the indicator function.

%\smallskip
\begin{thm}[ergodicity of $\{\Z_t\}$; {\cite[Theorem 3.2]{yu-etdarx}}] \label{thm-2.1}
Under Assumption~\ref{cond-bpolicy}, the Markov chain $\{\Z_t\}$ has a unique invariant probability measure $\zeta$, and for each initial condition $\Z_0=z$, the sequence $\{\mu_{z,t}\}$ of occupation probability measures converges weakly to $\zeta$, almost surely.
\end{thm}

Let $\E_{\zeta}$ denote expectation with respect to the stationary process $\{\Z_t\}$ with $\zeta$ as its initial distribution.
By the definition of weak convergence, the weak convergence of $\{\mu_{z,t}\}$ given in Theorem~\ref{thm-2.1} implies that for each given initial condition of $\Z_0$,  the averages $\tfrac{1}{t} \sum_{k = 0}^{t-1} f(\Z_k)$ converge almost surely to $\E_\zeta \{ f(\Z_0)\}$ for any bounded continuous function $f$.%footnote starts
\footnote{With the usual discrete topology for the finite space $\S \times \A$ and the usual topology for the Euclidean space $\re^{n+1}$, the space $\S \times \A \times \re^{n+1}$ equipped with the product topology is metrizable. A continuous function $f(s, a, \e, \F)$ on this space is a function that is continuous in $(\e,\F)$ for each $(s, a) \in \S \times \A$.}
%footnote ends
To study the average dynamics of the algorithm (\ref{eq-emtd-const0}), however, we need to also consider unbounded functions.
In particular, the function related to both (\ref{eq-emtd-const0}) and the unconstrained ETD($\lambda$) is $h : \rn \times \Xi \to \rn$, 
\begin{equation} \label{eq-h}
   h(\theta, \xi) =  \e \cdot \rho(s, a) \, \big( r(s, a, s') + \gamma(s') \, \fe(s')^\top \theta - \fe(s)^\top \theta \big),  
\end{equation} 
where 
$$\xi = (\e, \F, s, a, s') \in \Xi: = \re^{n+1} \times \S \times \A \times \S.$$
Writing $\xi_t$ for the traces and transition at time $t$: $\xi_t = (\e_t, \F_t, S_t, A_t, S_{t+1})$, we can express the recursion (\ref{eq-emtd-const0}) equivalently as
\begin{equation} \label{eq-emtd-const}
 \theta_{t+1} = \Pi_{\H} \big( \theta_t + \alpha_t \, h(\theta_t, \xi_t) + \alpha_t \, \e_t \cdot \tilde \omega_{t+1} \big),
\end{equation} 
where $\tilde \omega_{t+1}  = \rho_t \, (R_{t} - r(S_t, A_t, S_{t+1}))$ is the noise part of the observed reward.

The convergence to $\bar h(\theta)$ of the averaged sequence $\tfrac{1}{t} \sum_{k=0}^{t-1} h(\theta, \xi_k)$, with $\theta$ held fixed and $t$ going to infinity, will be needed to prove that (\ref{eq-pode}) is the mean ODE of (\ref{eq-emtd-const0}). Since $\bar h(\theta) = C \theta + b$, this convergence for each fixed $\theta$ can be identified with the convergence of the matrix and vector iterates calculated by ELSTD($\lambda$) (the least-squares version of ETD($\lambda$)) to approximate the left-hand side of the equation $C \theta + b = 0$. It was proved in \cite{yu-etdarx} as a special case of the convergence of averaged sequences for a larger set of functions including $h(\theta, \cdot)$. Since this general result will be needed in analyzing variants of (\ref{eq-emtd-const0}), we give its formulation here. 

Throughout the rest of the paper, we let $\| \cdot \|$ denote the infinity norm of a Euclidean space, and we use this notation for both vectors and matrices (viewed as vectors). 
For $\re^m$-valued random variables $X_t$, we say $\{X_t\}$ converges to a random variable $X$ in mean if $\E [ \| X_t - X \| ] \to 0$ as $t \to \infty$. 

Consider a vector-valued function $g : \Xi \to \re^m$ such that with $\xi = (\e, \F, s, a, s')$, $g(\xi)$ is Lipschitz continuous in $(\e, \F)$ uniformly in $(s,a,s')$. That is, 
there exists a finite constant $L_g$ such that for any $(\e, \F), (\hat \e, \hat \F) \in \re^{n+1}$,
\begin{equation} \label{cond-lip}
 \big\| g( \e, \F, s, a, s') - g(\hat \e, \hat \F, s, a, s') \big\| \leq L_g  \big\| (\e, \F) - (\hat \e, \hat \F) \big\|, \qquad \forall \,   (s, a, s') \in \mathcal{S} \times \mathcal{A} \times \mathcal{S}.
\end{equation} 
For each $\theta \in \rn$, the function $h(\theta, \cdot)$ in (\ref{eq-h}) is a special case of $g$. 
The convergence of the averaged sequence $\tfrac{1}{t} \sum_{k=0}^{t-1} g(\xi_k)$ is given in the theorem below; the part on convergence in mean will be used frequently later in this paper (and was actually also needed in \cite{yu-etdarx} to prove the ergodicity of $\{Z_t\}$ given earlier). The convergence of $\tfrac{1}{t} \sum_{k=0}^{t-1} h(\theta, \xi_k)$ then follows as a special case.

%\smallskip
\begin{thm}[convergence of averaged sequences; {\cite[Theorems 3.1-3.3]{yu-etdarx}}] \label{thm-2.2}
Let $g$ be a vector-valued function satisfying the Lipschitz condition~(\ref{cond-lip}). Then under Assumption~\ref{cond-bpolicy}, $\E_{\zeta} \big[ \| g(\xi_0) \| \big] < \infty$ and
for any given initial $(\e_0,\F_0) \in \re^{n+1}$, as $t \to \infty$, $\tfrac{1}{t} \sum_{k=0}^{t-1} g(\xi_k)$ converges to $\bar g = \E_\zeta \big[g(\xi_0) \big]$ in mean and almost surely. 
\end{thm}
%\smallskip

\begin{cor}[{\cite[Theorem 2.1]{yu-etdarx}}] \label{cor-2.1}
Under Assumption~\ref{cond-bpolicy}, for the functions $\bar h, h$ given in (\ref{eq-bh}), (\ref{eq-h}) respectively, the following hold:
For each $\theta \in \rn$, $\E_{\zeta} \big[ \| h(\theta, \xi_0) \| \big] < \infty$ and $\bar h(\theta) = \E_\zeta \big[h(\theta, \xi_0) \big]$; and
for any given initial $(\e_0,\F_0) \in \re^{n+1}$, as $t \to \infty$, $\tfrac{1}{t} \sum_{k=0}^{t-1} h(\theta, \xi_k)$ converges to $\bar h(\theta)$ in mean and almost surely. 
\end{cor}

\section{Convergence Results for Constrained ETD($\lambda$)} \label{sec-3}

In this section we present the convergence properties of the constrained ETD($\lambda$) algorithm (\ref{eq-emtd-const0}) and several variants of it, for constant stepsize and for stepsize that diminishes slowly. We will explain briefly how the results are obtained, leaving the detailed analyses to Section~\ref{sec-4}. The first set of results about the algorithm (\ref{eq-emtd-const0}) will be given first in Section~\ref{sec-3.1}, followed by similar results in Section~\ref{sec-3.2} for two variant algorithms that have biases but can mitigate the variance issue in off-policy learning better. These results are obtained through applying two general convergence theorems from \cite{KuY03}, which concern weak convergence of stochastic approximation algorithms for diminishing and constant stepsize. Finally, the constant-stepsize case will be analyzed further in Section~\ref{sec-3.3}, in order to refine some results of Sections~\ref{sec-3.1}-\ref{sec-3.2} so that the asymptotic behavior of the algorithms for a fixed stepsize can be characterized explicitly. In that subsection, besides the three algorithms just mentioned, we will also discuss another variant algorithm with perturbation.

Regarding notation, recall that $\I(\cdot)$ is the indicator function, $|\cdot|$ stands for the usual (unweighted) Euclidean norm and $\| \cdot \|$ the infinity norm for $\re^m$. 
We denote by $N_\delta(D)$ the $\delta$-neighborhood of a set $D \subset \re^m$: $N_\delta(D) = \{ x  \in \re^m \mid \inf_{y \in D} | x - y | \leq \delta \}$, and we write $N_\delta(\theta^*)$ for the $\delta$-neighborhood of $\theta^*$. For the iteration index $t$, the notation $t \in [k_1, k_2]$ or $t \in [k_1, k_2)$ will be used to mean that the range of $t$ is the set of integers in the interval $[k_1, k_2]$ or $[k_1, k_2)$.
More definitions and notation will be introduced later where they are needed.

\subsection{Main Results} \label{sec-3.1}

We consider first the algorithm (\ref{eq-emtd-const0}) for diminishing stepsize. Let the stepsize change slowly in the following sense.

\begin{assumption}[condition on diminishing stepsize] \label{cond-large-stepsize} 
The (deterministic) nonnegative sequence $\{\alpha_t\}$ satisfies that $\sum_{t \geq 0} \alpha_t = \infty$, $\alpha_t \to 0$ as $t \to \infty$, and for some sequence of integers $m_t \to \infty$,
\begin{equation} \label{cond-w-stepsize}
\lim_{t \to \infty} \, \sup_{0 \leq j \leq m_t} \left| \frac{\alpha_{t+j}}{\alpha_t} - 1 \right| = 0.
\end{equation}
\end{assumption}

%\medskip
The condition~(\ref{cond-w-stepsize}) is the condition A.8.2.8 in \cite[Chap.\ 8]{KuY03} and allows stepsizes much larger than $O(1/t)$. 
We can have $\alpha_t =  O(t^{-\beta})$, $\beta \in (0,1]$, and even larger stepsizes are possible. For example, partition the time interval $[0, \infty)$ into increasingly longer intervals $I_k, k \geq 0,$ and set $\alpha_t$ to be constant within each interval $I_k$. Then the condition~(\ref{cond-w-stepsize}) can be fulfilled by letting the constants for each $I_k$ decrease as $O(k^{-\beta})$, $\beta \in (0,1]$.

We now state the convergence result. For any $T > 0$, let 
$m(k, T) = \min \{ t \geq k \mid \sum_{j=k}^{t+1} \alpha_j > T \}.$
If we draw a continuous timeline and put each iteration of the algorithm at a specific moment, with the stepsize $\alpha_j$ being the length of time between iterations $j$ and $j+1$, then $m(k,T)$ is the latest iteration before time $T$ has elapsed since the $k$-th iteration. If $\alpha_t =  O(t^{-\beta})$, $\beta \in (0,1]$, for example, then for fixed $T$, 
there are $O(k^\beta)$ iterates between the $k$-th and $m(k,T)$-th iteration.

%\smallskip
\begin{thm}[convergence properties of constrained ETD with diminishing stepsize] \label{thm-dim-stepsize}
Suppose Assumptions~\ref{cond-bpolicy},~\ref{cond-features} hold and the radius of $\H$ exceeds the threshold given in Lemma~\ref{lma-pode}.
Let $\{\theta_t\}$ be generated by the algorithm (\ref{eq-emtd-const0}) with stepsize $\{\alpha_t\}$ satisfying Assumption~\ref{cond-large-stepsize}, from any given initial condition $(\e_0, \F_0)$.
Then there exists a sequence $T_k \to \infty$ such that for any $\delta > 0$,
$$ \limsup_{k \to \infty} \, \Pr \Big( \, \theta_t \not\in N_\delta(\theta^*), \, \text{some} \  t \in \big[ \, k , \, m(k, T_k) \, \big] \Big) = 0.$$
\end{thm}
\smallskip

This theorem implies $\theta_t \to \theta^*$ in probability. Since $\{\theta_t\}$ is bounded, by \cite[Theorem 10.3.6]{Dud02}, $\theta_t$ must also converge to $\theta^*$ in mean:

\begin{cor}[convergence in mean] \label{cor-dim-stepsize}
In the setting of Theorem~\ref{thm-dim-stepsize}, $\E \big[ \| \theta_t - \theta^* \| \big] \to 0$ as $t \to \infty$.
\end{cor}

Another important note is that the conclusion of Theorem~\ref{thm-dim-stepsize} is much stronger than that $\theta_t \to \theta^*$ in probability. Here as $k \to \infty$, we consider an increasingly longer segment $[k, m(k, T_k)]$ of iterates, and are able to conclude that the probability of that \emph{entire segment} being inside an arbitrarily small neighborhood of $\theta^*$ approaches $1$. (This is the power of the weak convergence methods \cite{KuC78,KuS84a,KuY03}, by which our conclusion is obtained.)

In the case of constant stepsize, we consider all the trajectories that can be produced by the algorithm (\ref{eq-emtd-const0}) using some constant stepsize, and we ask what the properties of these trajectories are in the limit as the stepsize parameter approaches $0$. Here there is a common timeline used in relating trajectories generated with different stepsizes (and it comes from the ODE-based analysis): we imagine again a continuous timeline, along which we put the iterations at moments that are evenly separated in time by $\alpha$, if the stepsize parameter is $\alpha$. The scalars $T, T_\alpha$ in the theorem below represent amounts of time with respect to this continuous timeline. 

\begin{thm}[convergence properties of constrained ETD with constant stepsize] \label{thm-const-stepsize}
Suppose Assumptions~\ref{cond-bpolicy},~\ref{cond-features} hold and the radius of $\H$ exceeds the threshold given in Lemma~\ref{lma-pode}. 
For each $\alpha > 0$, let $\{\theta_t^{\alpha}\}$ be generated by the algorithm (\ref{eq-emtd-const0}) with constant stepsize $\alpha$, from any given initial condition $(\e_0, \F_0)$. Let $\{k_\alpha \mid \alpha > 0\}$ be any sequence of nonnegative integers that are nondecreasing as $\alpha \to 0$.
Then the following hold:
\begin{enumerate}
\item[\rm (i)] For any $\delta > 0$, 
$$ \lim_{T \to \infty} \lim_{\alpha \to 0} \,  \frac{1}{T/\alpha} \sum_{t=k_\alpha}^{k_\alpha+ \lfloor T/\alpha \rfloor}  \I\big( \theta_{t}^\alpha \in N_\delta(\theta^*) \big)  = 1 \quad \text{in probability}.$$ 
\item[\rm (ii)] Let $\alpha k_\alpha \to \infty$ as $\alpha \to 0$. Then there exists a sequence $\{T_\alpha \mid \alpha > 0\}$ with $T_\alpha \to \infty$ as $\alpha \to 0$, such that for any $\delta > 0$,
$$ \limsup_{\alpha \to 0} \, \Pr \Big( \, \theta_t^\alpha \not\in N_\delta(\theta^*), \, \text{some} \  t \in \big[ \, k_\alpha , \,  k_\alpha + T_\alpha/\alpha \, \big] \Big) = 0.$$
\end{enumerate}
\end{thm}
\smallskip

Part (ii) above is similar to Theorem~\ref{thm-dim-stepsize}. Here as $\alpha \to 0$, an increasingly longer segment $[k_\alpha,  k_\alpha + T_\alpha/\alpha]$ of the tail of the trajectory $\{\theta_t^\alpha\}$ is considered, and it is concluded that the probability of that \emph{entire segment} being inside an arbitrarily small neighborhood of $\theta^*$ approaches $1$.
Part (i) above, roughly speaking, says that as $\alpha$ diminishes, within the segment $[k_\alpha,  k_\alpha + T/\alpha]$, the fraction of iterates $\theta_t^\alpha$ that lie in a small $\delta$-neighborhood of $\theta^*$ approaches $1$ for sufficiently large $T$.

We give the proofs of Theorems~\ref{thm-dim-stepsize}-\ref{thm-const-stepsize} in Section~\ref{sec-4.1}. As mentioned earlier, most of our efforts will be to use the properties of ETD iterates to show that the conditions of two general convergence theorems from stochastic approximation theory \cite[Theorems 8.2.2, 8.2.3]{KuY03} are satisfied by the algorithm (\ref{eq-emtd-const0}). 
After that we can specialize the conclusions of those theorems to obtain Theorems~\ref{thm-dim-stepsize}-\ref{thm-const-stepsize}. 
Specifically, after furnishing their conditions, applying \cite[Theorems 8.2.2, 8.2.3]{KuY03} will give us directly the desired conclusions in Theorems~\ref{thm-dim-stepsize}-\ref{thm-const-stepsize} with $N_\delta(L_\H)$ in place of $N_\delta(\theta^*)$, where $N_\delta(L_\H)$ is the $\delta$-neighborhood of the \emph{limit set} $L_\H$ for the projected ODE (\ref{eq-pode}).
This limit set is defined as follows:
$$ L_\H : = \cap_{\bar \tau > 0} \, \overline{\, \cup_{x(0) \in \H} \{ x(\tau), \, \tau \geq \bar \tau \}}$$
where $x(\tau)$ is a solution of the projected ODE (\ref{eq-pode}) with initial condition $x(0)$, the union is over all the solutions with initial $x(0) \in \H$, and $\overline{D}$ for a set $D$ denotes taking the closure of $D$. It can be shown that $L_\H = \{\theta^*\}$ under our assumptions, so Theorems~\ref{thm-dim-stepsize}-\ref{thm-const-stepsize} will then follow as special cases of \cite[Theorems 8.2.2, 8.2.3]{KuY03}.

\smallskip
\begin{rem}[on weak convergence methods]
The theorems from \cite{KuY03} which we will apply are based on the weak convergence methods. While it is beyond the scope of this paper to explain these powerful methods, let us mention here a few basic facts about them to elucidate the origin of the convergence theorems we gave above. 
In the framework of \cite{KuY03}, one studies a trajectory of iterates produced by an algorithm by working with continuous-time processes that are piecewise constant or linear interpolations of the iterates. (Often one also left-shifts a trajectory of iterates to bring the ``asymptotic part'' of the trajectory closer to the origin of the continuous time axis.) In the case of our problem, for example, 
for diminishing stepsize, these continuous-time processes are $x^k(\tau), \tau \in [0, \infty)$, indexed by $k \geq 0$, where for each $k$, $x^k$ is a piecewise constant interpolation of $\theta_{k+t}, t \geq 0$, given by $x^k(\tau) = \theta_{k}$ for $\tau \in [0, \alpha_k)$ and $x^k(\tau) = \theta_{k+t}$ for $\tau \in [\sum_{m=0}^{t-1} \alpha_{k+m}, \sum_{m=0}^{t} \alpha_{k+m})$, $t \geq 1$. Similarly, for constant stepsize, the continuous-time processes involved are $x^\alpha(\tau), \tau \in [0, \infty)$, indexed by $\alpha > 0$, and for each $\alpha$, $x^\alpha$ is a piecewise constant interpolation of $\theta_{k_\alpha+t}^\alpha, t \geq 0$, given by $x^\alpha(\tau) = \theta_{k_\alpha+t}$ for $\tau \in [t \alpha,  (t+1) \alpha)$. 
The behavior of the sequence $\{x^k\}$ or $\{x^\alpha\}$ as $k \to \infty$ or $\alpha \to 0$, tells us the asymptotic properties of the algorithm as the number of iterations grows to infinity or as the stepsize parameter approaches $0$.
With the weak convergence methods, one considers the probability distributions of the continuous-time processes in such sequences, and analyze the convergence of these probability distributions and their limiting distributions along any subsequences. 
Here each continuous-time process takes values in a space of vector-valued functions on $[0, \infty)$ or $(-\infty, \infty)$ that are right-continuous and have left-hand limits, and this function space equipped with an appropriate metric, known as the Skorohod metric, is a complete separable metric space \cite[p.\ 238-240]{KuY03}. 
On this space, one analyzes the weak convergence of the probability distributions of the continuous-time processes. 
Under certain conditions on the algorithm, the general conclusions from \cite[Theorems 8.2.2, 8.2.3]{KuY03} are that any subsequence of these probability distributions contains a further subsequence which is convergent, and that all the limiting probability distributions must assign the full measure $1$ to the set of solutions of the mean ODE associated with the algorithm. This general weak convergence property then yields various conclusions about the asymptotic behavior of the algorithm and its relation with the mean ODE solutions. When further combined with the solution properties of the mean ODE, it leads to specific results such as the theorems we give in this section. \qed
\end{rem}

\subsection{Two Variants of Constrained ETD($\lambda$) with Biases} \label{sec-3.2}

We now consider two simple variants of (\ref{eq-emtd-const0}). They constrain the ETD iterates even more, at a price of introducing biases in this process, so that unlike (\ref{eq-emtd-const0}), they can no longer get to $\theta^*$ arbitrarily closely. Instead they aim at a small neighborhood of $\theta^*$, the size of which depends on how they modify the ETD iterates. On the other hand, because the trace iterates $\{(\e_t, \F_t)\}$ can have unbounded variances and are also naturally unbounded in common off-policy situations (see discussions in \cite[Prop.\ 3.1 and Footnote 3, p.~3320-3322]{Yu-siam-lstd} and \cite[Remark A.1, p.~23]{yu-etdarx}), these variant algorithms have the advantage that they make the $\theta$-iterates more robust against the drastic changes that can occur to the trace iterates. Indeed our definition of the variant algorithms below follows a well-known approach to ``robustifying'' algorithms in stochastic approximation theory (see discussions in \cite[p.\ 23 and p.\ 141]{KuY03}).

The two variant algorithms are defined as follows.
For each $K > 0$, let $\psi_K : \rn \to \rn$ be a bounded Lipschitz continuous function such that
\begin{equation} \label{eq-psi}
  \| \psi_K(x) \| \leq \| x \| \ \  \forall \, x \in \rn, \quad \text{and} \quad \psi_K(x) = x \ \ \text{if} \ \| x \| \leq K. 
\end{equation}
(For instance, let $\psi_K(x) = \bar r x/|x|$ if $|x| \geq \bar r$ and $\psi_K(x) = x$ otherwise, for $\bar r = \sqrt{n} K$; or let $\psi_K(x)$ be the result of truncating each component of $x$ to be within $[-K, K]$.) 
For the first variant of the algorithm (\ref{eq-emtd-const0}), we replace $\e_t$ in (\ref{eq-emtd-const0}) by $\psi_K(\e_t)$:
\begin{equation} \label{eq-emtd-const1}
 \theta_{t+1} = \Pi_{\H} \Big( \theta_t + \alpha_t \, \psi_K(\e_t) \cdot \rho_t \big(R_t + \gamma_{t+1} \fe(S_{t+1})^\top \theta_t - \fe(S_t)^\top \theta_t \big) \Big).
\end{equation}
For the second variant, we apply $\psi_K$ to bound the entire increment in (\ref{eq-emtd-const0}) before it is multiplied by the stepsize $\alpha_t$ and added to $\theta_t$:
\begin{equation} \label{eq-emtd-const2}
 \theta_{t+1} = \Pi_{\H} \left( \theta_t + \alpha_t \, \psi_K (Y_t) \right), \quad \text{where} \ \ Y_t = \e_t \cdot \rho_t \big(R_t + \gamma_{t+1} \fe(S_{t+1})^\top \theta_t - \fe(S_t)^\top \theta_t \big).
\end{equation}

As will be proved later, these two algorithms are associated with mean ODEs of the form,
\begin{equation} \label{eq-podeK0}
  \dot{x} = \bar h_K(x) + z, \qquad z \in - \mathcal{N}_\H(x),
\end{equation} 
where $\bar h_K: \rn \to \rn$ is determined by each algorithm and deviates from the function $\bar h(x) = C x + b$ due to the alterations introduced by $\psi_K$. This ODE is similar to the projected ODE (\ref{eq-pode}), except that since $\bar h_K$ is an approximation of $\bar h$, $\theta^*$ is no longer a stable or stationary point for the mean ODE (\ref{eq-podeK0}).
The two variant algorithms thus have a bias in their $\theta$-iterates, and the bias can be made smaller by choosing a larger $K$. 
This is reflected in the two convergence theorems given below.
They are similar to the previous two theorems for the algorithm (\ref{eq-emtd-const0}), except that now given a desired small neighborhood of $\theta^*$, a sufficiently large $K$ needs to be used in order for the $\theta$-iterates to reach that neighborhood of $\theta^*$ and exhibit properties similar to those shown in the previous case.

%\smallskip
\begin{thm}[convergence properties of constrained ETD variants with diminishing stepsize] \label{thm-dim-stepsize-b}
In the setting of Theorem~\ref{thm-dim-stepsize}, let $\{\theta_t\}$ be generated instead by the algorithm (\ref{eq-emtd-const1}) or (\ref{eq-emtd-const2}), with a bounded Lipschitz continuous function $\psi_K$ satisfying (\ref{eq-psi}), and with stepsize $\{\alpha_t\}$ satisfying Assumption~\ref{cond-large-stepsize}.
Then for each $\delta > 0$, there exists $K_\delta > 0$ such that if $K \geq K_\delta$, then
it holds for some sequence $T_k \to \infty$ that
$$ \limsup_{k \to \infty} \, \Pr \Big( \, \theta_t \not\in N_\delta(\theta^*), \, \text{some} \  t \in \big[ \, k , \, m(k, T_k) \, \big] \Big) = 0.$$
\end{thm}

%\smallskip
\begin{thm}[convergence properties of constrained ETD variants with constant stepsize] \label{thm-const-stepsize-b}
In the setting of Theorem~\ref{thm-const-stepsize}, let $\{\theta^\alpha_t\}$ be generated instead by the algorithm (\ref{eq-emtd-const1}) or (\ref{eq-emtd-const2}), with a bounded Lipschitz continuous function $\psi_K$ satisfying (\ref{eq-psi}) and with constant stepsize $\alpha > 0$.
Let $\{k_\alpha \mid \alpha > 0\}$ be any sequence of nonnegative integers that are nondecreasing as $\alpha \to 0$.
Then for each $\delta > 0$, there exists $K_\delta > 0$ such that the following hold if $K \geq K_\delta$:
\begin{enumerate}
\item[\rm (i)]
$$ \lim_{T \to \infty} \lim_{\alpha \to 0} \,  \frac{1}{T/\alpha} \sum_{t=k_\alpha}^{k_\alpha+ \lfloor T/\alpha \rfloor}  \I\big( \theta_{t}^\alpha  \in N_\delta(\theta^*) \big)  = 1 \quad \text{in probability}.$$ 
\item[\rm (ii)] Let $\alpha k_\alpha \to \infty$ as $\alpha \to 0$. Then there exists a sequence $\{T_\alpha \mid \alpha > 0\}$ with $T_\alpha \to \infty$ as $\alpha \to 0$, such that
$$ \limsup_{\alpha \to 0} \, \Pr \Big( \, \theta_t^\alpha \not\in N_\delta(\theta^*), \, \text{some} \  t \in \big[ \, k_\alpha , \,  k_\alpha + T_\alpha/\alpha \, \big]  \Big) = 0.$$
\end{enumerate}
\end{thm}
\smallskip

We give the proofs of the above two theorems in Section~\ref{sec-4.2}. 
The arguments are largely the same as those that we will use first in Section~\ref{sec-4.1} to prove Theorems~\ref{thm-dim-stepsize}-\ref{thm-const-stepsize} for the algorithm (\ref{eq-emtd-const0}).
Indeed, for all the three algorithms, the main proof step is the same, which is to apply the general conclusions of \cite[Theorems 8.2.2, 8.2.3]{KuY03} to establish the connection between the iterates of an algorithm and the solutions of an associated mean ODE, and this step does not concern what the solutions of the ODE are actually. (For the two variant algorithms, verifying that the conditions of \cite[Theorems 8.2.2, 8.2.3]{KuY03} are met is, in fact, easier than for the algorithm (\ref{eq-emtd-const0}), because various functions involved in the analysis become bounded due to the use of the bounded function $\psi_K$.) 
For the two variant algorithms, the result of this step is that the same conclusions given in Theorems~\ref{thm-dim-stepsize}-\ref{thm-const-stepsize} hold with $N_\delta(L_\H)$ in place of $N_\delta(\theta^*)$, where $L_\H$ is the limit set of the projected mean ODE (\ref{eq-podeK0}) associated with each variant algorithm. 
To attain Theorems~\ref{thm-dim-stepsize-b}-\ref{thm-const-stepsize-b}, we then combine this with the fact that by choosing $K$ sufficiently large, one can make the limit set $L_\H \subset N_\delta(\theta^*)$ for an arbitrarily small $\delta$.

\subsection{More about the Constant-stepsize Case} \label{sec-3.3}

For the constant-stepsize case, the results given in Theorems~\ref{thm-const-stepsize} and \ref{thm-const-stepsize-b} bear similarities to their counterparts for the diminishing stepsize case given in Theorems~\ref{thm-dim-stepsize} and \ref{thm-dim-stepsize-b}. However, they characterize the behavior of the iterates in the limit as the stepsize parameter approaches $0$, and deal with only a finite segment of the iterates for each stepsize (although in their part (ii) both the segment's length $T_\alpha/\alpha \to \infty$ and its starting position $k_\alpha \to \infty$ as $\alpha \to 0$). So unlike in the diminishing stepsize case, these results do not tell us explicitly about the behavior of $\theta_t^\alpha$ for a fixed stepsize $\alpha$ as we take $t$ to infinity. 

The purpose of the present subsection is to analyze further the case of a fixed stepsize just mentioned. 
We observe that for a fixed stepsize $\alpha$, the iterates $\theta_t^\alpha$ together with $\Z_t = (S_t, A_t, \e_t, \F_t)$ form a weak Feller Markov chain $\{(\Z_t, \theta_t^\alpha)\}$ (see Lemma~\ref{lma-c1}). 
Thus we can apply several ergodic theorems for weak Feller Markov chains (Meyn~\cite{Mey89}, Meyn and Tweedie~\cite{MeT09}) to analyze the constant-stepsize case and combine the implications from these theorems with the results we obtained previously using stochastic approximation theory.

We now present our results using this approach. Let $\mathcal{M}_\alpha$ denote the set of invariant probability measures of the Markov chain $\{(\Z_t, \theta_t^\alpha)\}$. 
This set depends on the particular algorithm used to generate the $\theta$-iterates, but we shall use the notation $\mathcal{M}_\alpha$ for all the algorithms we discuss here, for notational simplicity.  
We know that $\{\Z_t\}$ has a unique invariant probability measure (Theorem~\ref{thm-2.1}), but it need not be so for the Markov chain $\{(\Z_t, \theta_t^\alpha)\}$ when $\{\theta^\alpha_t\}$ is generated by the algorithm (\ref{eq-emtd-const0}) or its two variants. The set $\mathcal{M}_\alpha$ can therefore have multiple elements (it is nonempty; see Prop.~\ref{prop-c1}).
We denote by $\bar{\mathcal{M}}_{\alpha}$ the set that consists of the marginal of $\mu$ on $\H$ (the space of the $\theta$'s), for all the invariant probability measures $\mu \in \mathcal{M}_\alpha$.

As in the previous analysis, we are interested in the behavior of multiple consecutive $\theta$-iterates. In order to characterize that, we consider for each $m \geq 1$, the Markov chain 
$$ \big\{ \big( (\Z_t, \theta_t^\alpha), \, (\Z_{t+1}, \theta_{t+1}^\alpha), \, \ldots, \, (\Z_{t+m-1}, \theta_{t+m-1}^\alpha) \big) \big\}_{t \geq 0}$$ 
(i.e., each state now consists of $m$ consecutive states of the chain $\{(\Z_t, \theta_t^\alpha)\}$). We shall refer to this chain as the \emph{$m$-step version} of $\{(\Z_t, \theta_t^\alpha)\}$. 
Similar to $\mathcal{M}_\alpha$, denote by $\mathcal{M}_\alpha^m$ the set of invariant probability measures of the $m$-step version of $\{(\Z_t, \theta_t^\alpha)\}$, and correspondingly define $\bar{\mathcal{M}}_{\alpha}^m$ to be the set of marginals of $\mu$ on $\H^m$ for all $\mu \in \mathcal{M}^m_\alpha$. 
The set $\mathcal{M}_\alpha^m$ is, of course, determined by $\mathcal{M}_\alpha$, since each invariant probability measure in $\mathcal{M}_\alpha^m$ is just the $m$-dimensional distribution of a stationary Markov chain $\{(\Z_t, \theta_t^\alpha)\}$.

Our first result, given in Theorem~\ref{thm-c1} below, says that for the algorithm (\ref{eq-emtd-const0}), as the stepsize $\alpha$ approaches zero, the invariant probability measures in $\mathcal{M}_\alpha^m$ will concentrate their masses on an arbitrarily small neighborhood of $(\theta^*, \ldots, \theta^*)$ ($m$ copies of $\theta^*$).
Moreover, for a fixed stepsize, as the number of iterations grows to infinity, the expected maximal deviation of the $m$ consecutive averaged iterates from $\theta^*$ can be bounded in terms of the masses those invariant probability measures assign to the vicinities of $(\theta^*, \ldots, \theta^*)$. Here by averaged iterates, we mean 
\begin{equation}  \label{eq-avetheta}
 \bar \theta_t^\alpha = \frac{1}{t} \sum_{k = 0}^{t-1} \theta_k^\alpha, \qquad \forall \, t \geq 1,
\end{equation} 
and we shall refer to $\{\bar \theta_t^\alpha\}$ as the \emph{averaged sequence} corresponding to $\{\theta_t^\alpha\}$. This iterative averaging is also known as ``Polyak-averaging'' when it is applied to accelerate the convergence of the $\theta$-iterates (see \cite{PoJ92}, \cite[Chap.\ 10]{KuY03}, and the references therein). This is not the role of the averaging operation here, however. The purpose here is to bring to bear the ergodic theorems for weak Feller Markov chains, in particular, the weak convergence of certain averaged probability measures or occupation probability measures to the invariant probability measures of the $m$-step version of $\{(\Z_t, \theta_t^\alpha)\}$. (For the details see Section~\ref{sec-4.3}, where the proofs of the results of this subsection will be given.) It can also be seen that for a sequence $\{\beta_t\}$ with $\beta_t \in [0,1), \beta_t \to 0$ as $t \to \infty$, if we drop a fraction $\beta_t$ of the terms in (\ref{eq-avetheta}) when averaging the $\theta$'s at each time $t$, the resulting differences in the averaged iterates $\bar \theta_t^\alpha$ are asymptotically negligible. Therefore, although our results below will be stated for (\ref{eq-avetheta}), they apply to a variety of averaging schemes.
  
Recall that $N_\delta(\theta^*)$ denotes the closed $\delta$-neighborhood of $\theta^*$. In what follows, $N'_\delta(\theta^*)$ denotes the open $\delta$-neighborhood of $\theta^*$, i.e., the open ball around $\theta^*$ with radius $\delta$. We write $[N_\delta(\theta^*)]^m$ or $[N'_\delta(\theta^*)]^m$ for the Cartesian product of $m$ copies of $N_\delta(\theta^*)$ or $N'_\delta(\theta^*)$. Recall also that $r_\H$ is the radius of the constraint set $\H$.

\begin{thm} \label{thm-c1}
In the setting of Theorem~\ref{thm-const-stepsize}, let $\{\theta_t^\alpha\}$ be generated by the algorithm (\ref{eq-emtd-const0}) with constant stepsize $\alpha > 0$, and let $\{\bar \theta_t^\alpha\}$ be the corresponding averaged sequence. Then the following hold for any $\delta > 0$ and $m \geq 1$:
\begin{enumerate}
\item[\rm (i)] $\liminf_{\alpha \to 0} \inf_{\mu \in \bar{\mathcal{M}}^m_{\alpha}} \mu\big( [N_\delta(\theta^*)]^m \big)  = 1$, and more strongly, with $m_\alpha = \lceil \frac{m}{\alpha} \rceil$, 
$$   \liminf_{\alpha \to 0} \inf_{\mu \in \bar{\mathcal{M}}^{m_\alpha}_{\alpha}} \mu\big( [N_\delta(\theta^*)]^{m_\alpha} \big) = 1.$$
\item[\rm (ii)] For each stepsize $\alpha$ and any initial condition of $(\e_0,\F_0,\theta_0^\alpha)$, 
$$ \limsup_{k \to \infty} \E \Big[ \sup_{k \leq t < k+m} \big| \bar \theta_t^\alpha - \theta^* \big| \, \Big] \leq \delta \, \kappa_{\alpha,m} + 2 r_\H\, (1 - \kappa_{\alpha,m}), $$ 
where $\kappa_{\alpha,m} = \inf_{\mu \in \bar{\mathcal{M}}^m_\alpha} \mu([N'_\delta(\theta^*)]^m)$.
\end{enumerate}
\end{thm}

\medskip
Note that in part (ii) above, $\kappa_{\alpha, m} \to 1$ as $\alpha \to 0$ by part (i). Note also that for $m=1$, the conclusions from the preceding theorem take the simplest form:
$$\liminf_{\alpha \to 0} \inf_{\mu \in \bar{\mathcal{M}}_{\alpha}} \mu\big( N_\delta(\theta^*) \big) = 1,$$
$$\limsup_{t \to \infty} \E \big[ \big| \bar \theta_t^\alpha - \theta^* \big| \big] \leq \delta \, \kappa_{\alpha} + 2 r_\H\, (1 - \kappa_{\alpha}), \qquad \text{for} \ \kappa_{\alpha} = \inf_{\mu \in \bar{\mathcal{M}}_\alpha} \mu\big(N'_\delta(\theta^*)\big). $$ 
The conclusions for $m > 1$ are, however, much stronger. They also suggest that in practice, instead of simply choosing the last iterate of the algorithm as its final output at the end of its run, one can base that choice on the behavior of multiple consecutive $\bar \theta_t^\alpha$ during the run.

For the two variant algorithms~(\ref{eq-emtd-const1}) and (\ref{eq-emtd-const2}), we have a similar result given in Theorem~\ref{thm-c1b} below. Here the neighborhood of $(\theta^*, \ldots, \theta^*)$ around which the masses of the invariant probability measures are concentrated, depends not only on the stepsize $\alpha$ but also on the biases of these algorithms. The proofs of Theorems~\ref{thm-c1}-\ref{thm-c1b} are given in Section~\ref{sec-4.3.2}.

\begin{thm} \label{thm-c1b}
In the setting of Theorem~\ref{thm-const-stepsize}, let $\{\theta_t^\alpha\}$ be generated instead by the algorithm (\ref{eq-emtd-const1}) or (\ref{eq-emtd-const2}), with constant stepsize $\alpha > 0$ and with a bounded Lipschitz continuous function $\psi_K$ satisfying (\ref{eq-psi}). Let $\{\bar \theta_t^\alpha\}$ be the corresponding averaged sequence.
Then the following hold:
\begin{enumerate}
\item[\rm (i)] For any given $\delta > 0$, there exists $K_\delta > 0$ such that for all $K \geq K_\delta$, 
$$ \liminf_{\alpha \to 0} \inf_{\mu \in \bar{\mathcal{M}}^m_{\alpha}} \mu\big( [N_\delta(\theta^*)]^m \big) = 1, \qquad \forall \, m \geq 1,$$ 
and more strongly, with $m_\alpha = \lceil \frac{m}{\alpha} \rceil$,
$$   \liminf_{\alpha \to 0} \inf_{\mu \in \bar{\mathcal{M}}^{m_\alpha}_{\alpha}} \mu\big( [N_\delta(\theta^*)]^{m_\alpha} \big) = 1, \qquad \forall \, m \geq 1.$$
\item[\rm (ii)] Regardless of the choice of $K$, given any $\delta > 0, m \geq 1$ and stepsize $\alpha$, for each initial condition of $(\e_0,\F_0,\theta_0^\alpha)$, 
$$ \limsup_{k \to \infty} \E \Big[ \sup_{k \leq t < k+m} \big| \bar \theta_t^\alpha - \theta^* \big| \, \Big] \leq \delta \, \kappa_{\alpha,m} + 2 r_\H\, (1 - \kappa_{\alpha,m}), $$ 
where $\kappa_{\alpha,m} = \inf_{\mu \in \bar{\mathcal{M}}^m_\alpha} \mu([N'_\delta(\theta^*)]^m)$.
\end{enumerate}
\end{thm}

\medskip
Finally, we consider a simple modification of the preceding algorithms, for which the conclusions of Theorems~\ref{thm-c1}(ii) and~\ref{thm-c1b}(ii) can be strengthened. This is our motivation for introducing the modification, but we shall postpone the discussion till Remark~\ref{rmk-c2a} at the end of this subsection. 

For any of the algorithms (\ref{eq-emtd-const0}), (\ref{eq-emtd-const1}) or (\ref{eq-emtd-const2}), if the original recursion under a constant stepsize $\alpha$ can be written as
$$ \theta_{t+1}^\alpha = \Pi_\H \big( \theta_t^\alpha + \alpha Y_t^\alpha \big),$$
we now modify this recursion formula by adding a perturbation term $\alpha \Delta_{\theta,t}^\alpha$ as follows. Let
\begin{equation} \label{eq-valg}
 \theta_{t+1}^\alpha = \Pi_\H \big( \theta_t^\alpha + \alpha Y_t^\alpha + \alpha \Delta_{\theta,t}^\alpha \big),
\end{equation}
where for each $\alpha > 0$, $\Delta_{\theta,t}^\alpha, t \geq 0$, are $\re^{n}$-valued random variables such that%footnote starts
\footnote{We adopt these conditions for simplicity. They are not the weakest possible for our purpose, and our proof techniques can be applied to  other types of perturbations as well. See the discussions in Remark~\ref{rmk-c2a}, Remark~\ref{rmk-c2b}, and before Prop.~\ref{prop-c3} in Section~\ref{sec-4.3.3}.}
%footnote ends
\begin{enumerate}
\item[(i)] they are independent of each other and also independent of the process $\{\Z_t\}$; 
\item[(ii)] they are identically distributed with zero mean and finite variance, where the variance can be bounded uniformly for all $\alpha$; and 
\item[(iii)] they have a positive continuous density function with respect to the Lebesgue measure.
\end{enumerate}
Below we refer to (\ref{eq-valg}) as the perturbed version of the algorithm (\ref{eq-emtd-const0}), (\ref{eq-emtd-const1}) or (\ref{eq-emtd-const2}).

\begin{thm} \label{thm-c2}
In the setting of Theorem~\ref{thm-const-stepsize}, let $\{\theta_t^\alpha\}$ be generated instead by the perturbed version (\ref{eq-valg}) of the algorithm (\ref{eq-emtd-const0}) for a constant stepsize $\alpha > 0$, and let $\{\bar \theta_t^\alpha\}$ be the corresponding averaged sequence.
Then the conclusions of Theorems~\ref{thm-const-stepsize} and~\ref{thm-c1} hold. Furthermore, let the stepsize $\alpha$ be given. 
Then the Markov chain $\{(\Z_t, \theta_t^\alpha)\}$ has a unique invariant probability measure $\mu_\alpha$, and for any $\delta > 0$, $m \geq 1$, and initial condition of $(\e_0,\F_0,\theta_0^\alpha)$, almost surely, 
$$ \liminf_{t \to \infty} \, \frac{1}{t} \sum_{k=0}^{t-1} \I\Big(\sup_{k \leq j <  k+m} \big| \theta_j^\alpha - \theta^* \big| < \delta \Big) \geq \bar{\mu}_\alpha^{(m)}\big([N'_\delta(\theta^*)]^m \big)$$
and
$$ \limsup_{t \to \infty} \big| \bar \theta_t^\alpha - \theta^* \big|  \leq \delta \, \kappa_{\alpha} + 2 r_\H\, (1 - \kappa_{\alpha}), \qquad \text{with} \ \kappa_\alpha = \bar{\mu}_\alpha\big(N'_\delta(\theta^*) \big),  $$
where $\bar{\mu}_\alpha^{(m)}$ is the unique element in $\bar{\mathcal{M}}^m_\alpha$, 
and $\bar{\mu}_\alpha$ is the marginal of $\mu_\alpha$ on $\H$.
\end{thm}

%\smallskip
\begin{thm} \label{thm-c2b}
In the setting of Theorem~\ref{thm-const-stepsize}, let $\{\theta_t^\alpha\}$ be generated instead by the perturbed version (\ref{eq-valg}) of the algorithm (\ref{eq-emtd-const1}) or (\ref{eq-emtd-const2}), with a constant stepsize $\alpha > 0$ and with a bounded Lipschitz continuous function $\psi_K$ satisfying (\ref{eq-psi}).
Let $\{\bar \theta_t^\alpha\}$ be the corresponding averaged sequence.
Then the conclusions of Theorems~\ref{thm-const-stepsize-b} and~\ref{thm-c1b} hold. Furthermore, for any given stepsize $\alpha$, the conclusions of the second part of Theorem~\ref{thm-c2} also hold.
\end{thm}

%\smallskip
Note that in the second part of Theorem~\ref{thm-c2}, both $\bar{\mu}_\alpha^{(m)}\big([N'_\delta(\theta^*)]^m \big)$ and $\kappa_\alpha$ approach $1$ as $\alpha \to 0$, since by the first part of the theorem, the conclusions of Theorem~\ref{thm-c1} hold.  
For the second part of Theorem~\ref{thm-c2b}, the same is true provided that $K$ is sufficiently large (so that $N_\delta(L_\H) \subset N_\delta(\theta^*)$ where $L_\H$ is the limit set of the ODE associated with the algorithm), and this can be seen from the conclusions of Theorem~\ref{thm-c1b}(i), which holds for the perturbed version (\ref{eq-valg}) of the two variant algorithms, as the first part of Theorem~\ref{thm-c2b} says.
The proofs of Theorems~\ref{thm-c2}-\ref{thm-c2b} are given in Section~\ref{sec-4.3.3}.

\smallskip
\begin{rem}[on the role of perturbation] \label{rmk-c2a} 
At first sight it may seem counter-productive to add noise to the $\theta$-iterates in the algorithm (\ref{eq-valg}).
Our motivation for such random perturbations of the $\theta$-iterates is that this can ensure that the Markov chain $\{(\Z_t, \theta_t^\alpha)\}$ has a unique invariant probability measure (see Prop.~\ref{prop-c4}). The uniqueness allows us to invoke a result of Meyn \cite{Mey89} on the convergence of the occupation probability measures of a weak Feller Markov chain, so that we can bound the deviation of the averaged iterates from $\theta^*$ not only in an expected sense as before, but also for almost all sample paths under each initial condition, as in the second part of Theorems~\ref{thm-c2}-\ref{thm-c2b}. For the unperturbed algorithms, we can only prove such pathwise bounds on $\limsup_{t \to \infty} | \bar \theta_t^\alpha - \theta^* |$ for a subset of the initial conditions of $(\Z_0, \theta_0^\alpha)$. A more detailed discussion of this is given in Remark~\ref{rmk-c2b}, at the end of Section~\ref{sec-4.3.3}, after the proofs of the preceding theorems.

Regarding other effects of the perturbation, intuitively, larger noise terms may help the Markov chain ``mix'' faster, but they can also result in less probability mass $\bar{\mu}_\alpha\big(N'_\delta(\theta^*) \big)$ around $\theta^*$ than in the case without perturbation. What is a suitable amount of noise to add to achieve a desired balance? We do not yet have an answer. It seems reasonable to us to let the magnitude of the variance of the perturbation terms $\Delta^\alpha_{\theta,t}$ be approximately $\alpha^{2\epsilon}$ for some $\epsilon \in (0,1]$, so that a typical perturbation $\alpha \Delta^\alpha_{\theta,t}$ is at a smaller scale relative to the ``signal part'' $\alpha Y_t^\alpha$ in an iteration. Further investigation is needed. \qed
\end{rem}

\section{Proofs for Section~\ref{sec-3}} \label{sec-4}

We now prove the theorems given in the preceding section.

\subsection{Proofs for Theorems~\ref{thm-dim-stepsize} and~\ref{thm-const-stepsize}} \label{sec-4.1}

In this subsection we prove Theorems~\ref{thm-dim-stepsize} and~\ref{thm-const-stepsize} on convergence properties of the constrained ETD($\lambda$) algorithm (\ref{eq-emtd-const0}). We will apply two theorems from \cite{KuY03}, Theorems 8.2.2 and 8.2.3, which concern weak convergence of stochastic approximation algorithms for constant and diminishing stepsize, respectively. This requires us to show that the conditions of those theorems are satisfied by our algorithm. The major conditions concern the uniform integrability, tightness, and convergence in mean of certain sequences of random variables involved in the algorithm. Our proofs will rely on many properties of the ETD iterates that we have established in \cite{yu-etdarx} when analyzing the almost sure convergence of the algorithm.
 
\subsubsection{Conditions to Verify} \label{sec-4.1.1}
We need some definitions and notation, before describing the conditions required. For some index set $\mathcal{K}$, let $\{X_k\}_{k \in \mathcal{K}}$ be a set of random variables taking values in a metric space $\X$ (in our context $\X$ will be $\re^m$ or $\Xi$). 
The set $\{X_k\}_{k \in \mathcal{K}}$ is said to be \emph{tight} or \emph{bounded in probability}, if there exists for each $\delta >0$ a compact set $D_\delta \subset \X$ such that 
$$\inf_{k \in \mathcal{K}} \, \Pr (X_k \in D_\delta) \geq 1 - \delta.$$
For $\re^m$-valued $X_k$, the set $\{X_k\}_{k \in \mathcal{K}}$ is said to be \emph{uniformly integrable} (u.i.) if 
$$ \lim_{a \to \infty} \, \sup_{k \in \mathcal{K}} \, \E \left[ \| X_k\| \, \I \big( \| X_k\| \geq a \big) \right] = 0.$$

To analyze the constrained ETD($\lambda$) algorithm (\ref{eq-emtd-const0}), which is given by
$$ \theta_{t+1} = \Pi_\H (\theta_t + \alpha_t Y_t), \qquad \text{where} \ Y_t: = \e_t \cdot \rho_t \big(R_t + \gamma_{t+1} \fe(S_{t+1})^\top \theta_t - \fe(S_t)^\top \theta_t \big),$$
let $\E_t$ denote expectation conditioned on $\mathcal{F}_t$, the sigma-algebra generated by $\theta_m, \xi_m, m \leq t$, where we recall $\xi_m = (\e_m, \F_m, S_m, A_m, S_{m+1})$ and its space $\re^{n+1} \times \mathcal{S} \times \mathcal{A}  \times \mathcal{S}$ is denoted by $\Xi$.
By writing $Y_t = \E_t [ Y_t ] + (Y_t -  \E_t [ Y_t ])$, we have the equivalent form of (\ref{eq-emtd-const0}) given in (\ref{eq-emtd-const}):
$$ \theta_{t+1} = \Pi_{\H} \!\left( \theta_t + \alpha_t \, h(\theta_t, \xi_t) + \alpha_t \, \e_t \cdot \tilde \omega_{t+1} \right).$$
In other words, $h(\theta_t, \xi_t) = \E_t [Y_t]$ and $\e_t \cdot \tilde \omega_{t+1} = Y_t -  \E_t [ Y_t ]$, a noise term that satisfies $\E_t [ \e_t \cdot \tilde \omega_{t+1} ] = 0$.

This algorithm belongs to the class of stochastic approximation algorithms with ``exogenous noises'' studied in the book \cite{KuY03}---the term ``exogenous noises'' reflects the fact that the evolution of $\{\xi_t\}$ is not driven by the $\theta$-iterates. 
Theorems~\ref{thm-dim-stepsize} and~\ref{thm-const-stepsize} will follow as special cases from Theorems 8.2.3 and 8.2.2 of \cite[Chap.\ 8]{KuY03}, respectively, if we can show that the algorithm (\ref{eq-emtd-const0}) satisfies the following conditions. 

%\smallskip
{\samepage
\noindent {\bf Conditions for the case of diminishing stepsize:}
\begin{enumerate}
\item[(i)] The sequence $\{ Y_t\} = \{ h(\theta_t, \xi_t) +  \e_t \cdot \tilde \omega_{t+1}\}$ is u.i. (This corresponds to the condition A.8.2.1 of \cite{KuY03}.) 
\item[(ii)] The function $h(\theta, \xi)$ is continuous in $\theta$ uniformly in $\xi \in D$, for each compact set $D \subset \Xi$. (This corresponds to the condition A.8.2.3 of \cite{KuY03}.) 
\item[(iii)] The sequence $\{\xi_t\}$ is tight. 
(This corresponds to the condition A.8.2.4 of \cite{KuY03}.)
\item[(iv)] The sequence $\{h(\theta_t, \xi_t)\}$ is u.i., and so is $\{h(\theta, \xi_t)\}$ for each fixed $\theta \in \H$. (This corresponds to the condition A.8.2.5 of \cite{KuY03}.)
\item[(v)] There is a continuous function $\bar h(\cdot)$ such that for each $\theta \in \H$ and each compact set $D \subset \Xi$,
$$ \lim_{k \to \infty, t \to \infty}  \, \frac{1}{k} \sum_{m = t}^{t+k-1} \E_t \left[ h(\theta, \xi_m) - \bar h(\theta) \right] \, \I\big(\xi_t \in D\big) = 0 \qquad \text{in mean},$$
where $k$ and $t$ are taken to $\infty$ in any way possible. In other words, if we denote the average on the left-hand side by $X_{k,t}$, then the requirement  ``$\lim_{k \to \infty,t \to \infty} X_{k,t} = 0$ in mean'' means that along any subsequences $k_j \to \infty, t_j \to \infty$, we must have $\lim_{j \to \infty} \E [ \| X_{k_j, t_j} \| ] = 0$.
(This condition corresponds to the condition A.8.2.7 of \cite{KuY03}.)
\end{enumerate}
}

\medskip
For the case of constant stepsize, we consider the iterates that could be generated by the algorithm for all stepsizes. To distinguish between the iterates associated with different stepsizes, in the conditions below, the superscript $\alpha$ is attached to the variables involved in the algorithm with stepsize $\alpha$, and similarly, 
the conditional expectation  $\E_t$ is denoted by $\E_t^\alpha$ instead. 

\smallskip
\noindent {\bf Conditions for the case of constant stepsize:}  \\
In addition to the condition (ii) above (which corresponds to the condition A.8.1.6 of \cite{KuY03} for the case of constant stepsize), the following conditions are required.
\begin{enumerate}
\item[(i$'$)] The set $\{ Y_t^\alpha \mid t \geq 0, \alpha > 0\} : = \{ h(\theta_t^\alpha, \xi_t^\alpha) +  \e_t^\alpha \cdot \tilde \omega_{t+1}^\alpha \mid t \geq 0, \alpha > 0\}$ is u.i. (This corresponds to the condition A.8.1.1 of \cite{KuY03}.)
\item[(iii$'$)] The set $\{\xi_t^\alpha \mid t \geq 0, \alpha > 0 \}$ is tight. (This corresponds to the condition A.8.1.7 of \cite{KuY03}.)
\item[(iv$'$)]  The set $\{h(\theta_t^\alpha, \xi_t^\alpha) \mid t \geq 0, \alpha > 0\}$ is u.i., in addition to the uniform integrability of $\{h(\theta, \xi_t^\alpha) \mid t \geq 0, \alpha > 0\}$ for each $\theta \in \H$. (This corresponds to the condition A.8.1.8 of \cite{KuY03}.)
\item[(v$'$)] There is a continuous function $\bar h(\cdot)$ such that for each $\theta \in \H$ and each compact set $D \subset \Xi$,
$$ \lim_{k \to \infty, t \to \infty, \alpha \to 0}  \, \frac{1}{k} \sum_{m = t}^{t+k-1} \E_t^\alpha \left[ h(\theta, \xi_m^\alpha) - \bar h(\theta) \right] \, \I\big(\xi_t^\alpha \in D\big) = 0 \qquad \text{in mean},$$
where $\alpha$ is taken to $0$ and $k, t$ are taken to $\infty$ in any way possible. (This condition corresponds to the condition A.8.1.9 of \cite{KuY03}, and it is in fact stronger than the latter condition but is satisfied by our algorithms as we will show.)
\end{enumerate}

\medskip
The preceding conditions allow $\xi_t^\alpha$ and $\theta_t^\alpha$ to be generated under different initial conditions for different $\alpha$. While we will need this generality later in Section~\ref{sec-4.3}, here we will focus on a common initial condition for all stepsizes, for simplicity. Then, the preceding conditions for the constant-stepsize case are essentially the same as those for the diminishing stepsize case, because except for the $\theta$-iterates, all the other variables (such as $\xi_t$ and $\tilde \omega_{t}$) involved in the algorithm have identical probability distributions for all stepsizes $\alpha$ and are not affected by the $\theta$-iterates. For this reason, in the proofs below, except for the $\theta$-iterates, we simply omit the superscript $\alpha$ for other variables in the case of constant stepsize, and to verify the two sets of conditions above, we shall treat the case of diminishing stepsize and the case of constant stepsize simultaneously.

As mentioned in Section~\ref{sec-2.3}, these conditions are to ensure that the projected ODE~(\ref{eq-pode}), $\dot{x} = \bar h(x) +z, z \in - \mathcal{N}_\H(x)$, is the mean ODE for the algorithm (\ref{eq-emtd-const0}) and reflects its average dynamics. Among the proofs for these conditions given next, the proof for the convergence in mean condition (v) and (v$'$) will be the most involved.

\subsubsection{Proofs}

The condition (ii) is clearly satisfied. In what follows, we prove that the rest of the conditions are satisfied as well. We start with the tightness conditions (iii) and (iii$'$), as they are immediately implied by a property of the trace iterates we already know. We then tackle the uniform integrability conditions (i), (i$'$), (iv) and (iv$'$), before we address the convergence in mean required in (v) and (v$'$). The proofs build upon several key properties of the ETD iterates we have established in \cite{yu-etdarx} and recounted in Section~\ref{sec-2.3} and Appendix~\ref{appsec-a}.

First, we show that the tightness conditions (iii) and (iii$'$) are satisfied. This is implied by the following property of traces: for any given initial condition $(\e_0, \F_0)$,
$\sup_{t \geq 0} \E \big[ \big\| (\e_t, \F_t) \big\| \big] < \infty$ (see Prop.~\ref{prp-bdtrace}, Appendix~\ref{appsec-a}).
   
\smallskip
\begin{prop} \label{prop-1}
Under Assumption~\ref{cond-bpolicy}, for each given initial $(\e_0,\F_0) \in \re^{n+1}$, $\{(\e_t, \F_t)\}$ is tight and hence $\{\xi_t\}$ is tight.
\end{prop}
\begin{proof}
By Prop.~\ref{prp-bdtrace}, $c : = \sup_{t \geq 0} \E \big[ \big\| (\e_t, \F_t) \big\| \big] < \infty$. Then, by the Markov inequality, for $a > 0$, 
$\sup_{t \geq 0} \Pr \big(  \big\| (\e_t, \F_t) \big\| \geq a \big) \leq c/a \to 0$ as $a \to \infty.$
This implies that $\{(\e_t, \F_t)\}$ is tight. Since the space $\S \times \A \times S$ is finite and $\xi_t = (\e_t, \F_t, S_t, A_t, S_{t+1})$, $\{\xi_t\}$ is also tight. 
\end{proof}

We now handle the uniform integrability conditions (i), (i$'$), (iv) and (iv$'$). The uniform integrability of the trace sequence $\{\e_t\}$, as we will prove, is important here.

%\smallskip
\begin{prop}\label{prop-2}
Under Assumption~\ref{cond-bpolicy}, for each given initial $(\e_0,\F_0) \in \re^{n+1}$, the following sets of random variables are u.i.: 
\begin{enumerate}
\item[\rm (i)] $\{\e_t\}$;
\item[\rm (ii)] $\{ h(\theta, \xi_t)\}$ for each fixed $\theta \in \H$;
\item[\rm (iii)] $\{ h(\theta_t, \xi_t)\}$ in the case of diminishing stepsize; and $\{ h(\theta_t^\alpha, \xi_t) \mid t \geq 0, \alpha > 0\}$ in the case of constant stepsize;
\item[\rm (iv)] $\{ h(\theta_t, \xi_t) + \e_t \, \tilde \omega_{t+1} \}$ in the case of diminishing stepsize; and $\{ h(\theta_t^\alpha, \xi_t) + \e_t \, \tilde \omega_{t+1} \mid t \geq 0, \alpha > 0 \}$ in the case of constant stepsize.
\end{enumerate}
\end{prop}

\smallskip
The proof of this proposition will use facts about u.i.\ sequences of random variables given in the lemma below.

\begin{lem} \label{lem-w1}
Let $X_k, Y_k, k \in \K$ (some index set) be real-valued random variables with $X_k$ and $Y_k$ defined on a common probability space for each $k$.
\begin{enumerate}
\item[\rm (i)] If $\{X_k\}_{k \in \K},\{Y_k\}_{k \in \K}$ are u.i., then $\{X_k + Y_k\}_{k \in \K}$ is u.i. 
\item[\rm (ii)] If $\{X_k\}_{k \in \K}$ is u.i.\ and for all $k$, $|Y_k| \leq |X_k|$ a.s., then $\{Y_k\}_{k \in \K}$ is u.i.
\item[\rm (iii)] If $\{X_k\}_{k \in \K},\{Y_k\}_{k \in \K}$ are u.i.\ and for some $c \geq 0$, $\E [ |Y_k| \mid X_k ] \leq c$ a.s.\ for all $k$, then $\{X_k Y_k\}_{k \in \K}$ is u.i.
\end{enumerate}
\end{lem}

\begin{proof}
Part (i) follows immediately from the definition of uniform integrability and the inequality that for any $a>0$,
$$ | X_k + Y_k | \, \I\big(|X_k + Y_k| \geq a \big) \, \leq \,  2 | X_k| \, \I\big(|X_k| \geq \tfrac{a}{2}\big) + 2 | Y_k| \, \I \big(|Y_k| \geq \tfrac{a}{2} \big).$$

Under the assumption of (ii),  for any $a > 0$, we have $|Y_k|  \I\big(|Y_k| \geq a\big)  \leq |X_k| \I\big(|X_k| \geq a\big)$ a.s., and (ii) is then evident.

Under the assumption of (iii), for any $a > \bar a > 0$, using the inequality 
$$ | X_k Y_k | \, \I\big(|X_k Y_k| \geq a \big)  \leq | X_k Y_k | \,  \I\big(|X_k| \leq \bar a,  |Y_k| \geq \tfrac{a}{\bar a} \big) +  | X_k Y_k | \,  \I\big(|X_k| \geq \bar a\big),$$ 
we have
\begin{align}
\E \left[ | X_k Y_k | \, \I\big(|X_k Y_k| \geq a \big) \right]  & \leq \E \left[ | X_k Y_k | \,  \I\big(|X_k| \leq \bar a,  |Y_k| \geq \tfrac{a}{\bar a} \big) \right] +  \E \left[  | X_k Y_k | \,  \I\big(|X_k| \geq \bar a\big) \right] \notag \\
& \leq \bar a \cdot \E \left[  |Y_k| \, \I\big(|Y_k| \geq \tfrac{a}{\bar a} \big) \right] + \E \left[  | X_k |  \,  \I\big(|X_k| \geq \bar a\big) \cdot \E \big[ |Y_k | \mid X_k \big] \right]  \notag \\
& \leq \bar a \cdot \E \left[  |Y_k| \, \I\big(|Y_k| \geq \tfrac{a}{\bar a} \big) \right] + c \cdot \E \left[  | X_k |  \,  \I\big(|X_k| \geq \bar a\big) \right], \label{eq-prf-w0}
\end{align}
where we used the assumption $\E [ |Y_k| \mid X_k ] \leq c$ a.s.\ in the last inequality.
Since $\{X_k\}_{k \in \K}, \{Y_k\}_{k \in \K}$ are assumed to be u.i., we have
\begin{equation}
  \lim_{a \to \infty} \sup_{k \in \K} \E \left[  |Y_k| \, \I\big(|Y_k| \geq \tfrac{a}{\bar a} \big) \right]  = 0 \ \ \text{for any given} \ \bar a > 0, \qquad
   \lim_{\bar a \to \infty} \sup_{k \in \K}  \E \left[  | X_k |  \,  \I\big(|X_k| \geq \bar a\big) \right] = 0. \label{eq-prf-w1}
\end{equation}   
From (\ref{eq-prf-w0}), by taking first $a \to \infty$ and then $\bar a \to \infty$, and by using (\ref{eq-prf-w1}), we then obtain
$$
\lim_{a \to \infty} \sup_{k \in \K} \E \left[ | X_k Y_k | \, \I\big(|X_k Y_k| \geq a \big) \right] = 0,$$
and this proves (iii).
\end{proof}

We now proceed to prove Prop.~\ref{prop-2}. The proof will involve auxiliary variables, which we call truncated traces. They are defined similarly to the trace iterates $(\e_t, \F_t)$, but instead of depending on all the past states and actions, they only depend on a certain number of the most recent states and actions. Specifically, for each integer $K \geq 1$, we define truncated traces $(\tilde{\e}_{t,K}, \tilde{\F}_{t,K})$ as follows:
$$ (\tilde{\e}_{t,K}, \tilde{\F}_{t,K}) = (\e_t, \F_t) \quad  \text{for} \ \  t \leq K,$$
and for $t \geq K+1$, with the shorthand $\beta_t : =\rho_{t-1} \gamma_t \lambda_t$,
\begin{align}
    \tilde{\F}_{t,K} & = \sum_{k=t-K}^t \i(S_k) \cdot \big(\rho_{k} \gamma_{k+1} \cdots \rho_{t-1} \gamma_t \big), \label{eq-tF0} \\
    \tilde{\M}_{t,K} & = \, \lambda_t \, \i(S_t) + ( 1 - \lambda_t) \tilde{\F}_{t,K}, \label{eq-tM0} \\
    \tilde{\e}_{t,K} & = \sum_{k=t-K}^t \tilde{\M}_{k,K} \cdot \fe(S_k) \cdot \big(\beta_{k+1} \cdots \beta_t \big). \label{eq-te0}
\end{align}
Note that when $t \geq 2K +1$, the traces $(\tilde{\e}_{t,K}, \tilde{\F}_{t,K})$ no longer depend on the initial $(\e_0,\F_0)$; being functions of the states and actions between time $t-2K$ and $t$ only, they lie in a bounded set determined by $K$, since the state and action spaces are finite. For $t = 0, \ldots, 2K$, $(\tilde{\e}_{t,K}, \tilde{\F}_{t,K})$ also lie in a bounded set, which is determined by $K$ and the initial $(\e_0, \F_0)$. 
We will use these bounded truncated traces to approximate the original traces $\{(\e_t, \F_t)\}$ in the analysis. 

An important approximation property, given in Prop.~\ref{prp-3} (Appendix~\ref{appsec-a}), is that for each $K$ and any initial $(\e_0, \F_0)$ from a given bounded set $E$,
$$\sup_{t \geq 0} \E \left[ \big\| (\e_t, \F_t)  - (\tilde{\e}_{t,K}, \tilde{\F}_{t,K}) \big\| \right] \leq \C_K,$$
where $\C_K$ is a finite constant that depends on $K$ and $E$ and decreases monotonically to $0$ as $K$ increases:
$$\C_K \downarrow 0 \ \ \ \text{as} \ K \to \infty.$$
We will use this property in the following analysis. 

\begin{proof}[Proof of Prop.~\ref{prop-2}]
First, we prove $\{\e_t\}$ is u.i. We then use this to show the uniform integrability of the other sets required in parts (ii)-(iv).

\medskip
\noindent (i) To prove $\{\e_t\}$ is u.i., we shall exploit its relation with the truncated traces, $\tilde \e_{t,K}, t \geq 0$ for integers $K \geq 1$. Note that since the state and action spaces are finite, the truncated traces $\{\tilde \e_{t,K}\}$ lie in a bounded set (this set depends on $K$ and the initial $(\e_0,\F_0)$), so there exists a constant $a_{K}$ such that $\| \tilde \e_{t,K}\| \leq a_K$ for all $t$. This fact will greatly simplify the analysis. 
Let us first fix $K$ and consider $a \geq a_k$. Denote $\bar a = a - a_K \geq 0$.
Then
\begin{align}
  \| \e_t\| \, \I \big(\|\e_t\| \geq a\big) \, & \leq \, \| \e_t\| \, \I \big(\|\e_t - \tilde \e_{t,K} \| \geq \bar a\big)   \notag \\
  & \leq \, \| \e_t - \tilde \e_{t,K} \| \, \I \big(\|\e_t - \tilde \e_{t,K}  \| \geq \bar a\big) + \| \tilde \e_{t,K} \| \, \I \big(\|\e_t - \tilde \e_{t,K} \| \geq \bar a\big) \notag \\
  & \leq \, \| \e_t - \tilde \e_{t,K} \| \, \I \big(\|\e_t - \tilde \e_{t,K} \| \geq \bar a\big) + a_K \, \I \big(\|\e_t - \tilde \e_{t,K} \| \geq \bar a\big). \label{eq-prf-w2}
\end{align}  
For the second term on the right-hand side, we can bound its expectation by
\begin{equation}
 \E  \left[ a_K \, \I \big(\|\e_t - \tilde \e_{t,K} \| \geq \bar a\big) \right] = a_K \Pr( \| \e_t - \tilde \e_{t,K} \| \geq \bar a) \leq a_K \cdot \C_K/\bar a, \qquad \forall \, t, \label{eq-prf-w3}
\end{equation} 
where in the last inequality $\C_K$ is a constant that depends on $K$ (and the initial $(\e_0,\F_0)$) and has the property that $\C_K \downarrow 0$ as $K \to \infty$, and this inequality is derived by combing the Markov inequality $\Pr( \| \e_t - \tilde \e_{t,K} \| \geq \bar a) \leq \E [ \| \e_t - \tilde \e_{t,K} \| ]/\bar a$ with Prop.\ \ref{prp-3}, which bounds $\sup_{t \geq 0} \E [ \| \e_t - \tilde \e_{t,K} \| ]$ by $\C_K$.
Similarly, for the first term on the right-hand side of (\ref{eq-prf-w2}), using Prop.\ \ref{prp-3}, we can bound its expectation by $\C_K$:
\begin{equation} \label{eq-prf-w4}
   \E \left[ \| \e_t - \tilde \e_{t,K} \| \, \I \big(\|\e_t - \tilde \e_{t,K} \| \geq \bar a\big) \right] \leq \E [ \| \e_t - \tilde \e_{t,K} \| ] \leq \C_K, \qquad \forall \, t.
\end{equation} 
From (\ref{eq-prf-w2})-(\ref{eq-prf-w4}) it follows that
$$ \sup_{t \geq 0} \E \left[  \| \e_t\| \, \I \big(\|\e_t\| \geq a\big) \right] \, \leq \, \C_K + a_K \cdot \C_K/(a - a_K),$$
so for fixed $K$, by taking $a \to \infty$, we obtain 
$$ \lim_{a \to \infty} \sup_{t \geq 0} \E \left[  \| \e_t\| \, \I \big(\|\e_t\| \geq a\big) \right] \, \leq \, \C_K.$$
Since $\C_K \downarrow 0$ as $K \to \infty$ (Prop.\ \ref{prp-3}), this implies
$\lim_{a \to \infty} \sup_{t \geq 0} \E \left[  \| \e_t\| \, \I \big(\|\e_t\| \geq a\big) \right] = 0,$
which proves the uniform integrability of $\{\e_t\}$.

\medskip
\noindent (ii) We now prove for each $\theta$, $\{h(\theta, \xi_t)\}$ is u.i. Since the state and action spaces are finite and $\theta$ is given, using the expression of $h(\theta, \xi_t)$, we can bound it as $\| h(\theta, \xi_t) \| \leq \C \|\e_t\|$ for some constant $\C$.
As just proved, $\{\e_t\}$ is u.i.\ (equivalently $\{ \| \e_t\| \}$ is u.i.) and thus $\{\C \|\e_t\|\}$ is u.i., so by Lemma~\ref{lem-w1}(ii), $\{h(\theta, \xi_t)\}$ is u.i.\ (since this is by definition equivalent to $\{\| h(\theta, \xi_t)\|\}$ being u.i., which is true by Lemma~\ref{lem-w1}(ii)).

\medskip
\noindent (iii) The uniform integrability of $\{h(\theta_t, \xi_t)\}$ in the case of diminishing stepsize or $\{h(\theta_t^\alpha, \xi_t) \mid t \geq 0, \alpha > 0\}$ in the case of constant stepsize follows from the same argument given for (ii) above, because $\theta_t$ or $\theta_t^\alpha$ for all $t \geq 0$ and $\alpha > 0$ lie in the bounded set $\H$ by the definition of the constrained ETD($\lambda$) algorithm.

\medskip
\noindent (iv)  Consider first the case of diminishing stepsize. We prove that $\{h(\theta_t, \xi_t) + \e_t \, \tilde \omega_{t+1} \}$ is u.i.\ (recall $\tilde \omega_{t+1}  = \rho_t \, (R_{t} - r(S_t, A_t, S_{t+1}))$ is the noise part of the observed reward). Since we already showed that $\{h(\theta_t, \xi_t)\}$ is u.i., by Lemma~\ref{lem-w1}(i), it is sufficient to prove that $\{\e_t \, \tilde \omega_{t+1} \}$ is u.i. Now $\{\e_t\}$ is u.i.\ by part (i).
Since the random rewards $R_t$ in our model have bounded variances, the noise variables $\tilde \omega_{t+1}, t \geq 0,$ also have bounded variances. 
This implies that $\{\tilde \omega_{t+1}\}$ is u.i.\ \citep[p.\ 32]{Bil68} and that $\E [ | \tilde \omega_{t+1} | \mid \e_t ] < c$ for some constant $c$ (independent of $t$). It then follows from Lemma~\ref{lem-w1}(iii) that $\{\e_t \, \tilde \omega_{t+1} \}$ is u.i., and hence $\{h(\theta_t, \xi_t) + \e_t \, \tilde \omega_{t+1} \}$ is u.i.

Similarly, in the case of constant stepsize, it follows from Lemma~\ref{lem-w1}(i) that the set $\{ h(\theta_t^\alpha, \xi_t) + \e_t \, \tilde \omega_{t+1} \mid t \geq 0, \alpha > 0 \}$ is u.i., because $\{ h(\theta_t^\alpha, \xi_t) \mid t \geq 0, \alpha > 0 \}$ is u.i.\ by part (iii) proved earlier and $\{\e_t \, \tilde \omega_{t+1}\}$ is u.i.\ as we just proved. 
\end{proof}

Finally, we handle the conditions (v) and (v$'$) stated in Section~\ref{sec-4.1.1}. The two conditions are the same condition in the case here, because they concern each fixed $\theta$, whereas $\{\xi_t\}$ is not affected by the stepsize and the $\theta$-iterates. So we can focus just on the condition (v) in presenting the proof, for notational simplicity.
For the algorithm (\ref{eq-emtd-const0}), the continuous function $\bar h$ required in the condition is the function $\bar h(\theta) = C \theta + b$ associated with the desired mean ODE (\ref{eq-pode}). We now prove the required convergence in mean by using the properties of trace iterates and the convergence results given in Theorem~\ref{thm-2.2} and Cor.~\ref{cor-2.1}. 

\smallskip
\begin{prop} \label{prop-3}
Let Assumption~\ref{cond-bpolicy} hold. For each $\theta \in \H$ and each compact set $D \subset \Xi$, 
$$ \lim_{k \to \infty, t \to \infty}  \, \frac{1}{k} \sum_{m = t}^{t+k-1} \E_t \left[ h(\theta, \xi_m) - \bar h(\theta) \right] \, \I\big(\xi_t \in D\big) = 0 \qquad \text{in mean}.$$
\end{prop}

\begin{proof}
Denote $X_{k,t} = \frac{1}{k} \sum_{m = t}^{t+k-1} \big( h(\theta, \xi_m) - \bar h(\theta) \big) \, \I\big(\xi_t \in D\big)$. Since
$\E\big[ \big \| \E_t \{ X_{k,t} \} \big\| \big] \leq \E [ \| X_{k,t} \| ]$, to prove $\lim_{k,t} \E\big[ \big\| \E_t \{ X_{k,t} \} \big\| \big] = 0$ (here and in what follows we simply write ``$k,t$'' under a limit symbol for ``$k \to \infty, t \to \infty$''),
it is sufficient to prove $\lim_{k,t} \E [ \| X_{k,t} \| ] = 0$, that is, to prove
\begin{equation} \label{eq-prf-w5a}
 \lim_{k, t} \, \frac{1}{k} \sum_{m = t}^{t+k-1} \left( h(\theta, \xi_m) - \bar h(\theta) \right) \, \I\big(\xi_t \in D\big) = 0 \qquad \text{in mean}.
\end{equation} 
Furthermore, since 
$$\limsup_{k, t} \, \E \big[ \| X_{k,t}\| \, \I\big(\xi_t \in D \big) \big] \leq \sum_{(s,a,s') \in \S \times \A \times \S} \limsup_{k, t} \, \E \left[ \| X_{k,t}\| \, \I\big(\xi_t \in D, (S_t, A_t, S_{t+1}) = (s,a,s') \big) \right],$$
it is sufficient in the proof to consider only those compact sets $D$ of the form $D = E \times \{(s,a,s')\}$, for each compact set $E \subset \re^{n+1}$ and each $(s,a,s') \in \S \times \A \times \S$.
Henceforth, let us fix a compact set $E$ together with a triplet $(s,a,s')$ as the set $D$ under consideration, and for this set $D$, we proceed to prove (\ref{eq-prf-w5a}).

To show (\ref{eq-prf-w5a}), what we need to show is that for two arbitrary subsequences of integers $k_j \to \infty$, $t_j \to \infty$, 
\begin{equation} \label{eq-prf-w5}
  \lim_{j \to \infty} \, \frac{1}{k_j} \sum_{m = t_j}^{t_j+k_j-1} \left( h(\theta, \xi_m) - \bar h(\theta) \right) \, \I\big(\xi_{t_j} \in D\big) = 0 \qquad \text{in mean}.
\end{equation}  
To this end, we first define auxiliary trace variables to decompose each difference term $h(\theta, \xi_m) - \bar h(\theta)$ into two difference terms as follows:
\begin{enumerate}
\item[(a)] Fix a point $(\bar e, \bar F) \in E$. 
\item[(b)] For each $j \geq 1$, define a sequence of trace pairs, $(\e_m^{j}, \F_m^{j})$, $m \geq t_j$, by using the same recursion (\ref{eq-td3})-(\ref{eq-td1}) that defines the traces $\{(\e_t, \F_t)\}$, based on the same trajectory $\{(S_t, A_t)\}$, but starting at time $m = t_j$ with the initial $(\e_{t_j}^{j},  \F_{t_j}^{j}) = (\bar \e, \bar F)$.  
\end{enumerate}
Denote $\xi^j_m = (\e_m^j, \F_m^j, S_m, A_m, S_{m+1})$ for $m \geq t_j$; it differs from $\xi_m$ only in the two trace components.
Next, for each $m$, we write $h(\theta, \xi_m) - \bar h(\theta) = (h(\theta, \xi_m^{j}) - \bar h(\theta)) + (h(\theta, \xi_m) - h(\theta, \xi_m^{j}))$ and correspondingly, we write
$$  \frac{1}{k_j} \sum_{m = t_j}^{t_j+k_j-1} \big( h(\theta, \xi_m) - \bar h(\theta) \big) = \frac{1}{k_j} \sum_{m = t_j}^{t_j+k_j-1}  \big( h(\theta, \xi_m^{j}) - \bar h(\theta) \big)  + \frac{1}{k_j} \sum_{m = t_j}^{t_j+k_j-1}  \big( h(\theta, \xi_m) - h(\theta, \xi_m^{j}) \big).$$
We see that for (\ref{eq-prf-w5}) to hold, it is sufficient that 
\begin{equation} \label{eq-prf-w6}
  \lim_{j \to \infty} \, \frac{1}{k_j} \sum_{m = t_j}^{t_j+k_j-1} \left( h(\theta, \xi_m^{j}) - \bar h(\theta)  \right) \, \I\big(\xi_{t_j} \in D\big) = 0 \qquad \text{in mean},
\end{equation}   
and 
\begin{equation} \label{eq-prf-w7}
  \lim_{j \to \infty} \, \frac{1}{k_j} \sum_{m = t_j}^{t_j+k_j-1} \left( h(\theta, \xi_m) - h(\theta, \xi_m^{j}) \right) \, \I\big(\xi_{t_j} \in D\big) = 0 \qquad \text{in mean}.
\end{equation} 
Let us now prove these two statements.

\medskip
\noindent
Proof of (\ref{eq-prf-w6}): Since the set $D=E \times \{(s,a,s')\}$ and $\I\big(\xi_{t_j} \in D\big) \leq  \I\big((S_{t_j}, A_{t_j}, S_{t_j+1})=(s,a,s')\big),$ 
we can remove $\xi_{t_j}$ from consideration and show instead
\begin{equation} \label{eq-prf-w8}
 \lim_{j \to \infty} \, \frac{1}{k_j} \sum_{m = t_j}^{t_j+k_j-1} \left( h(\theta, \xi_m^{j}) - \bar h(\theta)  \right) \, \I\big((S_{t_j}, A_{t_j}, S_{t_j+1})=(s,a,s')\big) = 0 \qquad \text{in mean},
\end{equation} 
which will imply (\ref{eq-prf-w6}). By definition $\xi_m^j, m \geq t_j$, are generated from the initial trace pairs $(\bar \e, \bar \F)$ and initial transition $(S_{t_j}, A_{t_j}, S_{t_j+1})$ at time $m = t_j$. So if $(S_{t_j}, A_{t_j}, S_{t_j+1}) = (s,a,s')$, then conditioned on this transition at $t_j$, the sequence $\{\xi_m^j, m \geq t_j\}$ has the same probability distribution as a sequence $\hat \xi_m, m \geq 0$, where $\hat \xi_m =(\hat \e_m, \hat \F_m, \hat S_m, \hat A_m, \hat S_{m+1})$ is generated from the initial condition $\hat \xi_0 = (\bar \e, \bar \F, s, a, s')$ by the same recursion (\ref{eq-td3})-(\ref{eq-td1}) and a trajectory $\{(\hat S_m, \hat A_m)\}$ of states and actions under the behavior policy.
This shows that
$$ \E \left[ \left\| \frac{1}{k_j} \sum_{m = t_j}^{t_j+k_j-1} \! \Big( h(\theta, \xi_m^{j}) - \bar h(\theta)  \Big) \, \I\Big((S_{t_j}, A_{t_j}, S_{t_j+1})=(s,a,s')\Big) \right\| \right] \leq \E \left[ \left\| \frac{1}{k_j} \sum_{m = 0}^{k_j-1} \left( h(\theta, \hat \xi_m) - \bar h(\theta)  \right) \right\| \right],$$
from which we see that the convergence in mean stated by (\ref{eq-prf-w8}) holds if we have
\begin{equation} \label{eq-prf-w8b}
 \lim_{k \to \infty} \, \frac{1}{k} \sum_{m = 0}^{k-1} \left( h(\theta, \hat \xi_m) - \bar h(\theta)  \right) = 0 \qquad \text{in mean}.
\end{equation} 
Now since for each $\theta$, the function $h(\theta, \cdot)$ is Lipschitz continuous in $\e$ uniformly in the other arguments, (\ref{eq-prf-w8b}) holds by Theorem~\ref{thm-2.2} and its implication Cor.\ \ref{cor-2.1}. Consequently, (\ref{eq-prf-w8}) holds, and this implies (\ref{eq-prf-w6}).

\medskip
\noindent
Proof of (\ref{eq-prf-w7}): Using the expression of $h$ and the finiteness of the state and action spaces, we can bound the difference $h(\theta, \xi_m) - h(\theta, \xi_m^{j})$ by
$$ \big\| h(\theta, \xi_m) - h(\theta, \xi_m^{j}) \big\| \leq c \cdot \big\| \e_m - \e_m^j \big\|$$
for some constant $c$ (independent of $m, j$). Let us show
\begin{equation} \label{eq-prf-w9}
 \lim_{j \to \infty} \, \frac{1}{k_j} \sum_{m = t_j}^{t_j+k_j-1} \left\| \e_m - \e_m^{j} \right\| \, \I\big(\xi_{t_j} \in D \big) = 0 \qquad \text{in mean},
\end{equation} 
which will imply (\ref{eq-prf-w7}).

To prove (\ref{eq-prf-w9}), similarly to the preceding proof, we first decompose each difference term $\e_m - \e_m^{j}$ in (\ref{eq-prf-w9}) into several difference terms, by using truncated traces $\{(\tilde \e_{m,K}, \tilde \F_{m,K})\}$ and $\{(\tilde \e_{m,K}^j, \tilde \F_{m,K}^j) \mid m \geq t_j\}$, 
$j \geq 1, K \geq 1$, which we now introduce.
Specifically, for each $K \geq 1$, $\{(\tilde \e_{m,K}, \tilde \F_{m,K})\}$ are defined by (\ref{eq-tF0})-(\ref{eq-te0}). 
For each $j \geq 1$ and $K \geq 1$, the truncated traces $\{(\tilde \e_{m,K}^j, \tilde \F_{m,K}^j) \mid m \geq t_j\}$ are also defined by (\ref{eq-tF0})-(\ref{eq-te0}), except that the initial time is set to be $t_j$ (instead of $0$) and for $m \leq t_j +K$, $(\tilde \e_{m,K}^j, \tilde \F_{m,K}^j)$ is set to be $(\e^j_m, \F^j_m)$ (instead of $(\e_m, \F_m)$).

Let us fix $K$ for now. We bound the difference $\e_m - \e_m^j$ by the sum of three difference terms as
\begin{equation} \label{eq-prf-w10a}
 \big\| \e_m - \e_m^j \big\| \leq \big\| \e_m - \tilde \e_{m, K} \big\| +  \big\| \e_m^j - \tilde \e_{m, K}^j \big\|  + \big\| \tilde \e_{m, K} - \tilde \e_{m, K}^j \big\|, 
\end{equation} 
and correspondingly, we consider the following three sequences of variables, as $j$ tends to $\infty$:
\begin{equation} \label{eq-prf-w10}
 \frac{1}{k_j} \sum_{m = t_j}^{t_j+k_j-1}  \big\| \e_m - \tilde \e_{m, K} \big\| \, \I\big(\xi_{t_j} \in D \big) , \qquad \frac{1}{k_j} \sum_{m = t_j}^{t_j+k_j-1}  \big\| \e_m^j - \tilde \e_{m, K}^j \big\|, 
\end{equation} 
and
\begin{equation} \label{eq-prf-w11}
 \frac{1}{k_j} \sum_{m = t_j}^{t_j+k_j-1} \big\| \tilde \e_{m, K} - \tilde \e_{m, K}^j \big\| \, \I\big(\xi_{t_j} \in D \big).
\end{equation}
In what follows, we will bound their expected values as $j \to \infty$ and then take $K \to \infty$; this will lead to (\ref{eq-prf-w9}).

The analyses for the two sequences in (\ref{eq-prf-w10}) are similar. Recall $D = E \times \{(s,a,s')\}$, so $\xi_{t_j} \in D$ implies $(\e_{t_j}, \F_{t_j}) \in E$. Since the set $E$ is bounded, if $(\e_{t_j}, \F_{t_j}) \in E$, then we can use Prop.\ \ref{prp-3} to bound the expectation of $\| \e_{m} - \tilde \e_{m, K} \|$ for $m \geq t_j$ conditioned on $\mathcal{F}_{t_j}$, and this gives us the bound
$$ \sup_{m \geq t_j} \E_{t_j} \left[  \big\| \e_{m} - \tilde \e_{m, K} \big\| \right] \, \I\big(\xi_{t_j} \in D \big) \leq \C_K$$
where $\C_K$ is a constant that depends on $K$ and the set $E$, and has the property that $\C_K \downarrow 0$ as $K \to \infty$.
From this bound, we obtain
\begin{equation} \label{eq-prf-w12}
 \E \left[ \frac{1}{k_j} \sum_{m = t_j}^{t_j+k_j-1} \big\| \e_{m} - \tilde \e_{m, K} \big\|  \, \I\big(\xi_{t_j} \in D \big)  \right] \leq \C_K, \qquad \forall \, j \geq 1.
\end{equation} 
Similarly, for the second sequence in (\ref{eq-prf-w10}), by Prop.\ \ref{prp-3} we have
\begin{equation} \label{eq-prf-w13}
 \E \left[ \frac{1}{k_j} \sum_{m = t_j}^{t_j+k_j-1} \big\| \e_{m}^j - \tilde \e_{m, K}^j \big\|  \right] \leq \C_K, \qquad \forall \, j \geq 1,
\end{equation} 
where $\C_K$ is some constant that can be chosen to be the same constant in (\ref{eq-prf-w12}) (because the point $(\bar \e, \bar \F)$, which is the initial trace pair for $(\e_m^j, \F_m^j)$ at time $m=t_j$, lies in $E$).

Consider now the sequence in (\ref{eq-prf-w11}). 
As discussed after the definition (\ref{eq-tF0})-(\ref{eq-te0}) of truncated traces, because of truncation, these traces lie in a bounded set determined by $K$ and the set in which the initial trace pair lies. Therefore, there exists a finite constant $c_K$ which depends on $K$ and $E$, such that for all $m \geq t_j$,
$$ \| \tilde \e_{m,K}^j \| \leq c_K, \qquad \text{and} \quad  \| \tilde \e_{m,K} \| \leq c_K  \ \ \text{if} \ (\e_{t_j}, \F_{t_j}) \in E.$$
Also by their definition, once $m$ is sufficiently large, the truncated traces do not depend on the initial trace pairs; in particular, 
$$  \tilde \e_{m,K}^j =  \tilde \e_{m,K}, \qquad \forall \, m \geq t_j + 2K + 1.$$
From these two arguments it follows that
\begin{equation} \label{eq-prf-w14}
\E \left[  \frac{1}{k_j} \sum_{m = t_j}^{t_j+k_j-1} \big\| \tilde \e_{m, K} - \tilde \e_{m, K}^j \big\| \, \I\big(\xi_{t_j} \in D \big) \right] \leq \frac{(2K+1) \cdot 2 c_K}{k_j} \to 0 \ \ \  \text{as} \  j \to \infty.
\end{equation} 

Finally, combining (\ref{eq-prf-w12})-(\ref{eq-prf-w14}) with (\ref{eq-prf-w10a}), we obtain
\begin{align*}
 \limsup_{j \to \infty} \, \E \left[ \frac{1}{k_j} \sum_{m = t_j}^{t_j+k_j-1} \left\| \e_m - \e_m^{j} \right\| \, \I\big(\xi_{t_j} \in D \big) \right] \, & \leq \, \limsup_{j \to \infty} \E \left[ \frac{1}{k_j} \sum_{m = t_j}^{t_j+k_j-1} \big\| \e_{m} - \tilde \e_{m, K} \big\|  \, \I\big(\xi_{t_j} \in D \big)  \right] \\
 & \ \quad + \limsup_{j \to \infty}  \E \left[ \frac{1}{k_j} \sum_{m = t_j}^{t_j+k_j-1} \big\| \e_{m}^j - \tilde \e_{m, K}^j \big\|  \right] \\
 & \ \quad + \lim_{j \to \infty} \E \left[  \frac{1}{k_j} \sum_{m = t_j}^{t_j+k_j-1} \big\| \tilde \e_{m, K} - \tilde \e_{m, K}^j \big\| \, \I\big(\xi_{t_j} \in D \big) \right]
 \\
& \leq \, 2 \C_K.
\end{align*}
Since $\C_K \downarrow 0$ as $K \to \infty$ (Prop.\ \ref{prp-3}), by taking $K \to \infty$, we obtain
$$  \lim_{j \to \infty} \, \E \left[ \frac{1}{k_j} \sum_{m = t_j}^{t_j+k_j-1} \left\| \e_m - \e_m^{j} \right\| \, \I\big(\xi_{t_j} \in D \big) \right] = 0.$$
This proves (\ref{eq-prf-w9}), which implies (\ref{eq-prf-w7}).
\end{proof}
\smallskip

With Props.~\ref{prop-1}-\ref{prop-3}, we have furnished all the conditions required in order to apply \cite[Theorems 8.2.2, 8.2.3]{KuY03} to the constrained ETD algorithm (\ref{eq-emtd-const0}), so we can now specialize the conclusions of these two theorems to our problem.
In particular, they tell us that the projected ODE (\ref{eq-pode}) is the mean ODE for (\ref{eq-emtd-const0}), and furthermore, by \cite[Theorem 8.2.3]{KuY03} (respectively, \cite[Theorem 8.2.2]{KuY03}), the conclusions of Theorem~\ref{thm-dim-stepsize} (respectively, Theorem~\ref{thm-const-stepsize}) hold with $N_\delta(L_\H)$ in place of $N_\delta(\theta^*)$, where $N_\delta(L_\H)$ is the $\delta$-neighborhood of the limit set $L_\H$ for the projected ODE (\ref{eq-pode}). 
Recall that this limit set is given by
$$ L_\H = \cap_{\bar \tau > 0} \, \overline{\, \cup_{x(0) \in \H} \{ x(\tau), \, \tau \geq \bar \tau \}}$$
where $x(\tau)$ is a solution of the projected ODE (\ref{eq-pode}) with initial condition $x(0)$, the union is over all the solutions with initial $x(0) \in \H$, and $\overline{D}$ for a set $D$ denotes taking the closure of $D$.

Now when the matrix $C$ is negative definite (as implied by Assumptions~\ref{cond-bpolicy},~\ref{cond-features}) and when the radius of $\H$ exceeds the threshold given in Lemma~\ref{lma-pode}, by the latter lemma, the solutions $x(\tau), \tau \in [0, \infty),$ of the ODE (\ref{eq-pode}) coincide with the solutions of $\dot{x} = \bar h(x) = C x + b$ for all initial $x(0) \in \H$. Then from the negative definiteness of $C$ (Theorem~\ref{thm-matrix}), it follows that as $\tau \to \infty$, $x(\tau) \to \theta^*$ uniformly in the initial condition, and consequently, $L_\H = \{\theta^*\}$.%footnote starts
\footnote{The details for this statement are as follows. 
Since $\bar h$ is bounded on $\H$ and the boundary reflection term $z(\cdot) \equiv 0$ under our assumptions (Lemma~\ref{lma-pode}), a solution $x(\cdot)$ of (\ref{eq-pode}) is Lipschitz continuous on $[0, \infty)$. 
We calculate $\dot{V}(\tau)$ for the Lyapunov function $V(\tau) = | x(\tau) - \theta^*|^2$.
By the negative definiteness of the matrix $C$, for some $c > 0$, $x^\top C x \leq - c | x|^2$ for all $x \in \rn$. Then, since $\bar h(x) = C x + b = C (x - \theta^*)$, we have
$\dot{V}(\tau) = 2 \, \big\langle x(\tau) - \theta^* \, , \, \bar h(x(\tau)) \big\rangle  
          \leq - 2 c \, \big|x(\tau) - \theta^* \big|^2$,
and hence for any $\delta > 0$, there exists $\epsilon > 0$ such that 
$\dot{V}(\tau) \leq - \epsilon$ if $V(\tau) = |x(\tau) - \theta^* |^2 \geq \delta^2$.
This together with the continuity of the solution $x(\cdot)$ implies that for any $x(0) \in \H$, within time $\bar \tau = r_\H^2/\epsilon$, 
the trajectory $x(\tau)$ must reach $N_\delta(\theta^*)$ and stay in that set thereafter. By the definition of the limit set and the arbitrariness of $\delta$, this implies $L_\H = \{ \theta^* \}$.}
%footnote ends
Thus $N_\delta(L_\H) = N_\delta(\theta^*)$ and we obtain Theorems~\ref{thm-dim-stepsize} and~\ref{thm-const-stepsize}.

\subsection{Proofs for Theorems~\ref{thm-dim-stepsize-b} and~\ref{thm-const-stepsize-b}} \label{sec-4.2}

In this subsection we prove Theorems~\ref{thm-dim-stepsize-b}-\ref{thm-const-stepsize-b} for the two variants of the constrained ETD($\lambda$) algorithm given in (\ref{eq-emtd-const1}) and (\ref{eq-emtd-const2}). Like in the previous subsection, we will apply \cite[Theorems 8.2.2, 8.2.3]{KuY03} and show, separately for each variant algorithm, that the required conditions are met. Using the properties of the mean ODEs of the variant algorithms, we will then specialize the conclusions of those theorems to obtain the desired results.

\subsubsection{Proofs for the First Variant}
Consider the first variant algorithm (\ref{eq-emtd-const1}):
$$\theta_{t+1} = \Pi_{\H} \Big( \theta_t + \alpha_t \, \psi_K(\e_t) \cdot \rho_t \big(R_t + \gamma_{t+1} \fe(S_{t+1})^\top \theta_t - \fe(S_t)^\top \theta_t \big) \Big).$$
We define a function $h_K : \rn \times \Xi \to \rn$ by
\begin{equation} \label{eq-hK}
 h_K(\theta, \xi) = \psi_K(\e ) \cdot \rho(s, a) \, \big( r(s, a, s') + \gamma(s') \, \fe(s')^\top \theta - \fe(s)^\top \theta \big), \quad \text{for} \ \xi = (\e, \F, s, a, s'),
\end{equation} 
and write (\ref{eq-emtd-const1}) equivalently as
$$  \theta_{t+1} = \Pi_{\H} \Big( \theta_t + \alpha_t \, h_K(\theta_t, \xi_t) + \alpha_t \, \psi_K(\e_t ) \cdot \tilde \omega_{t+1} \Big) $$
with $\tilde \omega_{t+1}  = \rho_t (R_t - r(S_t, A_t, S_{t+1}))$ as before. Note that $\E_t \left[ \psi_K(\e ) \, \tilde \omega_{t+1} \right] = 0$, and the algorithm is similar to the algorithm (\ref{eq-emtd-const0}) (equivalently, (\ref{eq-emtd-const})), except that we have $h_K$ and $\psi_K(\e_t)$ in place of $h$ and $\e_t$, respectively. 

We note two properties of the function $h_K$. They follow from direct calculations and will be useful in our analysis shortly:
\begin{enumerate}
\item[(a)] Using the Lipschitz continuity of the function $\psi_K$ (cf.\ (\ref{eq-psi})), we have that
for each $\theta \in \rn$, there exists a finite $c > 0$ such that with $\xi = (\e, \F, s, a, s')$ and $\xi' = (\e', \F', s, a, s')$, 
\begin{equation} \label{eq-prf-wa1}
  \| h_K(\theta, \xi) - h_K(\theta, \xi') \| \leq c \, \| \e - \e' \|, \qquad \forall \, (s, a, s') \in \S \times \A \times \S.
\end{equation}  
Thus $h_K(\theta, \cdot)$ is Lipschitz continuous in $(\e,\F)$ uniformly in $(s,a,s')$. 
\item[(b)] Since the set $\H$ is bounded, we can bound the difference $h_K(\theta, \xi) - h(\theta, \xi)$ for all $\theta$ in $\H$ as follows.
For some finite constant $c > 0$,  
\begin{equation}\label{eq-prf-wa2}
 \| h_K(\theta, \xi) - h(\theta, \xi) \| \leq c \, \| \psi_K(\e) - \e \| \leq 2 c \, \| \e \| \cdot \I (\| \e \| \geq K), \qquad \forall \, \theta \in \H,
\end{equation} 
where the last inequality follows from the property (\ref{eq-psi}) of $\psi_K$:
$$  \| \psi_K(x) \| \leq \| x \| \ \  \forall \, x \in \rn, \quad \text{and} \quad \psi_K(x) = x \ \ \text{if} \ \| x \| \leq K.$$
\end{enumerate}

We now apply \cite[Theorems 8.2.2, 8.2.3]{KuY03} to obtain the desired conclusions in Theorems~\ref{thm-dim-stepsize-b}-\ref{thm-const-stepsize-b} for the algorithm (\ref{eq-emtd-const1}). This requires us to show that the conditions (i)-(v) and (i$'$)-(v$'$) given in Section~\ref{sec-4.1.1} are still satisfied when we replace $\e_t$ by $\psi_K(\e_t)$ and $h$ by $h_K$. 
The uniform integrability conditions (i), (i$'$), (iv) and (iv$'$) require the following sets to be u.i.: $\{ h_K(\theta_t, \xi_t) +  \psi_K(\e_t) \cdot \tilde \omega_{t+1}\}$ and $\{ h_K(\theta_t^\alpha, \xi_t) +  \psi_K(\e_t) \cdot \tilde \omega_{t+1} \mid t \geq 0, \alpha > 0\}$,  $\{h_K(\theta_t, \xi_t)\}$ and $\{h_K(\theta_t^\alpha, \xi_t) \mid t \geq 0, \alpha > 0\}$, and $\{h_K(\theta, \xi_t)\}$ for each $\theta$. These conditions are evidently satisfied, in view of the boundedness of the functions $\psi_K$ and $h_K(\theta, \cdot)$ for each $\theta$, the boundedness of the $\theta$-iterates due to constraints, and the finite variances of $\{\tilde \omega_{t}\}$.
The condition (ii) on the continuity of $h_K(\cdot, \xi)$ uniformly in $\xi \in D$, for each compact set $D \subset \Xi$, is also clearly satisfied, whereas the condition (iii) (equivalently (iii$'$)) on the tightness of $\{\xi_t\}$ was already verified earlier in Prop.~\ref{prop-1}.

What remains is the condition (v) (which is equivalent to (v$'$), for the same reason as discussed immediately before Prop.~\ref{prop-3}). It requires the existence of a continuous function $\bar h_K: \rn \to \rn$ such that for each $\theta \in \H$ and each compact set $D \subset \Xi$,
\begin{equation} \label{eq-condv-v1}
 \lim_{k \to \infty, t \to \infty}  \, \frac{1}{k} \sum_{m = t}^{t+k-1} \E_t \big[ h_K(\theta, \xi_m) - \bar h_K(\theta) \big] \, \I\big(\xi_t \in D\big) = 0 \qquad \text{in mean}.
\end{equation}
If this condition is satisfied as well, then the mean ODE for the algorithm (\ref{eq-emtd-const1}) is given by
\begin{equation} \label{eq-podeK}
  \dot{x} = \bar h_K(x) + z, \qquad z \in - \mathcal{N}_\H(x).
\end{equation}  

To furnish the condition (v), we first identify the function $\bar h_K(\theta)$ to be $\E_\zeta [h_K(\theta, \xi_0)]$, the expectation of $h_K(\theta, \xi_0)$ under the stationary distribution of the process $\{\Z_t\}$ with the invariant probability measure $\zeta$ as its initial distribution. We relate the functions $\bar h_K, K > 0$, to $\bar h$ in the proposition below, and we will use it to characterize the bias of the algorithm (\ref{eq-emtd-const1}) later.

\begin{prop} \label{prop-4}
Let Assumption~\ref{cond-bpolicy} hold. Consider the setting of the algorithm (\ref{eq-emtd-const1}), and for each $\theta \in \rn$, let $\bar h_K(\theta) = \E_\zeta [ h_K(\theta, \xi_0) ]$. Then the function $\bar h_K$ is Lipschitz continuous on $\rn$, and
\begin{equation} \label{eq-approx}
 \sup_{\theta \in \H} \, \| \bar h_K(\theta) - \bar h(\theta) \| \to 0 \ \ \text{as} \ K \to \infty.
\end{equation} 
\end{prop}

\begin{proof}
For each $\theta$, the function $h_K(\theta, \cdot)$ is by definition bounded. Under Assumption~\ref{cond-bpolicy}, the Markov chain $\{(S_t, A_t, \e_t, \F_t)\}$ has a unique invariant probability measure $\zeta$ (Theorem~\ref{thm-2.1}). Therefore, $\bar h_K(\theta)$ is well-defined and finite. 
Let $c_1 = \sup_{\e \in \rn} \| \psi_K(\e) \| < \infty$ (since $\psi_K$ is bounded). 
For any $\theta, \theta'$, using the definition of $h_K$, a direct calculation shows that for some $c_2 > 0$,
$\| h_K(\theta, \xi) - h_K(\theta', \xi) \| \leq c_1 c_2 \| \theta - \theta' \|$ for all $\xi \in \Xi$,
from which it follows that
$$  \| \bar h_K(\theta) - \bar h_K(\theta') \| \leq \E_\zeta \left[ \| h_K(\theta, \xi_0) - h_K(\theta', \xi_0) \| \right] \leq c_1 c_2 \| \theta - \theta' \|.$$
This shows that $\bar h_K$ is Lipschitz continuous.
We now prove (\ref{eq-approx}). Since $\bar h_K(\theta) = \E_\zeta [ h_K(\theta, \xi_0) ]$ by definition and $\bar h(\theta) = \E_\zeta [ h(\theta, \xi_0) ]$ by Cor.~\ref{cor-2.1},  it is sufficient to prove the following statement, which entails (\ref{eq-approx}): 
\begin{equation} \label{eq-prf-wa3}
 \sup_{\theta \in \H}  \E_\zeta \left[  \big\| h_K(\theta, \xi_0) - h(\theta, \xi_0) \big\| \right] \to 0 \ \ \text{as} \ K \to \infty.
\end{equation}  
By (\ref{eq-prf-wa2}), for some constant $c > 0$,  
$$\| h_K(\theta, \xi_0) - h(\theta, \xi_0) \| \leq 2 c \, \| \e_0 \| \cdot \I (\| \e_0 \| \geq K), \qquad \forall \, \theta \in \H,$$
and therefore,
$$ \sup_{\theta \in \H}  \E_\zeta \left[  \big\| h_K(\theta, \xi_0) - h(\theta, \xi_0) \big\| \right] \leq 2 c \, \E_\zeta \big[ \| \e_0 \| \cdot \I (\| \e_0 \| \geq K) \big].$$
By Theorem~\ref{thm-2.2}, $\E_\zeta [ \| \e_0\| ] < \infty$ and hence $\E_\zeta [ \| \e_0 \| \cdot \I (\| \e_0 \| \geq K) ] \to 0$ as $K \to \infty$. Together with the preceding inequality, this implies (\ref{eq-prf-wa3}), which in turn implies (\ref{eq-approx}).
\end{proof}

We now show that the convergence in mean required in (\ref{eq-condv-v1}) is satisfied.

\begin{prop} \label{prop-5}
Under Assumption~\ref{cond-bpolicy}, the conclusion of Prop.~\ref{prop-3} holds in the setting of the algorithm (\ref{eq-emtd-const1}), with the functions $h_K$ and $\bar h_K$ in place of $h$ and $\bar h$, respectively.
\end{prop}

\begin{proof}
The same arguments given in the proof of Prop.~\ref{prop-3} apply here, with the functions $h_K, \bar h_K$ in place of $h, \bar h$, respectively. 
Only two details are worth noting here. The proof relies on the Lipschitz continuity property of $h_K$ given in (\ref{eq-prf-wa1}). 
As mentioned earlier, this property implies that for each $\theta$, with $\xi=(\e, \F, s, a, s')$, $h_K(\theta, \xi)$ is Lipschitz continuous in $(\e,\F)$ uniformly in $(s,a,s')$, so we can apply Theorem~\ref{thm-2.2} to conclude that (\ref{eq-prf-w8b}) and hence (\ref{eq-prf-w6}) hold in this case (for $h_K, \bar h_K$ instead of $h, \bar h$). The property (\ref{eq-prf-wa1}) also allows us to obtain (\ref{eq-prf-w7}) in this case, by exactly the same proof given earlier.
\end{proof}

Thus we have furnished all the conditions required by \cite[Theorems 8.2.2, 8.2.3]{KuY03}. 
As in the case of (\ref{eq-emtd-const0}), by these two theorems, the assertions of Theorems~\ref{thm-dim-stepsize} and \ref{thm-const-stepsize} hold for the variant algorithm (\ref{eq-emtd-const1}) with $N_\delta(L_\H)$ in place of $N_\delta(\theta^*)$, where $L_\H$ is the limit set of the 
projected mean ODE associated with (\ref{eq-emtd-const1}):
\begin{equation} 
  \dot{x} = \bar h_K(x) + z, \qquad z \in - \mathcal{N}_\H(x). \notag
\end{equation}  
To finish the proof for Theorems~\ref{thm-dim-stepsize-b}-\ref{thm-const-stepsize-b}, it is now sufficient to show that for any given $\delta  > 0$, we can choose a number $K_\delta$ large enough so that $L_\H \subset N_{\delta}(\theta^*)$ for all $K \geq K_\delta$. 
We prove this below, using Prop.~\ref{prop-4}. Note that the set $L_\H$ reflects the bias of the constrained algorithm (\ref{eq-emtd-const1}), so what we are showing now is that this bias decreases as $K$ increases.

\begin{lem} \label{lma-3.2}
Let Assumptions~\ref{cond-bpolicy},~\ref{cond-features} hold, and let the radius of the set $\H$ exceed the threshold given in Lemma~\ref{lma-pode}. Then for all $K$ sufficiently large, given any initial condition $x(0) \in \H$, a solution to the projected ODE (\ref{eq-podeK}) coincides with the unique solution to $\dot{x} = \bar h_K(x)$, with the boundary reflection term being $z(\cdot) \equiv 0$.
Given $\delta > 0$, there exists $K_\delta$ such that for $K \geq K_\delta$, the limit set $L_\H$ of (\ref{eq-podeK}) satisfies $L_\H \subset N_\delta(\theta^*)$.
\end{lem}

\begin{proof}
Under Assumptions~\ref{cond-bpolicy},~\ref{cond-features}, the matrix $C$ is negative definite (Theorem~\ref{thm-matrix}), and when the radius of the set $\H$ exceeds the threshold given in Lemma~\ref{lma-pode}, there exists a constant $\epsilon > 0$ such that for all boundary points $x$ of $\H$, $\langle x, \bar h(x) \rangle < - \epsilon$. 
At such points $x$, the normal cone $\mathcal{N}_\H(x) = \{ a x \mid a \geq 0\}$, and 
$$ \langle x, \bar h_K(x) \rangle = \langle x, \bar h(x) \rangle + \langle x, \bar h_K(x) - \bar h(x)  \rangle < - \epsilon + \langle x, \bar h_K(x) - \bar h(x)  \rangle.$$
By (\ref{eq-approx}) in Prop.~\ref{prop-4}, $\langle x, \bar h_K(x) - \bar h(x)  \rangle \to 0$ uniformly on $\H$ as $K \to \infty$.
Thus when $K$ is sufficiently large, at all boundary points $x$ of $\H$, $\langle x, \bar h_K(x) \rangle < 0$; i.e., $\bar h_K(x)$ points inside $\H$ and the boundary reflection term $z = 0$.
It then follows that for such $K$, given an initial condition $x(0) \in \H$, a solution to (\ref{eq-podeK}) coincides with the unique solution to $\dot{x} = \bar h_K(x)$, where the uniqueness is ensured by the Lipschitz continuity of $\bar h_K$ proved in Prop.~\ref{prop-4} (cf.\ \cite[Chap.\ 11.2]{Bor08}). 
 
To prove the second statement concerning the limit set of the projected ODE, let $K$ be large enough so that the conclusion of the first part holds. Let $x(\tau), \tau \in [0, \infty),$ be the solution of (\ref{eq-podeK}) for a given initial $x(0) \in \H$. Since $\bar h_K$ is bounded on $\H$, $x(\cdot)$ is Lipschitz continuous on $[0, \infty)$.  
Let $V(\tau) = | x(\tau) - \theta^*|^2$, and we calculate $\dot{V}(\tau)$.
Since for all $x$, $\bar h(x) = C x + b = C (x - \theta^*)$ and $x^\top C x \leq - c | x|^2$ for some $c > 0$ by the negative definiteness of $C$, a direct calculation shows that
\begin{align*}
    \dot{V}(\tau) & = 2 \, \big\langle x(\tau) - \theta^* \, , \, \bar h_K(x(\tau)) \big\rangle = 2 \, \big\langle x(\tau) - \theta^* \, , \, \bar h(x(\tau)) \big\rangle + 2 \,\big\langle x(\tau) - \theta^* , \bar h_K(x(\tau)) - \bar h (x(\tau))  \big\rangle \\    
     & \leq - 2 c \, \big|x(\tau) - \theta^* \big|^2 + 2 \, \big| x(\tau) - \theta^* \big| \cdot \big| \bar h_K(x(\tau)) - \bar h(x(\tau)) \big|.
\end{align*}
By (\ref{eq-approx}) in Prop.~\ref{prop-4}, $\sup_{x \in \H} |h_K(x) - \bar h(x)| \to 0$ as $K \to \infty$. It then follows that for any $\delta > 0$, there exist $\epsilon > 0$ and $K_\delta > 0$ such that for all $K \geq K_\delta$,
$\dot{V}(\tau) \leq - \epsilon$ if $V(\tau) = |x(\tau) - \theta^* |^2 \geq \delta^2$.
This together with the continuity of the solution $x(\cdot)$ shows that for any $x(0) \in \H$, within time $\bar \tau = r_B^2/\epsilon$ (where $r_B$ is the radius of $\H$), the trajectory $x(\tau)$ must reach $N_\delta(\theta^*)$ and stay in that set thereafter. Consequently, for all $K \geq K_\delta$, the limit set $L_\H = \cap_{\bar \tau \geq 0} \overline{ \, \cup_{x(0) \in \H} \{ x(\tau), \tau \geq \bar \tau \}} \subset N_\delta(\theta^*)$.  
\end{proof}

This completes the proofs of Theorems~\ref{thm-dim-stepsize-b} and \ref{thm-const-stepsize-b} for the first variant.

\subsubsection{Proofs for the Second Variant}

We now analyze the second variant algorithm (\ref{eq-emtd-const2}), 
$$ \theta_{t+1} = \Pi_{\H} \left( \theta_t + \alpha_t \, \psi_K(Y_t) \right), \quad \text{where} \ \ Y_t = \e_t \cdot \rho_t \big(R_t + \gamma_{t+1} \fe(S_{t+1})^\top \theta_t - \fe(S_t)^\top \theta_t \big).$$
Similarly to the previous case, with $\xi = (\e, \F, s, a, s')$, we define a bounded function $h_K: \rn \times \Xi \to \rn$ by
$$ h_K(\theta, \xi) =  \int   \psi_K \Big( \e  \cdot \rho(s, a) \, \big( r + \gamma(s') \, \fe(s')^\top \theta - \fe(s)^\top \theta \big) \Big) \, q( d r  \mid s, a, s'), $$
where we recall that $q( d r  \mid s, a, s')$ is the conditional probability distribution of the reward given the state transition $(s,s')$ under the action $a$.
We can write the algorithm (\ref{eq-emtd-const2}) equivalently in terms of $h_K$ as
$$  \theta_{t+1} = \Pi_{\H} \Big( \theta_t + \alpha_t \, h_K(\theta_t, \xi_t) + \alpha_t \, \Delta_t \Big), $$
where $\Delta_t = \psi_K(Y_t) -  h_K(\theta_t, \xi_t)$, and it can be seen that $h_K(\theta_t, \xi_t) = \E_t [ \psi_K(Y_t) ]$ and $\E_t [ \Delta_t ] = 0$.

Two properties of the function $h_K$ will be useful shortly in our analysis:
\begin{enumerate}
\item[(a)] The Lipschitz continuity property (\ref{eq-prf-wa1}) holds for the function $h_K$ here. In particular, let $c_1 > 0$ be the Lipschitz modulus of the function $\psi_K$ with respect to $\| \cdot \|$.
A direct calculation using the Lipschitz property of $\psi_K$ shows that with $\xi = (\e, \F, s, a, s')$ and $\xi' = (\e', \F', s, a, s')$,
\begin{align*}
 \|   h_K(\theta, \xi) - h_K(\theta, \xi')  \| \leq \int c_1 \Big\| ( \e - \e') \cdot \rho(s, a) \, \big( r + \gamma(s') \, \fe(s')^\top \theta - \fe(s)^\top \theta \big) \Big\| \, q( d r  | s, a, s'),
\end{align*}
so for each $\theta$, there exists a finite constant $c > 0$ such that
\begin{equation} \label{eq-prf-wb1}
  \| h_K(\theta, \xi) - h_K(\theta, \xi') \| \leq c \, \| \e - \e' \|, \qquad \forall \, (s, a, s') \in \S \times \A \times \S.
\end{equation}  
\item[(b)] The second property given below also follows from a direct calculation using the Lipschitz continuity of $\psi_K$:  there exists a finite constant $c > 0$ such that for any $\theta, \theta'$, 
\begin{equation} \label{eq-prf-wb2} 
\| h_K(\theta, \xi) - h_K(\theta', \xi) \| \leq c \| \e \| \cdot \| \theta - \theta' \|, \qquad \forall \,  \xi \in \Xi.
\end{equation}
\end{enumerate}

We now proceed to prove Theorems~\ref{thm-dim-stepsize-b}-\ref{thm-const-stepsize-b} for the algorithm (\ref{eq-emtd-const2}).
As before, we will apply \cite[Theorems 8.2.2, 8.2.3]{KuY03}, and this requires us to show that the conditions (i)-(v) and (i$'$)-(v$'$) given in Section~\ref{sec-4.1.1}, with the function $h_K$ above in place of $h$, are satisfied.
The conditions (i)-(iv) and (i$'$)-(iv$'$) are clearly met. In particular, the uniform integrability conditions (i), (i$'$), (iv) and (iv$'$) are trivially fulfilled because by the definitions of $h_K$ and the algorithm (\ref{eq-emtd-const2}),  $\{ \psi_K(Y_t)\}$, 
$\{h_K(\theta_t, \xi_t)\}$, and $\{h_K(\theta, \xi_t)\}$ for each $\theta$, regardless of stepsizes, all lie in a bounded set determined by $K$. As for the continuity condition (ii), in view of the boundedness and Lipschitz continuity of $\psi_K$, it is also clear that $h_K(\theta, \xi)$ is bounded and continuous in $\theta$ uniformly in $\xi \in D$, for each compact set $D \subset \Xi$ (cf.\ (\ref{eq-prf-wb2})).

The condition (v) (equivalently (v$'$))  requires the existence of a continuous function $\bar h_K: \rn \to \rn$ such that for each $\theta \in \H$ and each compact set $D \subset \Xi$,
\begin{equation} \label{cond-v-v2}
 \lim_{k \to \infty, t \to \infty}  \, \frac{1}{k} \sum_{m = t}^{t+k-1} \E_t \big[ h_K(\theta, \xi_m) - \bar h_K(\theta) \big] \, \I\big(\xi_t \in D\big) = 0 \qquad \text{in mean}.
\end{equation} 
Similarly to the analysis for the first variant algorithm, we identify this function $\bar h_K(\theta)$ to be the expectation of $h_K(\theta, \xi_0)$ with respect to the stationary distribution of the process $\{Z_t\}$, and if condition (v) is satisfied, then the mean ODE of the algorithm (\ref{eq-emtd-const2}) will be given by
$$  \dot{x} = \bar h_K(x) + z, \qquad z \in - \mathcal{N}_\H(x). \notag$$

\begin{prop} \label{prop-6}
Let Assumption~\ref{cond-bpolicy} hold. Consider the setting of the algorithm (\ref{eq-emtd-const2}), and for each $\theta \in \rn$, let $\bar h_K(\theta) = \E_\zeta [ h_K(\theta, \xi_0) ]$. Then the conclusion of Prop.~\ref{prop-4} holds.
\end{prop}

\begin{proof}
The function $h_K$ is by definition bounded, and under Assumption~\ref{cond-bpolicy}, the Markov chain $\{(S_t, A_t, \e_t, \F_t)\}$ has a unique invariant probability measure $\zeta$ (Theorem~\ref{thm-2.1}). The function $\bar h_K(\theta)$ is therefore well-defined and bounded. 
Using the property (\ref{eq-prf-wb2}) of the function $h_K$, we have that there exists a finite constant $c > 0$ such that for any $\theta, \theta'$,
$$  \| \bar h_K(\theta) - \bar h_K(\theta') \| \leq \E_\zeta \left[ \big\| h_K(\theta, \xi_0) - h_K(\theta', \xi_0) \big\| \right] \leq c \, \E_\zeta [ \| \e_0\| ] \cdot \| \theta - \theta' \|.$$
Since $\E_\zeta\{ \| \e_0\| \} < \infty$ by Theorem~\ref{thm-2.2}, this shows that $\bar h_K$ is Lipschitz continuous.

To prove (\ref{eq-approx}), let us prove 
\begin{equation} \label{eq-prf-wa4}
  \sup_{\theta \in \H}  \, \E_{\zeta'} \left[ \big\| \hat h_K(\theta, \xi_0, R_0) - \hat h(\theta, \xi_0, R_0) \big\| \right] \to 0 \ \ \text{as} \ K \to \infty,
\end{equation}  
where $\hat h_K, \hat h : \rn \times \Xi \times \re \to \rn$ are defined by
\begin{align*}
\hat h_K(\theta, \xi, r) & : = \psi_K \Big( \e  \cdot \rho(s, a) \, \big( r + \gamma(s') \, \fe(s')^\top \theta - \fe(s)^\top \theta \big) \Big), \\
\hat h(\theta, \xi, r) & : = \e  \cdot \rho(s, a) \, \big( r + \gamma(s') \, \fe(s')^\top \theta - \fe(s)^\top \theta \big),
\end{align*} 
and $\zeta'$ denotes the unique invariant probability measure of the Markov chain $\{(\xi_t, R_{t})\}$, the existence and uniqueness of such a measure being implied by Theorem~\ref{thm-2.1}, and the expectation $\E_{\zeta'}$ is over $(\xi_0, R_0)$ with respect to $\zeta'$. 
By taking expectation over $R_0$ conditioned on $\xi_0$, we have
$$  \E_{\zeta'} \big[ \hat h_K(\theta, \xi_0, R_0) \big] = \E_\zeta \big[ h_K(\theta, \xi_0) \big] = \bar h_K(\theta), \qquad \E_{\zeta'} \big[ \hat h(\theta, \xi_0, R_0) \big] = \E_\zeta \big[ h(\theta, \xi_0) \big] = \bar h(\theta),$$
so (\ref{eq-prf-wa4}) implies (\ref{eq-approx}).

We now prove (\ref{eq-prf-wa4}). Note that $\hat h_K(\theta, \xi_0, R_0)  = \psi_K( \hat h(\theta, \xi_0, R_0))$. So using the property (\ref{eq-psi}) of $\psi_K$, we have for any $\theta$,
\begin{align*} 
 \big\| \hat h_K(\theta, \xi_0, R_0) - \hat h(\theta, \xi_0, R_0) \big\| & \leq 2 \big\| \hat h(\theta, \xi_0, R_0) \big\| \cdot \I \big(  \big\| \hat h(\theta, \xi_0, R_0) \big\| \geq K \big), 
\end{align*}
and using the definition of $\hat h$ and the boundedness of $\H$, we also have that for some constants $c_1, c_2 > 0$,
$$ \big\| \hat h(\theta, \xi_0, R_0) \big\|  \leq c_1 \| \e_0 R_0 \| + c_2 \| \e_0 \| = : g(\e_0, R_0), \qquad \forall \, \theta \in \H.$$
Combining the preceding two relations, we have for any $\theta \in \H$,
$$  \big\| \hat h_K(\theta, \xi_0, R_0) - \hat h(\theta, \xi_0, R_0) \big\| \leq  2 \, g(\e_0, R_0)  \cdot \I \big( g(\e_0, R_0) \geq K \big)$$
and hence
\begin{equation} \label{eq-prf-wb3}
  \sup_{\theta \in \H}  \, \E_{\zeta'} \left[ \big\| \hat h_K(\theta, \xi_0, R_0) - \hat h(\theta, \xi_0, R_0) \big\| \right]  \leq 2 \, \E_{\zeta'} \left[ g(\e_0, R_0)  \cdot \I \big( g(\e_0, R_0) \geq K \big) \right].
\end{equation}  
Since $\E_{\zeta'} \left[ g(\e_0, R_0)  \right] = \E_{\zeta'} \left[  c_1 \| \e_0 R_0 \| + c_2 \| \e_0 \| \right] < \infty$ (we obtain the finiteness of the expectation here by first taking expectation over $R_0$ conditioned on $\xi_0$ and then applying Theorem~\ref{thm-2.2}), the expectation on the right-hand side of (\ref{eq-prf-wb3}) converges to $0$ as $K \to \infty$. We thus obtain (\ref{eq-prf-wa4}), which implies (\ref{eq-approx}).
\end{proof}

The rest of the analysis is similar to that for the first variant algorithm. First, using the Lipschitz continuity property (\ref{eq-prf-wb1}) of the function $h_K$ given earlier, we obtain that the convergence-in-mean condition (\ref{cond-v-v2}) holds, by the same proof arguments for Prop.~\ref{prop-3}:

\smallskip
\begin{prop} \label{prop-7}
Under Assumption~\ref{cond-bpolicy} and in the setting of the algorithm (\ref{eq-emtd-const2}), (\ref{cond-v-v2}) holds for each $\theta \in \H$ and each compact set $D \subset \Xi$.
\end{prop}

Now we have furnished all the conditions required by \cite[Theorems 8.2.2, 8.2.3]{KuY03}. 
By these two theorems, we can assert that the conclusions of Theorems~\ref{thm-dim-stepsize}-\ref{thm-const-stepsize} hold for the variant algorithm (\ref{eq-emtd-const2}) with $N_\delta(L_\H)$ in place of $N_\delta(\theta^*)$, where $L_\H$ is the limit set of the 
projected mean ODE associated with (\ref{eq-emtd-const2}):
\begin{equation} 
  \dot{x} = \bar h_K(x) + z, \qquad z \in - \mathcal{N}_\H(x). \notag
\end{equation}  
So to finish the proof for Theorems~\ref{thm-dim-stepsize-b}-\ref{thm-const-stepsize-b}, it is sufficient to show that for any given $\delta  > 0$, we can choose a number $K_\delta$ large enough so that $L_\H \subset N_{\delta}(\theta^*)$ for all $K \geq K_\delta$. In other words, the conclusions of Lemma~\ref{lma-3.2} hold for the case of the algorithm (\ref{eq-emtd-const2}). This is true by the same proof for Lemma~\ref{lma-3.2} with Prop.~\ref{prop-6} in place of Prop.~\ref{prop-4}. 
This completes the proofs of Theorems~\ref{thm-dim-stepsize-b}-\ref{thm-const-stepsize-b} for the second variant.

\subsection{Further Analysis of the Constant-stepsize Case} \label{sec-4.3}
We now consider again the case of constant stepsize, and prove Theorems~\ref{thm-c1}-\ref{thm-c2b} given in Section~\ref{sec-3.3}.
The proofs will be based on combining the results we obtained earlier by using the stochastic approximation theory, with the ergodic theorems of weak Feller Markov chains. As before the proofs will also rely on the key properties of the ETD iterates.

\subsubsection{Weak Feller Markov Chains} \label{sec-4.3.1}

We shall focus on Markov chains on complete separable metric spaces. 
For such a Markov chain $\{X_t\}$ with state space $\X$, 
let $P(\cdot, \cdot)$ denote its transition kernel, that is, $P: \X \times \mathcal{B}(\X) \to [0,1]$,
$$ P(x, D) = \Pr_x (X_1 \in D ), \qquad \forall \, x \in \X, \ D \in  \mathcal{B}(\X), $$
where $\mathcal{B}(\X)$ denotes the Borel sigma-algebra on $\X$, and $\Pr_x$ denotes the probability distribution of $\{X_t\}$ conditioned on $X_0 = x$.
Multiple-step transition kernels will also be needed. For $t \geq 1$, the $t$-step transition kernel $P^t(\cdot, \cdot) : \X \times \mathcal{B}(\X) \to [0,1]$ is given by
$$ P^t(x, D) = \Pr_x (X_t \in D), \qquad \forall \, x \in \X, \ D \in  \mathcal{B}(\X),$$
and for $t=0$, $P^0$ is defined as $P^0(x, \cdot) = \delta_{x}$, the Dirac measure that assigns probability $1$ to the point $x$, for each $x \in \X$.
Define averaged probability measures $\bar P_k(x, \cdot)$ for $k \geq 1$ and $x \in \X$, as
$$ \bar P_k(x, \cdot) = \frac{1}{k} \sum_{t=0}^{k-1} P^t(x, \cdot).$$

The Markov chain $\{X_t\}$ has the \emph{weak Feller} property if for every bounded continuous function $f$ on $\X$, 
$$ Pf(x) : = \int f(y) P(x, dy) = \E \big[ f(X_1) \mid X_0 = x \big]$$ 
is a continuous function of $x$~\cite[Prop.\ 6.1.1]{MeT09}. 
Weak Feller Markov chains have nice properties. In our analysis, we will use in particular several properties relating to the invariant probability measures of these chains and convergence of certain probability measures to the invariant probability measures.

Recall that if $\mu$ and $\mu_t, t \geq 0$, are probability measures on $\X$,  $\{\mu_t\}$ is said to converge weakly to $\mu$ if $\int f d\mu_t \to \int f d\mu$ for every bounded continuous function $f$ on $\X$. For $\{\mu_t\}$ that is not necessarily convergent, we shall call the limiting probability measure of any of its convergent subsequence, in the sense of weak convergence, a \emph{weak limit} of $\{\mu_t\}$. 
For an (arbitrary) index set $\mathcal{K}$, a set of probability measures $\{\mu_k\}_{k \in \mathcal{K}}$ on $\X$ is said to be \emph{tight} if for every $\delta > 0$, there exists a compact set $D_\delta \subset \X$ such that $\mu_k(D_\delta) \geq 1 - \delta$ for all $k \in \mathcal{K}$. 
An important fact is that on a complete separable metric space, any tight sequence of probability measures has a further subsequence that converges weakly to some probability measure~\cite[Theorem 11.5.4]{Dud02}.

For weak Feller Markov chains, their averaged probability measures $\{ \bar P_k(x, \cdot)\}_{k \geq 1}$ are known to have the following property; see e.g., the proof of Lemma 4.1 in~\citep{Mey89}.
It will be needed in our proofs of Theorems~\ref{thm-c1}-\ref{thm-c1b}.

\begin{lem} \label{lma-c1}
Let $\{X_t\}$ be a weak Feller Markov chain with transition kernel $P(\cdot, \cdot)$ on a metric space $\X$. For each $x \in \X$, any weak limit of $\{ \bar P_k(x, \cdot)\}$ is an invariant probability measure of $\{X_t\}$.
\end{lem}

Recall that the occupation probability measures of $\{X_t\}$, denoted $\{\mu_{x,t}\}$ for each initial condition $x \in \X$, are defined as follows: 
$$\mu_{x,t}(D) : = \frac{1}{t} \sum_{k=0}^{t-1} \I ( X_k \in D), \qquad \forall \, D \in \mathcal{B}(\X),$$
where the chain $\{X_t\}$ starts from $X_0=x$, and each $\mu_{x,t}$ is a random variable taking values in the space of probability measures on $\X$.
Let ``$\Pr_x$-a.s.'' stand for ``almost surely with respect to $\Pr_x$.''
The next lemma concerns the convergence of occupation probability measures of a weak Feller Markov chain. It is a result of Meyn \cite{Mey89} and will be needed in our proofs of Theorems~\ref{thm-c2}-\ref{thm-c2b}.

\begin{lem}[{\cite[Prop.\ 4.2]{Mey89}}] \label{lma-c1b}
Let $\{X_t\}$ be a weak Feller Markov chain with transition kernel $P(\cdot, \cdot)$ on a complete separable metric space $\X$. Suppose that 
\begin{enumerate}
\item[\rm (i)] $\{X_t\}$ has a unique invariant probability measure $\mu$;
\item[\rm (ii)] for each compact set $E \subset \X$, the set $\{ \bar P_k(x, \cdot) \mid x \in E, \, k \geq 1 \}$ is tight; and 
\item[\rm (iii)] for all initial conditions $x \in \X$, there exists a sequence of compact sets $E_k \uparrow \X$ (that is $E_k \subset E_{k+1}$ for all $k$ and $\cup_k E_k = \X$) such that 
$$ \lim_{k \to \infty} \liminf_{t \to \infty} \mu_{x,t}(E_k) = 1, \qquad \text{$\Pr_x$-a.s.}$$
\end{enumerate}
Then, for each initial condition $x \in \X$, the sequence $\{\mu_{x,t}\}$ of occupation probability measures converges weakly to $\mu$, $\Pr_x$-almost surely.
\end{lem}

The condition (iii) above is equivalent to that the sequence $\{\mu_{x,t}\}$ of occupation probability measures is almost surely tight for each initial condition.

\subsubsection{Proofs of Theorems~\ref{thm-c1} and~\ref{thm-c1b}} \label{sec-4.3.2}

In this subsection we prove Theorem~\ref{thm-c1} for the algorithm (\ref{eq-emtd-const0}) and Theorem~\ref{thm-c1b} for its two variants (\ref{eq-emtd-const1}) and (\ref{eq-emtd-const2}). We also show that the conclusions of Theorems~\ref{thm-c1}-\ref{thm-c1b} hold for the perturbed version (\ref{eq-valg}) of these algorithms as well. The proof arguments are largely the same for all the algorithms we consider here. So except where noted otherwise, it will be taken for granted through out this subsection that \emph{$\{\theta_t^\alpha\}$ is generated by either of the six algorithms just mentioned}, for a constant stepsize $\alpha > 0$. 

We start with some preliminary analysis given in the next two lemmas. Recall $\Z_t = (S_t, A_t, \e_t, \F_t)$ and $\{\Z_t\}$ is a weak Feller Markov chain on $\Zs : = \S \times \A \times \re^{n+1}$ \cite[Sec.\ 3.1]{yu-etdarx}, and its evolution is not affected by the $\theta$-iterates. We consider the Markov chain $\{(\Z_t, \theta_t^\alpha)\}$ on the state space $\Zs \times \H$ (note that this is a complete separable metric space).
This chain also has the weak Feller property:
 
\begin{lem} \label{lma-c0}
Let Assumption~\ref{cond-bpolicy}(ii) hold. The process $\{(\Z_t, \theta_t^\alpha)\}$ is a weak Feller Markov chain.
\end{lem}

\begin{proof}
We omit the superscript $\alpha$ in the proof. 
We need to show that for any bounded continuous function $f(z,\theta)$, the function $\E [ f(\Z_1, \theta_1) \mid Z_0=z, \theta_0=\theta]$ is continuous in $(z, \theta)$. Express $z$ in terms of its components as $z=(s, a, \e, \F)$. Given that the space $\S \times \A$ is discrete, a function continuous in $(z,\theta)$ is a function that is continuous in $(\e,\F,\theta)$ for each $(s,a) \in \S \times \A$. So in view of the finiteness of $\S \times \A$, 
what we need to show is that for any two state-action pairs $(s,a), (s',a') \in \S \times \A$, 
\begin{equation} \label{eq-prf-c1}
 \bar g (\e, \F, \theta) : = \E \big[ \, f(\Z_1, \theta_1) \,\big| \, (\e_0, \F_0, \theta_0)=(\e, \F, \theta), \, (S_0,A_0,S_1, A_1) = (s, a, s',a') \big]
\end{equation}  
is a continuous function of $(\e,\F,\theta)$.

Consider first the case where $\theta_t$ is generated by either of the algorithms (\ref{eq-emtd-const0}), (\ref{eq-emtd-const1}) and (\ref{eq-emtd-const2}). By the definitions of these algorithms, given $(S_0,A_0,S_1,A_1)$ as in (\ref{eq-prf-c1}), $\e_1$ and $\F_1$ are continuous functions of $(\e_0,\F_0)$, and $\theta_1$ is a continuous function of $(\e_0,\F_0,\theta_0,R_0)$. Thus $f(\Z_1, \theta_1)$ is a continuous function of $(\e_0,\F_0, \theta_0,R_0)$; to simplify notation, denote this function by $g(x,r)$ with $x=(\e,\F,\theta)$ and $r$ corresponding to the reward variable $R_0$. Then, to show the continuity of the function in (\ref{eq-prf-c1}) is to show the continuity of the function
$$ \bar g(x) = \int g(x, r) \, q(dr \!\mid\! s, a, s').$$
To verify that $\bar g$ is continuous, consider an arbitrary point $\bar x = (\bar \e, \bar \F, \bar \theta) \in \re^{n+1} \times \H$. Given any $\epsilon > 0$, pick $\bar r> 0$ large enough so that $q([-\bar r, \bar r] \!\mid\! s, a, s') \geq 1- \epsilon/(4\|g\|_\infty)$ (where $\|g\|_\infty \leq \|f\|_\infty < \infty$, and $\|g\|_\infty$, $\|f\|_\infty$ are the infinity norm of the functions $g, f$, respectively). Since $g$ is continuous, it is uniformly continuous on any compact set. Therefore, there exists $\delta > 0$ small enough so that for any $r \in [-\bar r, \bar r]$ and $x \in \re^{n+1} \times \H$ with $ | x - \bar x | \leq \delta$,
$| g(x, r) - g(\bar x, r) | \leq \epsilon/2$. Consequently, for any $x \in \re^{n+1} \times \H$ with $ | x - \bar x | \leq \delta$, we have
\begin{align*}
 \big| \bar g(x) - \bar g(\bar x) \big| & \leq \int \big| g(x, r)  - g(\bar x, r) \big| \, q(dr \!\mid\! s, a, s') \\
 & =  \int_{\{ | r | > \bar r \}} \big| g(x, r)  - g(\bar x, r) \big| \, q(dr \!\mid\! s, a, s')   + \int_{[-\bar r, \bar r]} \big| g(x, r)  - g(\bar x, r) \big| \, q(dr \!\mid\! s, a, s')   \\
 & \leq 2 \|g\|_\infty \cdot  \epsilon/(4\|g\|_\infty) + \epsilon/2   = \epsilon.
\end{align*} 
This shows that $\bar g$ is a continuous function, and proves that $\{(\Z_t, \theta_t^\alpha)\}$ is a weak Feller chain. 

The proof for the perturbed version (\ref{eq-valg}) of the algorithms (\ref{eq-emtd-const0}), (\ref{eq-emtd-const1}) and (\ref{eq-emtd-const2}) follows the same arguments, except that $g$ is now a continuous function of $(x, r, \Delta)$ where $\Delta$ is the perturbation variable, and $\bar g$ is defined by the integration over both $r$ and $\Delta$. The proof details are almost identical to those given above and therefore omitted. We only note a minor difference in the last step of the proof: in addition to choosing a sufficiently large $\bar r$, we also choose a sufficiently large compact set $E_\Delta$ on the space of $\Delta$, and we work with the resulting compact set $E_{\bar x} \times [- \bar r, \bar r] \times E_\Delta$ for a closed neighborhood $E_{\bar x}$ of the point $\bar x$, and use the uniform continuity of the function $g$ on compact sets, as in the above proof.
\end{proof}

In order to study the behavior of multiple consecutive $\theta$-iterates, we consider for $m \geq 1$, the $m$-step version of $\{(\Z_t, \theta_t^\alpha)\}$, that is, the Markov chain $\{X_t\}$ on $(\Zs \times \H)^m$ where each state $X_t$ consists of $m$ consecutive states of the original chain $\{(\Z_t, \theta_t^\alpha)\}$: 
$$X_t=\big( (\Z_t, \theta_t^\alpha),  \,  \ldots, \, (\Z_{t+m-1}, \theta_{t+m-1}^\alpha) \big).$$ 
Similarly to the proof of Lemma~\ref{lma-c0}, it is straightforward to show that the $m$-step version of a weak Feller Markov chain is a weak Feller chain as well. Thus the $m$-step version of $\{(\Z_t, \theta_t^\alpha)\}$ is also a weak Feller Markov chain, and 
we can apply the ergodic theorems for weak Feller Markov chains to analyze it. In particular, in this subsection we will use Lemma~\ref{lma-c1} to prove Theorems~\ref{thm-c1}-\ref{thm-c1b}; in the next subsection we will also use Lemma~\ref{lma-c1b}. 

In analyzing the $m$-step version of $\{(\Z_t, \theta_t^\alpha)\}$, sometimes it will be more convenient for us to take as its initial condition the condition of just $(\Z_0, \theta_0^\alpha)$---instead of $(\Z_0, \theta_0^\alpha, \ldots, \Z_{m-1}^\alpha, \theta_{m-1}^\alpha)$---and to work with the following objects that are essentially equivalent to the averaged probability measures $\{\bar P_{k}(x, \cdot)\}$ and the occupation probability measures $\{\mu_{x,t}\}$ defined earlier for a general Markov chain $\{X_t\}$. Specifically, with $\{X_t\}$ denoting the $m$-step version of $\{(\Z_t, \theta_t^\alpha)\}$, for each $(z,\theta) \in \Zs \times \H$, we define probability measures $\bar P_{(z,\theta)}^{(m,k)}, k \geq 1$, on the space $\X = (\Zs \times \H)^m$, by
\begin{equation} \label{eq-avep}
 \bar P_{(z,\theta)}^{(m,k)}(D) : = \frac{1}{k}\sum_{t=0}^{k-1} \Pr_{(z,\theta)} \big( X_t \in D \big), \qquad  \forall \, D \in \mathcal{B}(\X).
\end{equation} 
Similarly, we define occupation probability measures $\{\mu_{(z,\theta),t}^{(m)}\}$ for each $(z,\theta) \in \Zs \times \H$ by 
\begin{equation}
  \mu_{(z,\theta),t}^{(m)} (D) : = \frac{1}{t} \sum_{k=0}^{t-1} \I \big( X_k \in D \big), \qquad \, \forall \, D \in \mathcal{B}(\X), \label{eq-om}
\end{equation}  
where the initial $(\Z_0, \theta_0^\alpha) = (z, \theta)$. Compared with the definitions of $\{\bar P_{k}(x, \cdot)\}$ and $\{\mu_{x,t}\}$ for $\{X_t\}$, apparently, all the previous conclusions given in Section~\ref{sec-4.3.1} for $\{\bar P_{k}(x, \cdot)\}$ and $\{\mu_{x,t}\}$ hold for $\big\{\bar P_{(z,\theta)}^{(m,k)}\big\}$ and $\big\{\mu_{(z,\theta),t}^{(m)}\big\}$ as well; therefore we can use the objects $\big\{\bar P_{(z,\theta)}^{(m,k)}\big\}$ and $\{\bar P_{k}(x, \cdot)\}$, and $\big\{\mu_{(z,\theta),t}^{(m)}\big\}$ and $\{\mu_{x,t}\}$, interchangeably in our analysis.

\begin{lem} \label{lma-c2}
Let Assumption~\ref{cond-bpolicy} hold. For $m \geq 1$, let $\{X_t\}$ be the $m$-step version of $\{(\Z_t, \theta_t^\alpha)\}$ on $\X = (\Zs \times \H)^m$, with transition kernel $P(\cdot, \cdot)$.
Then $\{X_t\}$ satisfies the conditions (ii)-(iii) of Lemma~\ref{lma-c1b}. 
\end{lem}

\begin{proof}
To show that the condition (ii) of Lemma~\ref{lma-c1b} is satisfied, fix a compact set $E \subset \X$ and let us first show that the set $\{P^t(x, \cdot) \mid x \in E, t \geq 0\}$ is tight.  
Since the set $\H$ is compact and the state and action spaces are finite, of concern here is just the tightness of the marginals of these probability measures on the space of the trace components $(\e_t, \F_t, \ldots, \e_{t+m-1}, \F_{t+m-1})$ of the state $X_t$. By Prop.~\ref{prp-bdtrace}, for all initial conditions of $(\e_0, \F_0)$ in a given bounded subset of $\re^{n+1}$, $\sup_{t \geq 0} \E [ \| (\e_t, \F_t) \| ] \leq \C$ for a constant $\C$ (that depends on the subset). So for the set $E$, applying the Markov inequality together with the union bound, we have that there exists a constant $\C > 0$ such that for all $x \in E$ and $a > 0$,
$\Pr_x \left( \sup_{k \leq t < k +m} \| (\e_t, \F_t)\| \geq a \right) \leq m \C/a$ for all $k \geq 0.$
Now for any given $\delta > 0$, let $a$ be large enough so that $m \C/a < \delta$ and let $D_a$ be the closed ball in $\re^{n+1}$ centered at the origin with radius $a$. 
Then for the compact set $D =(\S \times \A \times D_a \times \H)^m$, we have 
$P^k(x, D) = \Pr_x \big( \sup_{k \leq t < k +m} \| (\e_t, \F_t) \| \leq a \big) \geq 1 - \delta$ for all $x \in E$ and all $k \geq 0$. This shows that the set $\{ P^t(x, \cdot) \mid x \in E, t \geq 0\}$ is tight. Consequently, the averages of the probability measures in this set must also form a tight set; in particular, the set $\{ \bar P_k(x, \cdot) \mid x \in E, k \geq 1\}$ must be tight. Hence $\{X_t\}$ satisfies the condition (ii) of Lemma~\ref{lma-c1b}. 

Consider now the condition (iii) of Lemma~\ref{lma-c1b}. For positive integers $k$, let $E_k$ in that condition be the compact set $(\S \times \A \times D_k \times \H )^m$, where $D_k$ is the closed ball of radius $k$ in $\re^{n+1}$ centered at the origin. We wish to show that for each initial condition $x \in \X$,
$$ \lim_{k \to \infty} \liminf_{t \to \infty} \, \mu_{x,t}(E_k) = 1, \qquad \text{$\Pr_x$-a.s.}$$
Since the $\theta$-iterates do not affect the evolution of $\Z_t$, they can be neglected in the proof. It is sufficient to consider instead the $m$-step version of $\{Z_t\}$ and show that for the compact sets $\hat E_k = (\S \times \A \times D_k)^m$, it holds for any initial condition $z \in \Zs$ of $\Z_0$ that 
\begin{equation} \label{eq-prf-lc2a0}
  \lim_{k \to \infty} \liminf_{t \to \infty} \, \hat \mu_{z,t}^{(m)} \big( \hat E_k \big) = 1, \qquad \text{$\Pr_z$-a.s.},
\end{equation}  
where $\{\hat \mu_{z,t}^{(m)}\}$ are the occupation probability measures of the $m$-step version of $\{\Z_t\}$, defined analogously to (\ref{eq-om}) with $(\Z_t, \ldots, \Z_{t+m-1})$ in place of $X_t$.

To prove (\ref{eq-prf-lc2a0}), consider $\{\Z_t\}$ first and its occupation probability measures $\{\hat \mu_{z,t}\}$ for each initial condition $\Z_0=z \in \Zs$.
By Theorem~\ref{thm-2.1}, $\Pr_z$-almost surely, $\{\hat \mu_{z,t}\}$ converges weakly to $\zeta$ (the unique invariant probability measure of $\{\Z_t\}$). So by \cite[Theorem 11.1.1]{Dud02}, for the open set $\tilde D_k = \S \times \A \times D^o_k$, where $D^o_k$ denotes the interior of $D_k$ (i.e., $D^o_k$ is the open ball with radius $k$), almost surely,
\begin{equation} \label{eq-prf-lc2a}
  \liminf_{t \to \infty} \, \hat \mu_{z,t} \big( \tilde D_k \big) \geq  \zeta(\tilde D_k),  \qquad \text{and hence} \ \  \lim_{k \to \infty} \liminf_{t \to \infty} \, \hat \mu_{z,t} \big( \tilde D_k \big) = 1.
\end{equation}  
Now for the $m$-step version of $\{Z_t\}$, with $[\tilde D_k]^m$ denoting the Cartesian product of $m$ copies of $\tilde D_k$, we have 
\begin{align}
 \hat \mu_{z,t}^{(m)} \big( [\tilde D_k]^m \big) : = \frac{1}{t} \sum_{j=0}^{t-1} \I \big( \Z_{j+j'} \in \tilde D_k, \,  0 \leq j' < m \big) \geq 1 - \sum_{j'=0}^{m-1} \frac{1}{t} \sum_{j=0}^{t-1}  \I \big( \Z_{j+j'} \not \in \tilde D_k \big). \label{eq-prf-lc2b}
\end{align}
For each $j' < m$, by the definition of $\hat \mu_{z,t}$,
$\limsup_{t \to \infty } \frac{1}{t} \sum_{j=0}^{t-1}  \I \big( \Z_{j+j'} \not \in \tilde D_k \big) =  \limsup_{t \to \infty } \hat \mu_{z,t} \big( \tilde D_k^c \big)$,
where $\tilde D_k^c$ denotes the complement of $\tilde D_k$ in $\S \times \A \times \re^{n+1}$. By (\ref{eq-prf-lc2a}), 
$\lim_{k \to \infty}  \limsup_{t \to \infty } \hat \mu_{z,t} \big( \tilde D_k^c \big) = 0$ almost surely. Hence for each $j' < m$,
$\lim_{k \to \infty} \limsup_{t \to \infty } \frac{1}{t} \sum_{j=0}^{t-1}  \I \big( \Z_{j+j'} \not \in \tilde D_k \big) = 0$ almost surely. We then obtain
from (\ref{eq-prf-lc2b}), by taking the limits as $t \to \infty$ and $k \to \infty$, that 
$$ \liminf_{k \to \infty} \, \liminf_{t \to \infty } \, \hat \mu_{z,t}^{(m)} \big( [\tilde D_k]^m \big)  \geq 1 - \sum_{j'=0}^{m-1} \limsup_{k \to \infty} \, \limsup_{t \to \infty } \frac{1}{t} \sum_{j=0}^{t-1}  \I \big( \Z_{j+j'} \not \in \tilde D_k \big)  = 1$$
almost surely. The desired equality (\ref{eq-prf-lc2a0}) then follows, since $[\tilde D_k]^m \subset \hat E_k$. 
\end{proof}

Recall that $\mathcal{M}^m_\alpha$ is the set of invariant probability measures of the $m$-step version of $\{(\Z_t, \theta_t^\alpha)\}$. 
By Lemma~\ref{lma-c2} the latter Markov chain satisfies the condition (ii) of Lemma~\ref{lma-c1b}, and this implies that the set $\big\{\bar P_{(z,\theta)}^{(m,k)}\big\}_{k \geq 1}$ is tight for each initial condition $(\Z_0, \theta_0^\alpha) = (z, \theta)$. Recall that any subsequence of a tight sequence has a further convergent subsequence~\cite[Theorem 11.5.4]{Dud02}. For $\big\{\bar P_{(z,\theta)}^{(m,k)}\big\}_{k \geq 1}$, all the weak limits (i.e., the limits of its convergent subsequences) must be invariant probability measures in $\mathcal{M}^m_\alpha$, by the property of weak Feller Markov chains given in Lemma~\ref{lma-c1}:

\begin{prop} \label{prop-c1}
Under Assumption~\ref{cond-bpolicy}, consider the $m$-step version of $\{(\Z_t, \theta_t^\alpha)\}$ for $m \geq 1$. For each $(z, \theta) \in \Zs \times \H$, the sequence $\big\{\bar P_{(z,\theta)}^{(m,k)}\big\}_{k \geq 1}$ of probability measures is tight, and any weak limit of this sequence is an invariant probability measure of the $m$-step version of $\{(\Z_t, \theta_t^\alpha)\}$. (Thus $\mathcal{M}^m_\alpha \not=\emptyset$.)
\end{prop}

\smallskip  
We are now ready to prove Theorems~\ref{thm-c1}-\ref{thm-c1b}. The idea is to use the conclusions on the $\theta$-iterates that we can obtain by applying \cite[Theorem 8.2.2]{KuY03}, to infer the concentration of the mass around a small neighborhood of $(\theta^*, \ldots, \theta^*)$ ($m$ copies of $\theta^*$) for the marginals of all the invariant probability measures in the set $\mathcal{M}^m_\alpha$, when $\alpha$ is sufficiently small. This can then be combined with Prop.~\ref{prop-c1} to prove the desired conclusions on the $\theta$-iterates for a given stepsize.

Recall that $\mathcal{M}_\alpha$ is the set of invariant probability measures of $\{(\Z_t, \theta_t^\alpha)\}$. Recall also that $\bar{\mathcal{M}}^m_\alpha$ 
denotes the set of marginals of the invariant probability measures in $\mathcal{M}^m_\alpha$, on the space of the $\theta$'s.

\begin{prop} \label{prop-c2}
In the setting of Theorem~\ref{thm-const-stepsize}, for each $\alpha > 0$, let $\{\theta_t^\alpha\}$ be generated instead by the algorithm (\ref{eq-emtd-const0}) or its perturbed version (\ref{eq-valg}), with constant stepsize $\alpha$ and under the condition that the initial $(\Z_0, \theta^\alpha_0)$ is distributed according to some invariant probability measure in $\mathcal{M}_\alpha$. Then the conclusions of Theorem~\ref{thm-const-stepsize} continue to hold.
\end{prop}

\begin{proof}
The proof arguments are the same as those for Theorem~\ref{thm-const-stepsize} given in Section~\ref{sec-4.1}. We only need to show that the conditions (ii) and (i$'$)-(v$'$) given in Section~\ref{sec-4.1.1} for applying \cite[Theorem 8.2.2]{KuY03} are still satisfied under our present assumptions. 

For the algorithm (\ref{eq-emtd-const0}), the only difference from the previous assumptions in Theorem~\ref{thm-const-stepsize} is that here for each stepsize $\alpha$, the initial $(\Z_0, \theta_0^\alpha)$ has a distribution $\mu_\alpha \in \mathcal{M}_\alpha$. The condition (ii) does not depend on such initial conditions, so it continues to hold. For the other conditions, note that since $\{\Z_t\}$ has a unique invariant probability measure $\zeta$ (Theorem~\ref{thm-2.1}), regardless of the choice of $\mu_\alpha$, for all $\alpha$, $\{\Z_t\}$ is stationary and has the same distribution. Then the tightness condition (iii$'$) trivially holds because as $\{\xi_t\}$ is also stationary and unaffected by the stepsize, each $\xi_t^\alpha$ in (iii$'$) has the same distribution. Similarly, since $\{\e_t\}$ is stationary and unaffected by the stepsize, and each $\e_t$ has the same distribution with the mean of $\| \e_t\|$ given by $\E_\zeta [ \| \e_t\| ] < \infty$ (Theorem~\ref{thm-2.2}), we obtain that $\{\e_t\}$ is u.i. From this the uniform integrability required in the conditions (i$'$) and (iv$'$) follows as a consequence, as shown in the proof of Prop.~\ref{prop-2}(ii)-(iv). Lastly, the convergence in mean condition (v$'$) continues to hold (by the same proof given for Prop.\ \ref{prop-3}). 
This is because $\{\xi_t\}$ has the same distribution regardless of the stepsize, and because the condition (v$'$) is for each compact set $D$ and concerns tails of a trajectory starting at instants $t$ with $\xi_{t} \in D$, which renders any initial condition on $\Z_0$ ineffective.
Thus all the required conditions are met, and we obtain the same conclusions on the $\theta$-iterates as given in Theorem~\ref{thm-const-stepsize}.

For the perturbed version (\ref{eq-valg}) of the algorithm (\ref{eq-emtd-const0}), the only difference to (\ref{eq-emtd-const0}) under the present assumptions is the perturbation variables $\Delta^\alpha_{\theta,t}$ involved in each iteration. But by definition these variables have conditional zero mean: $\E^\alpha_t [ \Delta^\alpha_{\theta,t}] = 0$, so the only condition in which they appear is the uniform integrability condition (i$'$): $\{Y_t^\alpha \mid t \geq 0, \alpha > 0\}$ is u.i., where $Y_t^\alpha$ is now given by $Y_t^\alpha = h(\theta_t^\alpha, \xi_t) +  \e_t \cdot \tilde \omega_{t+1} + \Delta^\alpha_{\theta,t}$. By definition $\Delta^\alpha_{\theta,t}$ for all $\alpha$ and $t$ have bounded variance, and hence $\{\Delta^\alpha_{\theta,t}\}$ is u.i.~\cite[p.\ 32]{Bil68}. The set $\{h(\theta_t^\alpha, \xi_t) +  \e_t \cdot \tilde \omega_{t+1} \mid t \geq 0, \alpha > 0\}$ is u.i., which follows from the u.i.\ of $\{\e_t\}$, as we just verified in the case of the algorithm (\ref{eq-emtd-const0}). Therefore, by Lemma~\ref{lem-w1}(i), $\{Y_t^\alpha \mid t \geq 0, \alpha > 0\}$ is u.i.\ and the condition (i$'$) is satisfied. Since the perturbed version (\ref{eq-valg}) meets all the required conditions, and shares with (\ref{eq-emtd-const0}) the same mean ODE, the same conclusions given in Theorem~\ref{thm-const-stepsize} hold for this algorithm as well.
\end{proof}

We now prove Theorem~\ref{thm-c1} for the algorithm (\ref{eq-emtd-const0}). We prove its part (i) and part (ii) separately, as the arguments are different. Our proofs below also apply to the perturbed version (\ref{eq-valg}) of the algorithm (\ref{eq-emtd-const0}), and together with the preceding proposition, they establish the first part of Theorem~\ref{thm-c2} (which says that the conclusions of both Theorem~\ref{thm-const-stepsize} and Theorem~\ref{thm-c1} hold for the perturbed algorithm).

\begin{proof}[Proof of Theorem~\ref{thm-c1}(i)]
Proof by contradiction. Consider the statement of Theorem~\ref{thm-c1}(i): 
$$\forall \, \delta > 0, \quad \liminf_{\alpha \to 0} \inf_{\mu \in \bar{\mathcal{M}}^{m_\alpha}_{\alpha}} \mu\big( [N_\delta(\theta^*)]^{m_\alpha} \big) = 1, \qquad \text{where} \  m_\alpha = \lceil \tfrac{m}{\alpha} \rceil.$$  
Suppose it is not true. Then there exist $\delta, \epsilon > 0$, $m \geq 1$, a sequence $\alpha_k \to 0$, and a sequence $\mu_{\alpha_k} \in \bar{\mathcal{M}}^{m_k}_{\alpha_k}$, where $m_k = m_{\alpha_k}$, such that
\begin{equation} \label{eq-prf-c11}
  \mu_{\alpha_k}([N_\delta(\theta^*)]^{m_k}) \leq 1 - \epsilon, \qquad \forall \, k \geq 0.
\end{equation}    
Each $\mu_{\alpha_k}$ corresponds to an invariant probability measure of $\{(\Z_t, \theta_t^{\alpha_k})\}$ in $\mathcal{M}_{\alpha_k}$, which we denote by $\hat \mu_{\alpha_k}$.
For each $k \geq 0$, generate the iterates $\{\theta_t^{\alpha_k}\}$ using $\hat \mu_{\alpha_k}$ as the initial distribution of $(\Z_0, \theta_0^{\alpha_k})$. For other values of $\alpha$, generate the iterates $\{\theta_t^{\alpha}\}$ using some $\hat \mu_\alpha \in \mathcal{M}_\alpha$ as the initial distribution of $(\Z_0, \theta_0^{\alpha})$. 
By Prop.~\ref{prop-c2}, the conclusions of Theorem~\ref{thm-const-stepsize} hold:
$$  \limsup_{\alpha \to 0} \, \Pr \Big( \, \theta_t^\alpha \not\in N_\delta(\theta^*), \, \text{some} \  t \in \big[ \, k_\alpha , \,  k_\alpha + T_\alpha/\alpha \, \big] \Big) = 0,$$
where $T_\alpha \to \infty$ as $\alpha \to 0$, and this implies for the given $m$,
\begin{equation} \label{eq-prf-c12}
  \limsup_{\alpha \to 0} \, \Pr \Big( \, \theta_t^\alpha \not\in N_\delta(\theta^*), \, \text{some} \  t \in \big[ \, k_\alpha , \,  k_\alpha + \lceil \tfrac{m}{\alpha} \rceil \, \big) \Big) = 0.
\end{equation}  
But for each $\alpha > 0$, the process $\{(\Z_t, \theta_t^{\alpha})\}$ with the initial distribution $\hat \mu_\alpha$ is stationary, so the probability in the left-hand side of (\ref{eq-prf-c12}) is just
$ 1 - \mu_\alpha([N_\delta(\theta^*)]^{m_\alpha})$, for the marginal probability measure $\mu_\alpha \in \bar{\mathcal{M}}^{m_\alpha}_{\alpha}$ that corresponds to the invariant probability measure $\hat \mu_\alpha$. Therefore, by (\ref{eq-prf-c12}), $\liminf_{\alpha \to 0} \mu_\alpha([N_\delta(\theta^*)]^{m_\alpha}) = 1$. On the other hand, by (\ref{eq-prf-c11}), $\liminf_{\alpha \to 0} \mu_\alpha([N_\delta(\theta^*)]^{m_\alpha}) \leq \liminf_{k \to \infty} \mu_{\alpha_k}([N_\delta(\theta^*)]^{m_k}) < 1$, a contradiction. Thus the statement of Theorem~\ref{thm-c1}(i) recounted at the beginning of this proof must hold.

This also proves the other statement of Theorem~\ref{thm-c1}(i), $\liminf_{\alpha \to 0} \inf_{\mu \in \bar{\mathcal{M}}^m_{\alpha}} \mu\big( [N_\delta(\theta^*)]^m \big)  = 1$,  
because for $\alpha < 1$, by the correspondences between those invariant probability measures in $\mathcal{M}^m_{\alpha}$ and those in $\mathcal{M}^{m_\alpha}_{\alpha}$,
$\inf_{\mu \in \bar{\mathcal{M}}^m_{\alpha}} \mu\big( [N_\delta(\theta^*)]^m \big) \geq \inf_{\mu \in \bar{\mathcal{M}}^{m_\alpha}_{\alpha}} \mu\big( [N_\delta(\theta^*)]^{m_\alpha} \big)$. This completes the proof.
\end{proof}

\begin{proof}[Proof of Theorem~\ref{thm-c1}(ii)]
We suppress the superscript $\alpha$ of $\theta_t^\alpha$ in the proof. The statement is trivially true if $\delta \geq 2 r_\H$, so consider the case $\delta < 2 r_\H$. 
Let $(z,\theta) \in \Z \times \H$ be the initial condition of $(\Z_0, \theta_0)$. By convexity of the Euclidean norm,
$\big| \bar \theta_t - \theta^* \big|  \leq \tfrac{1}{t} \sum_{j=0}^{t-1}  | \theta_j - \theta^* |,$
and therefore, for all $k \geq 1$,
\begin{equation} \label{eq-prf-c20}
  \sup_{k \leq t < k+m} \big| \bar \theta_t - \theta^* \big|  \leq  \frac{1}{k} \sum_{j=0}^{k-1} \, \sup_{j \leq t < j+m} | \theta_t - \theta^* |, 
\end{equation} 
and
\begin{equation} \label{eq-prf-c21}
 \E \left[ \sup_{k \leq t < k+m} \big| \bar \theta_t - \theta^* \big|  \right] \leq \frac{1}{k} \sum_{j=0}^{k-1} \, \E \left[ \sup_{j \leq t < j+m} | \theta_t - \theta^* | \right]. 
\end{equation} 
With $N'_\delta(\theta^*)$ denoting the open $\delta$-neighborhood of $\theta^*$, we have
\begin{align}
\frac{1}{k} \sum_{j=0}^{k-1} \E \left[ \sup_{j \leq t < j+m} | \theta_t - \theta^* | \right] 
& \leq \frac{1}{k} \sum_{j=0}^{k-1} \E \left[ \Big( \sup_{j \leq t < j+m} | \theta_t - \theta^* | \Big) \cdot  \I \big( \theta_t \in N'_\delta(\theta^*), \, j \leq t < j+m \big)  \right]  \notag \\
& \quad + \frac{1}{k} \sum_{j=0}^{k-1} \E \left[ \Big( \sup_{j \leq t < j+m} | \theta_t - \theta^* | \Big) \cdot  \I \big( \theta_t \not\in N'_\delta(\theta^*), \, \text{some} \ t \in  [j, j+m) \big) \right] \notag \\
& \leq \delta \cdot \bar P_{(z,\theta)}^{(m,k)}(D_{\delta})  + 2 r_\H \cdot \big( 1 - \bar P_{(z,\theta)}^{(m,k)}(D_{\delta}) \big), \label{eq-prf-c22}
\end{align}
where $D_{\delta} = \big\{ \big(z^1, \theta^1, \ldots, z^m, \theta^m \big) \in (\Zs \times \H)^m \ \big| \,  \sup_{1 \leq j \leq m} \big|\theta^j - \theta^* \big| < \delta \big\}$, and the second inequality follows from the definition (\ref{eq-avep}) of the averaged probability measure $\bar P_{(z,\theta)}^{(m,k)}$.

By Prop.~\ref{prop-c1}, $\big\{\bar P_{(z,\theta)}^{(m,k)}\big\}_{k \geq 1}$ is tight and all its weak limits are in $\mathcal{M}_\alpha^m$, the set of invariant probability measure of the $m$-step version of $\{(\Z_t, \theta_t)\}$. There is also the fact that on a metric space, if a sequence of probability measures $\mu_k$ converges to some probability measure $\mu$ weakly, then $\liminf_{k \to \infty} \mu_k(D) \geq \mu(D)$ for any open set $D$ \cite[Theorem 11.1.1]{Dud02}. From these two arguments we have that for the set $D_\delta$, which is open with respect to the topology on $(\Zs \times \H)^m$,
\begin{equation} \label{eq-prf-c23}
 \liminf_{k \to \infty} \bar P_{(z,\theta)}^{(m,k)}(D_{\delta}) \geq \inf_{\mu \in \mathcal{M}_\alpha^m} \mu(D_\delta) = \inf_{\mu \in \bar{\mathcal{M}}_\alpha^m} \mu \big([N'_\delta(\theta^*)]^m \big) = : \kappa_{\alpha,m}.
\end{equation} 
Combining the three inequalities (\ref{eq-prf-c21})-(\ref{eq-prf-c23}), and using also the relation $\delta < 2 r_\H$, we obtain
$$ \limsup_{k \to \infty} \E \left[ \sup_{k \leq t < k+m} \big| \bar \theta_t - \theta^* \big|  \right] \leq \delta \, \kappa_{\alpha,m} + 2 r_\H \big( 1- \kappa_{\alpha,m} \big). \qedhere$$
\end{proof}

%\medskip
We prove Theorem~\ref{thm-c1b} in exactly the same way as we proved Theorem~\ref{thm-c1}, so we omit the details and only outline the proof here.
First, for the variant algorithms~(\ref{eq-emtd-const1}) and (\ref{eq-emtd-const2}) as well as their perturbed version (\ref{eq-valg}), we consider fixed $K$ and $\psi_K$. 
Similar to Prop.~\ref{prop-c2}, we show that if for each stepsize $\alpha$, the initial $(\Z_0, \theta_0^\alpha)$ is distributed according to some invariant probability measure in $\mathcal{M}_\alpha$, then the algorithms continue to satisfy the conditions given in Section~\ref{sec-4.1.1}, so we can apply~\cite[Theorem 8.2.2]{KuY03} to assert that the conclusions of Theorem~\ref{thm-const-stepsize} continue to hold with $N_\delta(\theta^*)$ replaced by the limit set $N_\delta(L_\H)$ of the mean ODE associated with each algorithm. (Recall Theorem~\ref{thm-const-stepsize-b} is also obtained in this way.)
Subsequently, with $N_\delta(L_\H)$ in place of $N_\delta(\theta^*)$ again, and with $K$ and $\psi_K$ still held fixed, we use the same proof for Theorem~\ref{thm-c1}(i) to obtain that for any $\delta > 0$ and $m \geq 1$,
\begin{equation}
  \liminf_{\alpha \to 0}   \inf_{\mu \in \bar{\mathcal{M}}^{m_\alpha}_{\alpha}} \mu\big( [N_\delta(L_\H)]^{m_\alpha} \big) = 1, \qquad \text{where} \ m_\alpha = \lceil \tfrac{m}{\alpha} \rceil. \notag
\end{equation} 
Finally, we combine this with the fact that given any $\delta > 0$, the limit set $N_\delta(L_\H) \subset N_\delta(\theta^*)$ for all $K$ sufficiently large (see Lemma~\ref{lma-3.2}, which holds for (\ref{eq-emtd-const1}) and (\ref{eq-emtd-const2}), as well as their perturbed version (\ref{eq-valg}) since the latter has the same mean ODE as the original algorithm). Theorem~\ref{thm-c1b}(i) then follows: given $\delta > 0$, for all $K$ sufficiently large,
$$   \liminf_{\alpha \to 0}  \inf_{\mu \in \bar{\mathcal{M}}^{m_\alpha} _{\alpha}} \mu\big( [N_\delta(\theta^*)]^{m_\alpha} \big) = 1.$$
The proof for Theorem~\ref{thm-c1b}(ii) is exactly the same as that for Theorem~\ref{thm-c1}(ii) given earlier. In particular, this proof relies solely on the weak Feller property of the Markov chain $\{(\Z_t, \theta_t^\alpha)\}$ and the convergence property of the averaged probability measures of the $m$-step version of $\{(\Z_t, \theta_t^\alpha)\}$, all of which were proved for the algorithms~(\ref{eq-emtd-const1}) and (\ref{eq-emtd-const2}) and their perturbed version (\ref{eq-valg}) in this subsection.

The preceding arguments also show that the first part of Theorem~\ref{thm-c2b} holds; that is, the conclusions of Theorem~\ref{thm-const-stepsize-b} and Theorem~\ref{thm-c1b} hold for the perturbed version (\ref{eq-valg}) of the algorithm (\ref{eq-emtd-const1}) or (\ref{eq-emtd-const2}) as well.

\subsubsection{Proofs of Theorems~\ref{thm-c2} and~\ref{thm-c2b}} \label{sec-4.3.3}

In this subsection we establish completely Theorems~\ref{thm-c2} and~\ref{thm-c2b} regarding the perturbed version (\ref{eq-valg}) of the algorithms (\ref{eq-emtd-const0}), (\ref{eq-emtd-const1}) and (\ref{eq-emtd-const2}). We have already proved the first part of both of these theorems in the previous subsection. Below we tackle their second part, which, as we recall, is stronger than the corresponding part of Theorems~\ref{thm-c1} and~\ref{thm-c1b} in that for a fixed stepsize $\alpha$, the deviation of the averaged iterates $\{\bar \theta_t^\alpha\}$ from $\theta^*$ in the limit as $t \to \infty$ is now characterized not in an expected sense but for almost all sample paths.

To simplify the presentation, except where noted otherwise, it will be taken for granted throughout this subsection that \emph{$\{\theta_t^\alpha\}$ is generated by the perturbed version (\ref{eq-valg}) of any of the three algorithms (\ref{eq-emtd-const0}), (\ref{eq-emtd-const1}) and (\ref{eq-emtd-const2})}. 
Recall that when updating $\theta_{t}^\alpha$ to $\theta_{t+1}^\alpha$, the perturbed algorithm (\ref{eq-valg}) adds the perturbation term $\alpha \Delta_{\theta,t}$ to the iterate before the projection $\Pi_\H$, where $\Delta_{\theta,t}, t \geq 0$, are assumed to be i.i.d.\ $\rn$-valued random variables that have zero mean and bounded variances and have a positive continuous density function with respect to the Lebesgue measure. (Here and in what follows, we omit the superscript $\alpha$ of the noise terms $\Delta_{\theta,t}$ since we deal with a fixed stepsize $\alpha$ in this part of the analysis.) 
As mentioned in Section~\ref{sec-3.3}, these conditions are not as weak as possible. Indeed, the purpose of the perturbation is to make the invariant probability measure of $\{(\Z_t, \theta_t^\alpha)\}$ unique so that we can invoke the ergodic theorem for weak Feller Markov chains given in Lemma~\ref{lma-c1b}. Therefore, any conditions that can guarantee the uniqueness of the invariant probability measure can be used. In the present paper, for simplicity, we focus on the conditions we assumed earlier on $\Delta_{\theta,t}$, and prove the uniqueness just mentioned under these conditions, although our proof arguments can be useful for weaker conditions as well.

\begin{prop} \label{prop-c3}
Under Assumption~\ref{cond-bpolicy}, $\{(\Z_t, \theta_t^\alpha)\}$ has a unique invariant probability measure.
\end{prop}

The next two lemmas are the intermediate steps to prove Prop.\ \ref{prop-c3}. We need the notion of a stochastic kernel, of which the transition kernel of a Markov chain is one example. For two topological spaces $\X$ and $\Y$, a function $Q: \B(\X) \times \Y \to [0,1]$ is a (Borel measurable) \emph{stochastic kernel on $\X$ given $\Y$}, if for each $y \in \Y$, $Q(\cdot \mid y)$ is a probability measure on $\B(\X)$ and for each $D \in \B(\X)$, $Q(D \mid y)$ is a Borel measurable function on $\Y$. 
For the algorithms we consider, the iteration that generates $(\Z_{t+1}, \theta_{t+1}^\alpha)$ from $(\Z_t, \theta_{t}^\alpha)$ can be equivalently described in terms of stochastic kernels. In particular, the transition from $\Z_t$ to $\Z_{t+1}$ is described by the transition kernel of the Markov chain $\{\Z_t\}$, and the probability distribution of $\theta_{t+1}^\alpha$ given $\theta_t^\alpha$ and $\xi_t=(\e_t,\F_t,S_t,A_t,S_{t+1})$ is described by another stochastic kernel, which will be our focus in the analysis below.

\begin{lem} \label{lma-c4a}
Let Assumption~\ref{cond-bpolicy}(ii) hold. Let $Q(d\theta' \mid \xi, \theta)$ be the stochastic kernel (on $\H$ given $\Xi \times \H$) that describes the probability distribution of $\theta_{t+1}^\alpha$ given $\xi_t = \xi, \theta_t^\alpha = \theta$. 
Then for each bounded set $E \subset \Xi$, there exist $\beta \in (0,1]$ and a probability measure $Q_1$ on $\H$ such that 
\begin{equation} \label{eq-thetatrans0}
  Q(d\theta' \mid \xi , \theta) \geq  \beta \, Q_1(d\theta'), \qquad \forall \, \xi \in E, \ \theta \in \H.
\end{equation} 
\end{lem}

\begin{proof}
We consider only the case where $\{\theta_t^\alpha\}$ is generated by the perturbed version of the algorithm (\ref{eq-emtd-const0}); the proof for the perturbed version of the two other algorithms (\ref{eq-emtd-const1}) and (\ref{eq-emtd-const2}) follows exactly the same arguments. In the proof below we use the notation that for a scalar $c$ and a set $D \subset \re^{n}$, the set $c D =\{ c x \mid x \in D\}$. 

By the definitions of the algorithms (\ref{eq-emtd-const0}) and (\ref{eq-valg}), for $\xi = (\e, \F, s, a, s') \in \Xi$ and $\theta \in \H$, we can express $Q(\cdot \mid \xi, \theta)$ as
\begin{equation} \label{eq-prf-lc4a}
 Q(D \mid \xi, \theta)  = \int  \int \I \Big( \Pi_\H \big(\theta + \alpha f(\xi, \theta, r) + \alpha \Delta \big) \in D \Big)  \, p(d\Delta) \, q(dr \mid s, a, s'), \quad \forall \, D \in \mathcal{B}(\H),
\end{equation} 
where $f(\xi, \theta, r) = \e \cdot \rho(s,a) \big( r + \gamma(s') \phi(s')^\top \theta - \phi(s)^\top \theta \big)$, and $p(\cdot)$ is the common distribution of the perturbation variables $\Delta_{\theta,t}$. Let $\bar r > 0$ be large enough so that for some $c > 0$, $q([-\bar r, \bar r] \mid \bar s, \bar a, \bar s') \geq c$ for all $(\bar s, \bar a, \bar s') \in \S \times \A \times \S$. Let $E$ be an arbitrary bounded subset of $\Xi$. For all $\xi \in E, \theta \in \H$ and $r \in [-\bar r, \bar r]$, since $E$ and $\H$ are bounded, $g(\xi,\theta, r): = (\theta + \alpha f(\xi, \theta, r))/\alpha$ lies in a compact subset of $\re^n$, which we denote by $\bar D$. Let $\epsilon \in (0, r_\H/\alpha]$ and let $\bar D_\epsilon$ be the $\epsilon$-neighborhood of $\bar D$.  By our assumption on the perturbation variables involved in the algorithm (\ref{eq-valg}), $p(\cdot)$ has a positive continuous density function with respect to the Lebesgue measure $\ell(\cdot)$. Therefore, there exists some $c' > 0$ such that for any Borel subset $D$ of the compact set $-\bar D_\epsilon : = \{ - x \mid x \in \bar D_\epsilon \}$, $p(D) \geq c' \ell(D)$.

Now consider an arbitrary $\xi \in E$, $\theta \in \H$, and $r \in [-\bar r, \bar r]$. We have $y: = g(\xi,\theta, r) \in \bar D$. Let $B_\epsilon(-y)$ be the $\epsilon$-neighborhood of $-y$, and let $B_\epsilon$ denote the closed ball in $\rn$ centered at the origin with radius $\epsilon$.  
If $\Delta \in B_\epsilon(-y)$, then 
$\theta + \alpha f(\xi, \theta, r) + \alpha \Delta  = \alpha y + \alpha \Delta \in \alpha B_\epsilon \subset \H$ (since $\alpha \epsilon \leq r_\H$).
Therefore, for any $D \in \B(\H)$,
\begin{align}
  \int \I \Big( \Pi_\H \big(\theta + \alpha f(\xi, \theta, r) + \alpha \Delta \big) \in D \Big) \, p(d\Delta) & \geq \int_{B_\epsilon(-y)} \I \big(\alpha y + \alpha \Delta  \in D \big) \, p(d\Delta) \notag \\
  & \geq c' \int_{B_\epsilon(-y)} \I \big(\alpha y + \alpha \Delta  \in D \big) \, \ell(d\Delta) \notag \\
  & = c' \ell \big( \tfrac{1}{\alpha} D \cap B_\epsilon \big), \label{eq-prf-lc4b}
\end{align}  
where in the second inequality we used the fact that $B_\epsilon(-y) \subset - \bar D_\epsilon$ and restricted to $\B(-\bar D_\epsilon)$, $p(d\Delta) \geq c' \ell(d\Delta)$, as discussed earlier. 

To finish the proof, define the probability measure $Q_1$ on $\H$ by $Q_1(D) = \ell (\tfrac{1}{\alpha} D \cap B_\epsilon)/\ell(B_\epsilon)$ for all $D \in \mathcal{B}(\H)$.
Then for all $\xi \in E$ and $\theta \in \H$, using (\ref{eq-prf-lc4a}) and (\ref{eq-prf-lc4b}) and our choice of $\bar r$, we have
$$ Q(D \mid \xi, \theta)  \geq \int_{[-\bar r, \bar r]} c' \ell(B_\epsilon) \cdot  Q_1(D) \, q(dr \mid s, a, s') \geq c \cdot c' \ell(B_\epsilon) \cdot Q_1(D), \qquad D \in \B(\H),$$
and the desired inequality (\ref{eq-thetatrans0}) then follows by letting $\beta =c c' \ell(B_\epsilon) > 0$ (we must have $\beta \leq 1$ since we can choose $D = \H$ in the inequality above).
\end{proof}

We will use the preceding result in the proof of the next lemma.

\begin{lem} \label{lma-c4}
Let Assumption~\ref{cond-bpolicy} hold. Let $\{\mu_{x,t}\}$ be the sequence of occupation probability measures of $\{(\Z_t, \theta_t^\alpha)\}$ for each initial condition $x \in \Zs \times \H$.
Suppose that for some $x =(z, \theta) \in \Zs \times \H$ and $\mu \in \mathcal{M}_\alpha$, $\{\mu_{x,t}\}$ converges weakly to $\mu$, $\Pr_x$-almost surely. Then for each $\theta' \in \H$ and $x'=(z, \theta')$, $\{\mu_{x',t}\}$ also converges weakly to $\mu$, $\Pr_{x'}$-almost surely.  
\end{lem}
\begin{proof}
We use a coupling argument to prove the statement. In the proof, we suppress the superscript $\alpha$ of $\theta_t^\alpha$. 
Let $\{X_t\}$ denote the process $\{(\Z_t, \theta_t)\}$ with initial condition $x=(z, \theta)$, and let $\{X'_t\}$ denote the process $\{(\Z_t, \theta_t)\}$ with initial condition $x'=(z, \theta')$, for an arbitrary $\theta' \in \H$.
In what follows, we first define a sequence 
$$ \{(\Z_t, \tilde \theta_t, \tilde \theta_t')\} \quad \text{with} \ \ (\Z_0, \tilde \theta_0, \tilde \theta'_0) = (z, \theta, \theta'),$$ 
in such a way that the two marginal processes $\{(\Z_t, \tilde \theta_t)\}$ and $\{(\Z_t, \tilde \theta_t')\}$ have the same probability distributions as $\{X_t\}$ and $\{X'_t\}$, respectively. We then relate the occupation probability measures $\{\mu_{x,t}\}$, $\{\mu_{x',t}\}$ to those of the marginal processes, $\{\tilde \mu_{x,t}\}$, $\{\tilde \mu_{x',t}\}$, which are defined as
$$ \tilde \mu_{x,t}(D) = \frac{1}{t} \sum_{k=0}^{t-1} \I \big( (\Z_k, \tilde \theta_k) \in D \big), \qquad \tilde \mu_{x',t}(D) = \frac{1}{t} \sum_{k=0}^{t-1} \I \big( (\Z_k, \tilde \theta'_k) \in D \big), \qquad \forall \, D \in \B(\Zs \times \H). $$

We now define $\{(\Z_t, \tilde \theta_t, \tilde \theta_t')\}$. First, let $\{\Z_t\}$ be generated as before with $\Z_0=z$. 
Denote $\xi_t = (\e_t, \F_t, S_t, A_t, S_{t+1})$ as before, and let $Q$ be the stochastic kernel that describes the evolution of $\theta_{t+1}$ given $(\xi_t, \theta_t)$.
By Lemma~\ref{lma-c2}, the occupation probability measures of $\{\Z_t\}$ is almost surely tight for each initial condition. 
This implies the existence of a compact set $\bar E \subset \re^{n+1}$ such that for the compact set $E = \bar E \times \S \times \A \times \S \subset \Xi$, the sequence $\{\xi_t\}$ visits $E$ infinitely often with probability one. For this set $E$, by Lemma~\ref{lma-c4a}, there exist some $\beta \in (0,1]$ and probability measure $Q_1$ on $\H$ such that $Q(\cdot \mid \bar \xi, \bar \theta) \geq \beta Q_1(\cdot)$ for all $\bar \xi \in E$ and $\bar \theta \in \H$. Therefore, on $E \times \H$, we can write $Q(\cdot \mid \bar \xi, \bar \theta)$ as the convex combination of $Q_1$ and another stochastic kernel $Q_0$ as follows:
\begin{equation} \label{eq-thetatrans}
   Q(\cdot \mid \bar \xi, \bar \theta) = \beta \, Q_1(\cdot) + ( 1 - \beta) \, Q_0(\cdot \mid \bar \xi, \bar \theta),  \qquad \forall \, \bar \xi \in E, \ \bar \theta \in \H,
\end{equation}   
where $Q_0(\cdot \mid \bar \xi, \bar \theta) = \big[ Q(\cdot \mid \bar \xi, \bar \theta) - \beta \, Q_1(\cdot) \big]/(1 - \beta)$ and $Q_0$ is a stochastic kernel on $\H$ given $E \times \H$.

Next, independently of $\{Z_t\}$, generate a sequence $\{Y_t\}_{t \geq 1}$ of i.i.d., $\{0,1\}$-valued random variables such that $Y_t = 1$ with probability $\beta$ and $Y_t=0$ with probability $1 - \beta$. Set $Y_0=0$. Let 
$$t_Y = \min \{ t \geq 1 \mid Y_t=1, \xi_{t-1} \in E \}.$$ 
Then $t_Y < \infty$ with probability one. (Since $\{\xi_t\}$ visits $E$ infinitely often and the process $\{Y_t\}$ is independent of $\{\xi_t\}$, this follows easily from applying the Borel-Cantelli lemma to $\{(\xi_{t_k}, Y_{t_k+1})\}_{k \geq 1}$, where $t_k$ is when the $k$-th visit to $E$ by $\{\xi_t\}$ occurs.)

Now for each $t \geq 0$, let us define the pair $(\tilde \theta_{t+1}, \tilde \theta_{t+1}')$ according to the following rule, based on the values of $(\xi_0, \tilde \theta_0, \tilde \theta_0'), \ldots, (\xi_t, \tilde \theta_t, \tilde \theta_t')$ and $(Y_0, \ldots, Y_t, Y_{t+1})$:
\begin{enumerate}
\item[(i)] In the case $t < t_Y$ and $\xi_t \not\in E$, generate $\tilde \theta_{t+1}$ and $\tilde \theta'_{t+1}$ according to $Q(\cdot \mid \xi_t, \tilde \theta_t)$ and $Q(\cdot \mid \xi_t, \tilde \theta'_t)$ respectively. 
\item[(ii)] In the case $t < t_Y$ and $\xi_t \in E$, if $Y_{t+1} = 0$, generate $\tilde \theta_{t+1}$ and $\tilde \theta'_{t+1}$ according to $Q_0(\cdot \mid \xi_t, \tilde \theta_t)$ and $Q_0(\cdot \mid \xi_t, \tilde \theta'_t)$ respectively; if $Y_{t+1}=1$, generate $\tilde \theta_{t+1}$ according to $Q_1(\cdot)$ and let $\tilde \theta'_{t+1} = \tilde \theta_{t+1}$.
\item[(iii)] In the case $t \geq t_Y$, 
generate $\tilde \theta_{t+1}$ according to $Q(\cdot \mid \xi_t, \tilde \theta_t)$ and let $\tilde \theta'_{t+1} = \tilde \theta_{t+1}$.
\end{enumerate}

In view of (\ref{eq-thetatrans}), it can be verified directly by induction on $t$ that the marginal process $\{(\Z_t, \tilde \theta_t)\}$ (resp.\ $\{(\Z_t, \tilde \theta'_t)\}$) in the preceding construction has the same probability distribution as $\{X_t\}$ (resp.\ $\{X'_t\}$).
This implies that $\{\mu_{x,t}\}$ (resp.\ $\{\mu_{x',t}\}$) converges weakly to $\mu$ with probability one if and only if $\{\tilde \mu_{x,t}\}$ (resp.\ $\{\tilde \mu_{x',t}\}$) converges weakly to $\mu$ with probability one.
On the other hand, by construction $\tilde \theta_t = \tilde \theta_t'$ for $t \geq t_Y$, where $t_Y < \infty$ with probability one, so except on a null set, $\{\tilde \mu_{x,t}\}$ and $\{\tilde \mu_{x',t}\}$ have the same weak limits.
Combining these two arguments with the assumption that $\{\mu_{x,t}\}$ converges weakly to $\mu$ with probability one, it follows that the three sequences $\{\tilde \mu_{x,t}\}$, $\{\tilde \mu_{x',t}\}$, and $\{\mu_{x',t}\}$ must all converge weakly to $\mu$ with probability one. 
\end{proof}

\begin{proof}[Proof of Prop.~\ref{prop-c3}]
We suppress the superscript $\alpha$ of $\theta_t^\alpha$ in the proof. Let $\{X_t\} =\{(\Z_t, \theta_t)\}$. By Prop.~\ref{prop-c1}, the set $\mathcal{M}_\alpha$ of invariant probability measures of $\{X_t\}$ is nonempty. Recall also that since the evolution of $\{\Z_t\}$ is not affected by the $\theta$-iterates, the marginal of any $\mu \in \mathcal{M}_\alpha$ on the space $\Zs$ must equal $\zeta$, the unique invariant probability measure of $\{Z_t\}$ (Theorem~\ref{thm-2.1}).

Suppose $\{X_t\}$ has multiple invariant probability measures; i.e., there exist $\mu, \mu' \in \mathcal{M}_\alpha$ with $\mu \not= \mu'$. 
Then by \cite[Theorem 11.3.2]{Dud02} there exists a bounded continuous function $f$ on $\Zs \times \H$ such that 
\begin{equation} \label{eq-prf-c31}
\int f \, d\mu \not= \int f \, d\mu'.
\end{equation}

On the other hand, since $\mu$ is an invariant probability measure of $\{X_t\}$, applying a strong law of large numbers for stationary processes \cite[Chap.\ X, Theorem 2.1]{Doob53} (see also \cite[Lemma 17.1.1 and Theorem 17.1.2]{MeT09}) to the stationary Markov chain $\{X_t\}$ with initial distribution $\mu$, 
we have that there exist a set $D_1 \subset \Zs \times \H$ with $\mu(D_1) = 1$ and a measurable function $g_f$ on $\Zs \times \H$ such that
\begin{enumerate}
\item[(i)] for each $x \in D_1$, with the initial condition $X_0 = x$,
$\lim_{t \to \infty} \frac{1}{t} \sum_{k=0}^{t-1} f(X_k) = g_f(x)$, $\Pr_x$-a.s.;
\item[(ii)] $\E_\mu [ g_f(X_0)] = \E_\mu [ f(X_0)]$ (i.e., $\int g_f d\mu = \int f d\mu$).
\end{enumerate}
The same is true for the invariant probability measure $\mu'$: there exist a set $D_2 \subset \Zs \times \H$ with $ \mu'(D_2) = 1$ and a measurable function $g'_f(x)$ such that 
\begin{enumerate}
\item[(i)] for each $x \in D_2$, with the initial condition $X_0 = x$,
$\lim_{t \to \infty} \frac{1}{t} \sum_{k=0}^{t-1} f(X_k) = g'_f(x)$, $\Pr_x$-a.s.;
\item[(ii)] $\E_{\mu'} [ g'_f(X_0)] = \E_{\mu'} [ f(\X_0)]$ (i.e., $\int g'_f d\mu' = \int f d\mu'$).
\end{enumerate}
\smallskip

Also, since $\{X_t\}$ is a weak Feller Markov chain (Lemma~\ref{lma-c0}), by \cite[Prop.\ 4.1]{Mey89}, for a set of initial conditions $x$ with $\mu$-measure $1$, the occupation probability measures $\{\mu_{x, t}\}$ of $\{X_t\}$ converge weakly, $\Pr_x$-almost surely, to some (nonrandom) $\tilde \mu_x \in \mathcal{M}_\alpha$ that depends on the initial $x$.  The same is true for $\mu'$. So by excluding from $D_1$ a $\mu$-null set and from $D_2$ a $\mu'$-null set if necessary, we can assume that the sets $D_1, D_2$ above also satisfy that for each $x \in D_1 \cup D_2$, the occupation probability measures $\{\mu_{x, t}\}$ converge weakly to an invariant probability measure $\tilde \mu_x$ almost surely. Then since $\frac{1}{t} \sum_{k=0}^{t-1} f(X_k)$ is the same as $\int f d\mu_{x,t}$ for $X_0=x$, we have, by the weak convergence of $\{\mu_{x, t}\}$ just discussed, that 
\begin{equation} \label{eq-prf-c32}
  g_f(x) = \int f d \tilde \mu_x \ \ \ \text{for each } x \in D_1,  \qquad g'_f(x) = \int f d \tilde \mu_x \ \ \ \text{for each } x \in D_2.
\end{equation}  

Certainly we must have $g_f(x) = g_f'(x)$ on $D_1 \cap D_2$. We now relate the values of these two functions at points that share the same $z$-component. In particular, let $\text{proj}(D_1)$ denote the projection of $D_1$ on $\Zs$: $\text{proj}(D_1) = \{ z \in \Zs \mid \exists \, \theta \ \text{with} \  (z, \theta) \in D_1 \}$, and let $D_{1,z}$ be the vertical section of $D_1$ at $z$: 
$D_{1,z} = \{ \theta \mid (z, \theta) \in D_1 \}$. Define $\text{proj}(D_2)$ and $D_{2,z}$ similarly.
If $x = (z, \theta) \in D_1 \cup D_2$ and $x'=(z, \theta') \in D_1 \cup D_2$, then in view of Lemma~\ref{lma-c4} and the weak convergence of $\{\mu_{x, t}\}$ and $\{\mu_{x', t}\}$, we must have $\tilde \mu_x = \tilde \mu_{x'}$. Consequently, by (\ref{eq-prf-c32}), 
for each $z \in \text{proj}(D_1)$,  $g_f(z, \cdot)$ is constant on $D_{1,z}$; for each $z \in \text{proj}(D_2)$, $g_f'(z, \cdot)$ is constant on $D_{2,z}$; and
for each $z \in \text{proj}(D_1) \cap \text{proj}(D_2)$, the constants that $g_f(z, \cdot)$, $g_f'(z, \cdot)$ take on $D_{1,z}$, $D_{2,z}$, respectively, are the same.

We now show $\int f d\mu = \int f d\mu'$ to contradict (\ref{eq-prf-c31}) and finish the proof. 
Since $\mu(D_1)=\mu'(D_2) = 1$ and by Theorem~\ref{thm-2.1} $\mu, \mu'$ have the same marginal distribution on $\Zs$, which is $\zeta$, there exists a Borel set $E \subset \text{proj}(D_1) \cap \text{proj}(D_2)$ with $\zeta(E) = 1$. Consider the sets $ (E \times \H) \cap D_1$ and $(E \times \H) \cap D_2$, which have $\mu$-measure $1$ and $\mu'$-measure $1$, respectively.
By \cite[Prop.\ 10.2.8]{Dud02}, we can decompose $\mu, \mu'$ into the marginal $\zeta$ on $\Zs$ and the conditional distributions $\mu(d\theta \mid z), \mu'(d\theta \mid z)$ for $z \in \Zs$. Then
$$ 1 = \mu\big((E \times \H) \cap D_1\big) = \int_E \int_{D_{1,z}} \mu(d \theta \mid z) \, \zeta(dz), \quad 1 = \mu'\big((E \times \H) \cap D_2\big) = \int_E \int_{D_{2,z}} \mu'(d \theta \mid z) \, \zeta(dz),$$
where the equality for the iterated integral in each relation follows from \cite[Theorem 10.2.1(ii)]{Dud02}. These relations imply that for some set $E_0 \subset E$ with $\zeta(E_0) = 0$, 
\begin{equation} \label{eq-prf-c33}
 \int_{D_{1,z}} \mu(d \theta \mid z) = \int_{D_{2,z}} \mu'(d \theta \mid z) = 1, \qquad \forall \, z \in E \setminus E_0.
\end{equation}
We now calculate $\int g_f d\mu$ and $\int g_f' d\mu'$. We have
\begin{align}
 \int g_f \, d\mu & = \int_{(E \times \H) \cap D_1} g_f \, d\mu = \int_{E} \int_{D_{1,z}} g_f(z, \theta) \, \mu(d\theta \mid z) \, \zeta(dz),  \label{eq-prf-c34a}\\
 \int g_f' \, d\mu' & = \int_{(E \times \H) \cap D_2} g_f' \, d\mu' = \int_{E} \int_{D_{2,z}} g_f'(z, \theta) \, \mu'(d\theta \mid z) \, \zeta(dz), \label{eq-prf-c34b}
\end{align} 
where the equality for the iterated integral in each relation also follows from \cite[Theorem 10.2.1(ii)]{Dud02}.
As discussed earlier, for each $z \in E \subset  \text{proj}(D_1) \cap \text{proj}(D_2)$, the two constant functions, $g_f(z, \cdot)$ on $D_{1,z}$ and $g_f'(z, \cdot)$ on $D_{2,z}$, have the same value. Using this together with (\ref{eq-prf-c33}), we conclude that
\begin{equation} \label{eq-prf-c34c}
  \int_{D_{1,z}} g_f(z, \theta) \, \mu(d\theta \mid z) = \int_{D_{2,z}} g_f'(z, \theta) \, \mu'(d\theta \mid z), \qquad \forall \, z \in E \setminus E_0.
\end{equation}  
Since $\zeta(E_0) = 0$, we obtain from (\ref{eq-prf-c34a})-(\ref{eq-prf-c34c}) that
$$ \int g_f \, d\mu = \int g_f' \, d\mu'.$$
But $\int g_f d\mu = \int f d\mu$ and $\int g_f' d\mu' = \int f d\mu'$ (as we obtained at the beginning of the proof), so $\int f d\mu = \int f d\mu'$, a contradiction to (\ref{eq-prf-c31}). This proves that $\{X_t\}$ must have a unique invariant probability measure.
\end{proof}

Proposition~\ref{prop-c3} implies that for every $m \geq 1$, the $m$-step version of $\{(\Z_t, \theta_t^\alpha)\}$ has a unique invariant probability measure. 
This together with Lemma~\ref{lma-c2} furnishes the conditions (A1)-(A3) of \cite[Prop.\ 4.2]{Mey89} for weak Feller Markov chains (these conditions are the conditions (i)-(iii) of our Lemma~\ref{lma-c1b}). We can therefore apply the conclusions of \cite[Prop.\ 4.2]{Mey89} (see Lemma~\ref{lma-c1b} in our Section~\ref{sec-4.3.1}) to the $m$-step version of $\{(\Z_t, \theta_t^\alpha)\}$ here, and the result is the following proposition: 

\begin{prop} \label{prop-c4}
Under Assumption~\ref{cond-bpolicy}, for each $m \geq 1$, the $m$-step version of $\{(\Z_t, \theta_t^\alpha)\}$ has a unique invariant probability measure $\mu^{(m)}$, and the occupation probability measures $\mu^{(m)}_{(z,\theta),t}, t \geq 1$, as defined by (\ref{eq-om}), converge weakly to $\mu^{(m)}$ almost surely, for each initial condition $(z, \theta) \in \Zs \times \H$ of $(\Z_0, \theta_0^\alpha)$.
\end{prop}

With Prop.\ \ref{prop-c4} we can proceed to prove the second part of Theorems~\ref{thm-c2} and~\ref{thm-c2b}. Given that we have already established their first part in the previous subsection, the arguments for their second part are the same for both theorems and are given below.
The proof is similar to that for Theorem~\ref{thm-c1}(ii), except that here, instead of working with the averaged probability measures $\big\{\bar P_{(z,\theta)}^{(m,k)}\big\}$, Prop.~\ref{prop-c4} allows us to work with the occupation probability measures.

\begin{proof}[Proof of the second part of both Theorem~\ref{thm-c2} and Theorem~\ref{thm-c2b}]
We suppress the superscript $\alpha$ of $\theta_t^\alpha$ in the proof. By Prop.~\ref{prop-c4}, $\{(\Z_t, \theta_t)\}$ has a unique invariant probability measure $\mu_\alpha$, and its $m$-step version has a corresponding unique invariant probability measure $\mu^{(m)}_\alpha$. We prove first the statement that for each initial condition $(z, \theta) \in \Zs \times \H$, almost surely,
\begin{equation} \label{eq-c40}
 \liminf_{t \to \infty} \, \frac{1}{t} \sum_{k=0}^{t-1} \I\Big(\sup_{k \leq j <  k+m} \big| \theta_j - \theta^* \big| < \delta \Big) \geq \bar{\mu}_\alpha^{(m)}\big([N'_\delta(\theta^*)]^m \big),
\end{equation} 
where $\bar{\mu}_\alpha^{(m)}$ is the unique element in $\bar{\mathcal{M}}^m_\alpha$, and $N'_\delta(\theta^*)$ is the open $\delta$-neighborhood of $\theta^*$.
For each $t$, by the definition (\ref{eq-om}) of the occupation probability measure $\mu^{(m)}_{(z,\theta),t}$, the average in the left-hand side above is the same as 
$\mu^{(m)}_{(z,\theta),t}(D_{\delta})$, where  $D_{\delta} = \big\{ \big(z^1, \theta^1, \ldots, z^m, \theta^m \big) \in (\Zs \times \H)^m \ \big| \,  \sup_{1 \leq j \leq m} \big|\theta^j - \theta^* \big| < \delta \big\}$. 
By Prop.\ \ref{prop-c4}, $\Pr_{(z, \theta)}$-almost surely, $\{\mu^{(m)}_{(z,\theta),t}\}$ converges weakly to $\mu^{(m)}_\alpha$, and 
therefore, except on a null set of sample paths, we have by \cite[Theorem 11.1.1]{Dud02} that for the open set $D_\delta$,
\begin{equation} 
  \liminf_{t \to \infty} \mu^{(m)}_{(z,\theta),t}(D_\delta) \geq \mu^{(m)}_\alpha(D_\delta) = \bar{\mu}^{(m)}_\alpha\big([N'_\delta(\theta^*)]^m \big). \notag
\end{equation}  
This proves (\ref{eq-c40}).

We now prove the statement that for each initial condition $(z, \theta) \in \Zs \times \H$, almost surely,
\begin{equation} \label{eq-prf-c40b}
  \limsup_{t \to \infty} \big| \bar \theta_t - \theta^* \big|  \leq \delta \, \kappa_{\alpha} + 2 r_\H\, (1 - \kappa_{\alpha}), \qquad \text{where} \ \kappa_\alpha = \bar{\mu}_\alpha\big(N'_\delta(\theta^*) \big),
\end{equation}  
and $\bar{\mu}_\alpha$ is the marginal of $\mu_\alpha$ on $\H$.
The statement is trivially true if $\delta \geq 2 r_\H$, so consider the case $\delta < 2r_\H$. Fix an initial condition $(z, \theta) \in \Zs \times \H$ for $(Z_0, \theta_0)$,
and let $\{\mu_{(z,\theta),t}\}$ be the corresponding occupation probability measures of $\{(Z_t, \theta_t)\}$.
For the averaged sequence $\{\bar \theta_t\}$, by convexity of the norm $|\cdot|$,
\begin{equation} \label{eq-prf-c41}
  \big| \bar \theta_t - \theta^* \big|  \leq  \frac{1}{t} \sum_{k=0}^{t-1} \,  | \theta_k - \theta^* |.
\end{equation}  
We have
\begin{align}
   \frac{1}{t} \sum_{k=0}^{t-1}   | \theta_k - \theta^* | &  \leq \frac{1}{t} \sum_{k=0}^{t-1}   | \theta_k - \theta^* |  \cdot  \I \big( \theta_k \in N'_\delta(\theta^*) \big)  
   + \frac{1}{t} \sum_{k=0}^{t-1}  | \theta_t - \theta^* |  \cdot  \I \big( \theta_k \not\in N'_\delta(\theta^*) \big)  \notag \\
   & \leq \delta \cdot \mu_{(z,\theta),t}(D_{\delta})  + 2 r_\H \cdot \big( 1 - \mu_{(z,\theta),t}(D_{\delta}) \big), \label{eq-prf-c42}
\end{align}
where $D_{\delta} = \big\{ (z^1, \theta^1 ) \in \Zs \times \H \ \big| \,  | \theta^1 - \theta^* | < \delta \big\}$. 
By Prop.\ \ref{prop-c4}, $\Pr_{(z, \theta)}$-almost surely, $\{\mu_{(z,\theta),t}\}$ converges weakly to $\mu_\alpha$. Therefore, except on a null set of sample paths, we have by \cite[Theorem 11.1.1]{Dud02} that for the open set $D_\delta$,
\begin{equation} \label{eq-prf-c43}
  \liminf_{t \to \infty} \mu_{(z,\theta),t}(D_\delta) \geq \mu_\alpha(D_\delta) = \bar{\mu}_\alpha\big(N'_\delta(\theta^*) \big).
\end{equation}  
Combining the three inequalities (\ref{eq-prf-c41})-(\ref{eq-prf-c43}), and using also the relation $\delta < 2 r_\H$, we obtain that (\ref{eq-prf-c40b}) holds almost surely for each initial condition $(z, \theta) \in \Zs\times \H$. This completes the proof.
\end{proof}

%\smallskip
\begin{rem}[on the role of perturbation again] \label{rmk-c2b}
As mentioned before Prop.\ \ref{prop-c3}, our purpose of perturbing the constrained ETD algorithms is to guarantee that the Markov chain $\{(\Z_t, \theta_t^\alpha)\}$ has a unique invariant probability measure. Without the perturbation, this cannot be ensured, so we cannot apply the ergodic theorem given in Lemma~\ref{lma-c1b} to exploit the convergence of occupation probability measures, as we did in the preceding proof, even though $\{(\Z_t, \theta_t^\alpha)\}$ satisfies the remaining two conditions required by that ergodic theorem (cf.~Lemma~\ref{lma-c2}, Section~\ref{sec-4.3.2}). 

In connection with this discussion, let us clarify a point. We know that the occupation probability measures of $\{\Z_t\}$ converge weakly to its unique invariant probability measure $\zeta$ almost surely for each initial condition of $\Z_0$ (Theorem~\ref{thm-2.1}). But this fact alone cannot rule out the possibility that $\{(\Z_t, \theta_t^\alpha)\}$ has multiple invariant probability measures and that its occupation probability measures do not converge for some initial condition $(z, \theta)$.

Finally, another property of weak Feller Markov chains and its implication for our problem are worth noting here. By \cite[Prop.\ 4.1]{Mey89}, for a weak Feller Markov chain $\{X_t\}$, provided that an invariant probability measure $\mu$ exists, we have that for a set of initial conditions $x$ with $\mu$-measure $1$, the occupation probability measures $\{\mu_{x,t}\}$ converge weakly, $\Pr_x$-almost surely, to an invariant probability measure $\mu_x$ that depends on the initial condition. Thus, for the unperturbed algorithms (\ref{eq-emtd-const0}), (\ref{eq-emtd-const1}) and (\ref{eq-emtd-const2}), despite the possibility of $\{(\Z_t, \theta_t^\alpha)\}$ having multiple invariant probability measures, the preceding proof can be applied to those initial conditions from which the occupation probability measures converge almost surely. 
In particular, this argument leads to the following conclusion. In the case of the algorithm (\ref{eq-emtd-const0}), (\ref{eq-emtd-const1}) or (\ref{eq-emtd-const2}), under the same conditions as in Theorem~\ref{thm-c1} or~\ref{thm-c1b}, it holds for any invariant probability measure $\mu$ of $\{(\Z_t, \theta_t^\alpha)\}$ that for each initial condition $(z,\theta)$ from some set of initial conditions with $\mu$-measure $1$,
$$\qquad \limsup_{t \to \infty}  \big| \bar \theta_t^\alpha - \theta^* \big|  \leq \delta \, \kappa_{\alpha} + 2 r_\H\, (1 - \kappa_{\alpha}) \quad \text{$\Pr_{(z,\theta)}$-a.s.},$$
where $\kappa_{\alpha} = \inf_{\mu \in \bar{\mathcal{M}}_\alpha} \mu(N'_\delta(\theta^*))$. 
The limitation of this result, however, is that the set of initial conditions involved is unknown and can be small. \qed
\end{rem}

\section{Discussion} \label{sec-5}

In this section we discuss direct applications of our convergence results to ETD($\lambda$) under relaxed conditions and to two other algorithms, the off-policy TD($\lambda$) algorithm and the ETD($\lambda, \beta$) algorithm proposed by Hallak et al.~\citep{etd-errbd2}. We then discuss several open issues to conclude the paper.

\subsection{The Case without Assumption~\ref{cond-features}} \label{sec-5.1}

Let Assumption~\ref{cond-bpolicy} hold. Recall from Section~\ref{sec-2.2b} that ETD($\lambda$) aims to solve the equation $C \theta + b = 0$, where
$$ b   =  \Fe^\top  \bM \,  \rl, \qquad \quad C   =  - \Fe^\top  \G \,  \Fe \quad \text{with} \ \ \ \G = \bM (I - \PL).$$
In this paper we have focused on the case where Assumption~\ref{cond-features} holds and $C$ is negative definite (Theorem~\ref{thm-matrix}). If Assumption~\ref{cond-features} does not hold, then either there are less than $n$ emphasized states (i.e., states $s$ with $\bM_{ss}>0$), or the feature vectors of emphasized states are not rich enough to contain $n$ linearly independent vectors. In either case the function approximation capacity is not fully utilized. It is hence desirable to fulfill Assumption~\ref{cond-features} by adding more states with positive interest weights $\i(s)$ or by enriching the feature representation. 

Nevertheless, suppose Assumption~\ref{cond-features} does not hold (in which case $C$ is negative semidefinite~\cite{SuMW14}). This essentially has no effects on the convergence properties of the constrained or unconstrained ETD($\lambda$) algorithms, because of the emphatic weighting scheme (\ref{eq-td3})-(\ref{eq-td1}), as we explain now. 

Let there be at least one state $s$ with interest weight $\i(s) > 0$ (the case is vacuous otherwise). 
Partition the state space into the set of emphasized states and the set of non-emphasized states:
$$\J_1= \{ s \in \S \mid \bM_{ss} > 0 \}, \qquad \J_0= \{ s \in \S \mid \bM_{ss} = 0 \}. $$
Corresponding to the partition, by rearranging the indices of states if necessary, we can write 
$$ \Fe = \left[ \begin{matrix} \Fe_1 \\ \Fe_0 \end{matrix} \right],  \qquad \rl =  \left[ \begin{matrix} r_1 \\ r_0 \end{matrix} \right], 
\qquad \bM = \left[ \begin{matrix} \hat M & 0_{|\J_1| \times |\J_0|} \\
           0_{|\J_0| \times |\J_1|} & 0_{|\J_0| \times |\J_0|} \end{matrix} \right],$$
where $0_{m\times m'}$ denotes an $m \times m'$ zero matrix, $\hat M$ is a diagonal matrix with $\bM_{ss}, s \in \J_1$, as its diagonals.
Let $\hat Q$ be the sub-matrix of $\PL$ that consists of the entries whose row/column indices are in $\J_1$.
For the equation $C \theta + b =0$, clearly $b = \Fe_1^\top \hat M r_1$. Consider now the matrix $C$.
It is shown in the proof of Prop.\ C.2 in \cite{yu-etdarx} that 
$\G$ has a block-diagonal structure with respect to the partition $\{\J_1, \J_0\}$,
$$ \G = \left[ \begin{matrix} \hat \G & 0_{|\J_1| \times |\J_0|} \\
           0_{|\J_0| \times |\J_1|} & 0_{|\J_0| \times |\J_0|} \end{matrix} \right],$$
where the block corresponding to $\J_0$ is a zero matrix as shown above, and the block $\hat \G$ corresponding to $\J_1$ is a positive definite matrix given by
\begin{equation} \label{eq-expG1}
\hat \G = \hat M ( I - \hat Q),
\end{equation}
and $\hat M$ can be expressed explicitly as
\begin{equation} \label{eq-expG2}
  diag(\hat M)= {\bi^1}^{\top} (I - \hat Q)^{-1}, \qquad \bi^1 \in \re^{|\J_1|}, \  \bi^1(s) = d_{\pi^o}(s) \cdot  \i(s), \ s \in \J_1.
\end{equation} 
Thus the matrix $C$ has a special structure:

\begin{thm}[structure of the matrix $C$; {\cite[Appendix C.2, p.\ 41-44]{yu-etdarx}}] \label{thm-5.1}
Let Assumption~\ref{cond-bpolicy} hold, and let $\i(s) > 0$ for at least one state $s \in \S$. Then 
$$ C = - \Fe_1^\top \hat \G \Fe_1, \qquad \text{where} \ \hat \G = \hat M ( I - \hat Q) \  \text{is positive definite.}$$
\end{thm}

Let $\text{range}(A)$ denote the range space of a matrix $A$. By the positive definiteness of the matrix $\hat \G$ given in the preceding theorem, the negative semidefinite matrix $C$ possesses the following properties (we omit the straightforward proof):

\begin{prop} \label{prop-5.1}
Let Assumption~\ref{cond-bpolicy} hold, and let $\i(s) > 0$ for at least one state $s \in \S$. Then the matrix $C$ satisfies that
\begin{enumerate}
\item[\rm (i)] $\text{range}(C) = \text{range}(C^\top) = \text{span}\{\fe(s) \mid s \in \J_1\}$; and
\item[\rm (ii)] there exists $c > 0$ such that for all $x \in \text{span}\{\fe(s) \mid s \in \J_1\}$, $x^\top C x \leq - c \, | x |^2$.
\end{enumerate}
\end{prop}

Two observations then follow immediately: 
\begin{enumerate}
\item[(i)] Since $b = \Fe_1^\top \hat M r_1\in \text{span}\{\fe(s) | s \in \J_1\}$, Prop.~\ref{prop-5.1}(i) shows that the equation $C \theta + b = 0$ admits a solution, and a unique one in $\text{span}\{\fe(s) \mid s \in \J_1\}$, which we denote by $\theta^*$.% footnote starts
\footnote{\label{footnote-appr-sd}From the structures of $\G$, $\PL$, $\hat Q$ and $\hat M$ shown in \cite[Appendix C.2, p.\ 41-44]{yu-etdarx}, which give rise to (\ref{eq-expG1})-(\ref{eq-expG2}), we also have the following facts. The approximate value function $v = \Fe_1 \theta^*$ for the emphasized states $\J_1$ is the unique solution of the projected Bellman equation $v = \Pi (r_1 + \hat Q v)$, where $\Pi$ is the projection onto the column space of $\Fe_1$ with respect to the weighted Euclidean norm on $\re^{|\J_1|}$ defined by the weights $\bM_{ss}, s \in \J_1$ (the diagonals of $\hat M$). The equation $v = r_1 + \hat Q v$ is indeed a generalized Bellman equation for the emphasized states only, and has $v_\pi(s), s \in \J_1$, as its unique solution. Then for the emphasized states, the relation between the approximate value function $\Fe_1 \theta^*$ and $v_\pi$ on $\J_1$, in particular the approximation error, can again be characterized using the oblique projection viewpoint \cite{bruno-oblproj}, similar to the case with Assumption~\ref{cond-features} discussed in Section~\ref{sec-2.2b}.}
%footnote ends
\item[(ii)] Prop.~\ref{prop-5.1}(ii) shows that $C$ acts like a negative definite matrix on the space of feature vectors, $\text{span}\{\fe(s) | s \in \J_1\}$, that the ETD($\lambda$) algorithms naturally operate on.\footnote{Start ETD($\lambda$) from a state $S_0$ with $\i(S_0) > 0$. It can be verified that the emphatic weighting scheme dictates that if $S_t \in \J_0$, then the emphasis weight $\M_t$ for that state must be zero. Consequently, $\e_t$ is a linear combination of the features of the emphasized states and the initial $\e_0$. So when $\e_0 \in \text{span}\{\fe(s) | s \in \J_1\}$, $\e_t \in \text{span}\{\fe(s) | s \in \J_1\}$ always, and if in addition $\theta_0 \in \text{span}\{\fe(s) | s \in \J_1\}$, then $\theta_t \in \text{span}\{\fe(s) | s \in \J_1\}$ always. This is very similar to the case of TD($\lambda$) with possibly linearly dependent features discussed in \cite{tr-disc}.}
\end{enumerate}
We remark that for an arbitrary negative semidefinite matrix $C$, neither of these conclusions holds. They hold here as direct consequences of the positive definiteness of the matrix $\hat \G$ that underlies $C$, and this positive definiteness property is due to the emphatic weighting scheme (\ref{eq-td3})-(\ref{eq-td1}) employed by ETD($\lambda$). 

Now let us discuss the behavior of the constrained ETD($\lambda$) algorithms starting from some state $S_0$ of interest (i.e., $\i(S_0) > 0$), in the absence of Assumption~\ref{cond-features}. Recall that earlier we did not need Assumption~\ref{cond-features} when applying the two general convergence theorems from \cite{KuY03}, and we used the negative definiteness of $C$ implied by this assumption only near the end of our proofs to get the solution properties of the mean ODE associated with each algorithm. In the absence of Assumption~\ref{cond-features}, for the unperturbed algorithms (\ref{eq-emtd-const0}), (\ref{eq-emtd-const1}) and (\ref{eq-emtd-const2}), we can simply restrict attention to the subspace $\text{span}\{\fe(s) | s \in \J_1\}$ and use the property in Prop.~\ref{prop-5.1}(ii) in lieu of negative definiteness. 
After all, the $\theta$-iterates of these algorithms always lie in the span of the feature vectors if the initial $\theta_0, \e_0 \in \text{span}\{\fe(s) | s \in \J_1\}$ and in the case of the two biased algorithms (\ref{eq-emtd-const1}) and (\ref{eq-emtd-const2}), if the function $\psi_K(x)$ does not change the direction of $x$. On the subspace $\text{span}\{\fe(s) | s \in \J_1\}$, in view of Prop.~\ref{prop-5.1}(ii), the function $| \theta - \theta^*|^2$ serves again as a Lyapunov function for analyzing the ODE solutions in exactly the same way as before. Thus, in the absence of Assumption~\ref{cond-features}, for the algorithms (\ref{eq-emtd-const0}), (\ref{eq-emtd-const1}) and (\ref{eq-emtd-const2}) that set $\theta_0, \e_0$ and $\psi_K$ as just described, and for $r_\H > |b|/c$ where $c$ is as in Prop.~\ref{prop-5.1}(ii), the conclusions of Theorems \ref{thm-dim-stepsize}-\ref{thm-c1b} in Section~\ref{sec-3} continue to hold with $N_\delta(\theta^*)$ or $N'_\delta(\theta^*)$ replaced by $N_\delta(\theta^*) \cap \text{span}\{\fe(s) | s \in \J_1\}$ or  $N'_\delta(\theta^*) \cap \text{span}\{\fe(s) | s \in \J_1\}$.

The same is true for the almost sure convergence of the unconstrained ETD($\lambda$) algorithm (\ref{eq-emtd0}) under diminishing stepsize: with $\i(S_0) > 0$ and $\theta_0, \e_0 \in \text{span}\{\fe(s) | s \in \J_1\}$, the conclusion of \cite[Theorem 2.2]{yu-etdarx} continues to hold in the absence of Assumption~\ref{cond-features}; that is, for $\alpha_t = O(1/t)$ with $\tfrac{\alpha_t - \alpha_{t+1}}{\alpha_t} = O(1/t)$, $\theta_t \asto \theta^*$.

It can be seen now that without Assumption~\ref{cond-features}, complications can only arise through initializing the algorithms outside the desired subspace. We discuss such situations briefly, although they do not seem natural and can be easily avoided. Suppose for some reason we give the initial $\theta_0, \e_0$ a component perpendicular to $\text{span}\{\fe(s) | s \in \J_1\}$. Let $\i(S_0) > 0$. The behavior of the unconstrained ETD($\lambda$) algorithm (\ref{eq-emtd0}) is easy to describe. 
For each $t$, write $\theta_t = \theta_t^{(1)} + \theta_t^{(0)}$ and $\e_t = \e_t^{(1)} + \e_t^{(0)}$, where 
$\theta_t^{(1)}, \e_t^{(1)} \in \text{span}\{\fe(s) | s \in \J_1\}$ and $\theta_t^{(0)}, \e_t^{(0)} \perp \text{span}\{\fe(s) | s \in \J_1\}$. The latter components do not affect the evolution of $\{\theta_t^{(1)}\}$ and $\{\e_t^{(1)}\}$, and they also do not affect the approximate value functions $\Fe_1 \theta_t = \Fe_1 \theta_t^{(1)}$ for the emphasized states. For the component process $\{\theta_t^{(1)}\}$, by the discussion earlier, for stepsize $\alpha_t= O(1/t)$ with $\tfrac{\alpha_t - \alpha_{t+1}}{\alpha_t} = O(1/t)$ , $\theta_t^{(1)} \asto \theta^*$. If $\e_0^{(0)} = 0$, then clearly the components $\e_t^{(0)} =0$ and $\theta_t^{(0)} = \theta_0^{(0)}$ through out the iterations. If $\e_0^{(0)} \not= 0$, then by relating to the case where this component is zero and applying Prop.~\ref{prp-2}, we have $\e_t^{(0)} \asto 0$. In the on-policy case with $\gamma(s)\lambda(s) < 1$ for all $s$, the magnitude of $\e_t^{(0)}$ in fact shrinks exponentially fast and consequently it can be shown that $\{\theta_t^{(0)}\}$ converges to a point depending on the sample path. In the general off-policy case, depending on how fast $\e_t^{(0)}$ shrinks, $\theta_t^{(0)}$ may not converge, we think, although this does not affect the approximate value function for the emphasized states, as noted earlier.

For the constrained ETD($\lambda$) algorithms, if we decompose their iterates into two components as above, the evolution of $\{\theta_t^{(1)}\}$ and $\{\e_t^{(1)}\}$ can be affected by the components perpendicular to $\text{span}\{\fe(s) | s \in \J_1\}$ through the scaling performed by $\Pi_\H$ or $\psi_K$ (assuming again that $\psi_K(x)$ maintains the direction of $x$). 
Nevertheless, the asymptotic behavior of the algorithms is still characterized by the limit set of their respective mean ODEs. 
For the algorithm (\ref{eq-emtd-const0}), the mean ODE is $\dot{x} = \bar h(x) + z, z \in \mathcal{N}_\H(x)$, where $\bar h(x) = C x + b = C(x - \theta^*)$. 
Let $r_\H > |b|/c$ where $c$ is as in Prop.~\ref{prop-5.1}(ii). Let $(x(\tau), z(\tau))$ be a solution of this ODE with $x(0) \in \H$. Decompose $x$ as $x(\tau) = x^{(1)}(\tau) + x^{(0)}(\tau)$, with $x^{(1)}$ lying inside the subspace $\text{span}\{\fe(s) | s \in \J_1\}$ and $x^{(0)}$ perpendicular to that subspace.
Then since $\bar h(x) = C (x^{(1)} - \theta^*) \in \text{span}\{\fe(s) | s \in \J_1\}$ by Prop.~\ref{prop-5.1}(i), based on the Euclidean geometry and Prop.~\ref{prop-5.1}(ii), we observe that for $V_1(\tau) = |x^{(1)}(\tau) - \theta^*|^2$ and $V_0(\tau) = |x^{(0)}(\tau)|^2$, we have $\dot{V}_1(\tau) < 0$ whenever $x^{(1)}(\tau) \not= \theta^*$, and $\dot{V}_0(\tau) \leq 0$ always and $\dot{V}_0(\tau) < 0$ whenever $z(\tau) \not=0$. Following this observation it can be worked out that the limit set $L_\H \subset  \big\{ \theta^* + y \mid y \perp \text{span}\{\fe(s) | s \in \J_1\}\big\} \cap \H$; i.e., $L_\H$ is a subset of the solutions of $C \theta + b = 0$ in $\H$. Then the conclusions in Section~\ref{sec-3} about the algorithm (\ref{eq-emtd-const0}) and its perturbed version (\ref{eq-valg}) hold with the cylindrical solution neighborhood $N_\delta(L_\H)$ or $N'_\delta(L_\H)$ in place of $N_\delta(\theta^*)$ or $N'_\delta(\theta^*)$. Similar conclusions hold for the biased algorithms (\ref{eq-emtd-const1}) and (\ref{eq-emtd-const2}) and their perturbed version (\ref{eq-valg}), in view of the uniform approximation property given in (\ref{eq-approx}) for the functions $\bar h_K$ involved in their mean ODEs. We omit the details in part because it does not seem natural to initialize $\theta_0, \e_0$ with a component perpendicular to $\text{span}\{\fe(s) | s \in \J_1\}$ in the first place. 

As a final note, in the absence of Assumption~\ref{cond-features}, any solution $\bar \theta$ of $C\theta + b = 0$ gives \emph{the same approximate value function for emphasized states}, but the approximate values $\Fe_0 \bar \theta$ for non-emphasized states in $\J_0$ are \emph{different} for different solutions $\bar \theta$. Thus one needs to be cautious in using the approximate values $\Fe_0 \bar \theta$. They correspond to different extrapolations from the approximate values $\Fe_1 \theta^*$ for the emphasized states, whereas $\Fe_1 \theta^*$ is not defined to take into account approximation errors for those states in $\J_0$, although its approximation error for emphasized states can be well characterized (cf.\ Footnote~\ref{footnote-appr-sd}).

\subsection{Off-policy TD($\lambda$)} \label{sec-5.2}

Applying TD($\lambda$) to off-policy learning by using importance sampling techniques was first proposed in \cite{offpolicytd-pss,offpolicytd-psd}, and the focus there was on episodic data.
The analysis we gave in this paper applies directly to the (non-episodic) off-policy TD($\lambda$) algorithm studied in \cite{by08,Yu-siam-lstd,dnp14}, when its divergence issue is avoided by setting $\lambda$ sufficiently large. Specifically, we consider constant $\gamma \in [0,1)$ and constant $\lambda \in [0,1]$, and an infinitely long trajectory generated by the behavior policy as before.
The algorithm is the same as TD($\lambda$) except for incorporating the importance sampling weight $\rho_t$:% footnote starts 
\footnote{It is not necessary to multiply the term $\fe(S_t)^\top \theta_t$ by $\rho_t$, and that version of the algorithm was the one given in \cite{by08,Yu-siam-lstd}. The experimental results in \cite{dnp14} suggest to us that each version can have less variance than the other in some occasions, however. As far as convergence analysis is concerned, the two versions are essentially the same and the analyses given in \cite{Yu-siam-lstd,yu-etdarx} and this paper indeed apply simultaneously to both versions of the algorithm.}
% footnote ends 
$$      \theta_{t+1} = \theta_t + \alpha_t \, \e_t \cdot \rho_t \, \big( R_{t} + \gamma \fe(S_{t+1})^\top \theta_t - \fe(S_t)^\top \theta_t \big),$$
where
$$ \e_{t} =   \lambda  \gamma   \rho_{t-1} \, \e_{t-1} +  \fe(S_t).$$
The constrained versions of the algorithm are defined similarly to those for ETD($\lambda$).

Under Assumption~\ref{cond-bpolicy}(ii), the associated projected Bellman equation is the same as that for on-policy TD($\lambda$) \cite{tr-disc} except that the projection norm is the weighted Euclidean norm with weights given by the steady state probabilities $d_{\pi^o}(s), s \in \S$. Assuming $\Fe$ has full column rank, the corresponding equation in the $\theta$-space, $C \theta + b = 0$, has the desired property that the matrix $C$ is negative definite, if $\lambda$ is sufficiently large (in particular if $\lambda = 1$) \cite{by08}.
For that case, the conclusions given in this paper for constrained ETD($\lambda$) all hold for the corresponding versions of off-policy TD($\lambda$). (Similarly, for the case of $C$ being negative semidefinite due to $\Fe$ having rank less than $n$, the discussion given in the previous subsection for ETD($\lambda$) also applies.) The reason is that besides the property of $C$, the other properties of the iterates that we used in our analysis, which are given in Section~\ref{sec-2} and Appendix~\ref{appsec-a}, all hold for off-policy TD($\lambda$). (In fact, some of these properties were first derived for off-policy LSTD($\lambda$) and TD($\lambda$) in \cite{Yu-siam-lstd} and extended later in \cite{yu-etdarx} to ETD($\lambda$).)

For the same reason, the convergence analyses we gave in \cite{yu-etdarx} and this paper for ETD also apply to a variation of the ETD algorithm, ETD($\lambda, \beta$), proposed recently by Hallak et al.~\cite{etd-errbd2}, when the parameter $\beta$ is set in an appropriate range.

\subsection{Open Issues} 

A major difficulty in applying off-policy TD learning, especially with $\lambda > 0$, is the high variances of the iterates. For ETD($\lambda$), off-policy TD($\lambda$) and their least-squares versions, because of the growing variances of products of the importance sampling weights $\rho_t \rho_{t+1} \cdots$ along a trajectory, and because of the amplifying effects these weights can have on the traces, the variances of the traces iterates can grow unboundedly with time, severely affecting the behavior of the algorithms in practice. (The problem of growing variances when applying importance sampling to simulate Markov systems was also known earlier and discussed in prior works; see e.g., \cite{gi-sampling,rj-tdimportance}.)
The two biased constrained algorithms discussed in this paper were motivated by the need to mitigate the variance problem, and their robust behavior has been observed in our experiments \cite{MYWS15,etd-exp16}. However, beyond simply constraining the iterates, more variance reduction techniques are needed, such as control variates \cite{rj-tdimportance,abj-sampling} and weighted importance sampling \cite{offpolicytd-pss,offpolicytd-psd,wis14,wis15}. To overcome the variance problem in off-policy learning, further research is required. 

Regarding convergence analysis of ETD($\lambda$), the results we gave in \cite{yu-etdarx} and this paper concern only the convergence properties and not the rates of convergence. For on-policy TD($\lambda$) and LSTD($\lambda$), convergence rate analyses are available \cite[Chap.\ 6]{konda-thesis}. Such analyses in the off-policy case will give us better understanding of the asymptotic behavior of the off-policy algorithms. 
Finally, besides asymptotic behavior of the algorithms, their finite-time or finite-sample properties (such as those considered by \cite{ms08,asm08,lgm-lspi,pmtd3}), and their large deviations properties are also worth studying.

\section*{Acknowledgement}
I thank Professors Richard Sutton and Csaba Szepesv\'{a}ri for helpful discussions, and I thank two anonymous reviewers for their helpful feedback. This research was supported by a grant from Alberta Innovates---Technology Futures.

\addcontentsline{toc}{section}{References}
\bibliographystyle{apa}
\let\oldbibliography\thebibliography
\renewcommand{\thebibliography}[1]{%
  \oldbibliography{#1}%
  \setlength{\itemsep}{0pt}%
}
{\fontsize{9}{11} \selectfont
\bibliography{emTD_bib}
}

\addcontentsline{toc}{section}{Appendix}
\appendix
\section{Key Properties of Trace Iterates}  \label{appsec-a}

In this appendix we list four key properties of trace iterates $\{(\e_t, \F_t)\}$ generated by the ETD($\lambda$) algorithm. Three of them were derived in \cite[Appendix A]{yu-etdarx}, and used in the convergence analysis of ETD($\lambda$) in both \cite{yu-etdarx} and the present paper.

As discussed in Section~\ref{sec-3.2}, $\{(\e_t,\F_t)\}$ can have unbounded variances and is naturally unbounded in common off-policy situations. However, as the proposition below shows, $\{(\e_t,\F_t)\}$ is bounded in a stochastic sense.

\begin{prop} \label{prp-bdtrace}
Under Assumption~\ref{cond-bpolicy}, given a bounded set $E \subset \re^{n+1}$, there exists a constant $\C < \infty$ such that if the initial $(\e_0, \F_0) \in E$, then $\sup_{t \geq 0} \E \big[ \big\| (\e_t, \F_t) \big\| \big] < \C$.
\end{prop}

The preceding proposition is the same as \cite[Prop.\ A.1]{yu-etdarx} except that the conclusion is for all the initial $(\e_0,\F_0)$ from the set $E$, instead of a fixed initial $(\e_0,\F_0)$. By making explicit the dependence of the constant $\C$ on the initial $(\e_0,\F_0)$, the same proof of \cite[Prop.\ A.1]{yu-etdarx} (which is a relatively straightforward calculation) applies to the preceding proposition. 

We note that Prop.\ \ref{prp-bdtrace} does not imply the \emph{uniform integrability of $\{(\e_t, \F_t)\}$}---this stronger property does hold for the trace iterates, as we proved in Prop.~\ref{prop-2}(i). (The latter and its proof focus on $\{\e_t\}$ only, but the same argument applies to $\{(\e_t, \F_t)\}$.)

The next proposition concerns the change in the trace iterates due to the change in its initial condition. It is the same as \cite[Prop.\ A.2]{yu-etdarx}; its proof is more involved than the proofs of the two other properties of the trace iterates and uses, among others, a theorem for nonnegative random processes \cite{Nev75}. We did not use this proposition directly in the analysis of the present paper, but it is important in establishing that the Markov chain $\{\Z_t\}$ has a unique invariant probability measure (Theorem~\ref{thm-2.1}), which the results of the present paper rely on. In addition, it is helpful for understanding the behavior of the trace iterates.

Let $({\hat \e}_t, {\hat \F}_t)$, $t \geq 1$, be defined by the same recursion (\ref{eq-td3})-(\ref{eq-td1}) that defines $(\e_t, \F_t)$, using the same state and action random variables $\{(S_t, A_t)\}$, but with a different initial condition $(\hat \e_0, \hat{\F}_0)$. We write a zero vector in any Euclidean space as $\0$. 

\begin{prop} \label{prp-2}
Under Assumption~\ref{cond-bpolicy}, for any two given initial conditions $(\e_0, \F_0)$ and $(\hat{\e}_0, \hat{\F}_0)$, 
$$ \F_t - \hat{\F}_t \asto 0, \qquad  \e_t - \hat{\e}_t  \asto \0.$$
\end{prop}

\smallskip
The third proposition below concerns approximating the trace iterates $(\e_t, \F_t)$ by truncated traces that depend on a fixed number of the most recent states and actions only.  
First, let us express the traces $(\e_t, \F_t)$, by using their definitions (cf.\ Eqs.~(\ref{eq-td3})-(\ref{eq-td1})), as
\begin{align}
  \F_{t} & = \F_0 \cdot \big(\rho_{0} \gamma_{1} \cdots \rho_{t-1} \gamma_t \big) + \sum_{k=1}^t \i(S_k) \cdot \big(\rho_{k} \gamma_{k+1} \cdots \rho_{t-1} \gamma_t \big),  \label{eq-F} \\
  \e_t & = \e_0 \cdot \big(\beta_{1} \cdots \beta_t \big) + \sum_{k=1}^t \M_{k} \cdot \fe(S_k) \cdot \big(\beta_{k+1} \cdots \beta_t \big), \label{eq-e}
\end{align}
where $\beta_k = \rho_{k-1} \gamma_k \lambda_k$ and
$$ \M_k =  \lambda_k \, \i(S_k) + ( 1 - \lambda_k ) \, \F_k.$$
For each integer $K \geq 1$, the truncated traces $(\tilde{\e}_{t,K}, \tilde{\F}_{t,K})$ are defined by limiting the summations in (\ref{eq-F})-(\ref{eq-e}) to be over $K+1$ terms only as follows: 
$$(\tilde{\e}_{t,K}, \tilde{\F}_{t,K}) = (\e_t, \F_t) \quad  \text{for} \ \  t \leq K,$$
and for $t \geq K+1$,
\begin{align}
    \tilde{\F}_{t,K} & = \sum_{k=t-K}^t \i(S_k) \cdot \big(\rho_{k} \gamma_{k+1} \cdots \rho_{t-1} \gamma_t \big), \label{eq-tF} \\
    \tilde{\M}_{t,K} & = \, \lambda_t \, \i(S_t) + ( 1 - \lambda_t) \tilde{\F}_{t,K}, \label{eq-tM} \\
    \tilde{\e}_{t,K} & = \sum_{k=t-K}^t \tilde{\M}_{k,K} \cdot \fe(S_k) \cdot \big(\beta_{k+1} \cdots \beta_t \big). \label{eq-te}
\end{align}
We have the following approximation property for truncated traces, in which the notation ``$\C_K \downarrow 0$'' means that $\C_K$ decreases monotonically to $0$ as $K \to \infty$.

{\samepage
\begin{prop} \label{prp-3}
Let Assumption~\ref{cond-bpolicy} hold. Given a bounded set $E \subset \re^{n+1}$, there exist constants $\C_K, K \geq 1$, with $\C_K \downarrow 0$ as $K \to \infty$, such that if the initial $(\e_0, \F_0) \in E$, then
$$\sup_{t \geq 0} \E \left[ \big\| (\e_t, \F_t) - (\tilde{\e}_{t,K}, \tilde{\F}_{t,K}) \big\| \right] \leq \C_K.$$
\end{prop} 
}

The preceding proposition is the same as \cite[Prop.\ A.3(i)]{yu-etdarx}, except that the initial $(\e_0, \F_0)$ can be from a bounded set $E$ instead of being fixed. The proof given in \cite{yu-etdarx} applies here as well, similar to the case of Prop.~\ref{prp-bdtrace}. 
This proposition about truncated traces was used in \cite{yu-etdarx} to obtain the convergence in mean given in Theorem~\ref{thm-2.2} and allowed us to work with simple finite-space Markov chains, instead of working with the infinite-space Markov chain $\{Z_t\}$ directly, in that proof. In the present paper, it has expedited our proofs of Props.~\ref{prop-2},~\ref{prop-3} regarding the uniform integrability and convergence in mean conditions for constrained ETD($\lambda$).

Finally, the uniform integrability of $\{(\e_t, \F_t)\}$ (proved in Prop.~\ref{prop-2}(i) in this paper, as already mentioned) is important both for convergence analysis and for understanding the behavior of the trace iterates.

\end{document}